\documentclass{article}

\usepackage{arxiv}

\usepackage[utf8]{inputenc} % allow utf-8 input
\usepackage[T1]{fontenc}    % use 8-bit T1 fonts
\usepackage{url}            % simple URL typesetting
\usepackage{booktabs}       % professional-quality tables
\usepackage{amsfonts}       % blackboard math symbols
\usepackage{nicefrac}       % compact symbols for 1/2, etc.
\usepackage{microtype}      % microtypography
\usepackage{lipsum}

\usepackage[utf8]{inputenc}
\usepackage{amsmath}
\usepackage{amsthm}
\usepackage{tensor}
\usepackage{array,multirow,graphicx}
\usepackage{booktabs}
\usepackage{pgffor}
\usepackage{threeparttable}
\usepackage{subcaption}
\usepackage{amssymb}
\usepackage{stix}
\usepackage[ruled, linesnumbered, vlined]{algorithm2e}
\usepackage{accents}
\usepackage{mathrsfs}
\usepackage{hyperref}
\hypersetup{
    colorlinks=true,
}
\usepackage[table]{xcolor}
\usepackage{makecell}

\newcommand{\backgroundcolorlightgreen}{\cellcolor{green!10}}
\newcommand{\backgroundcolorlightblue}{\cellcolor{blue!10}}

\newtheorem{theorem}{Theorem}[section]
\newtheorem{corollary}{Corollary}[theorem]
\newtheorem{lemma}[theorem]{Lemma}
\theoremstyle{definition}
\newtheorem{definition}{Definition}[section]

\title{HAMLET: A Hierarchical Agent-based Machine Learning Platform}

\author{
    Ahmad Esmaeili\\
    Department of Computer and Information Technology\\Purdue University\\
    West Lafayette, IN 47907\\
    \texttt{aesmaei@purdue.edu}\\
\And
    John C.~Gallagher\\
    Department of Electrical Engineering and Computer Science\\
    University of Cincinnati\\
    Cincinnati, OH 45221\\
    \texttt{gallagj9@ucmail.uc.edu}\\
\And
    John A.~Springer\\
    Department of Computer and Information Technology\\Purdue University\\
    West Lafayette, IN 47907\\
    \texttt{jaspring@purdue.edu}\\
\And
    Eric T.~Matson\\
    Department of Computer and Information Technology\\Purdue University\\
    West Lafayette, IN 47907\\
    \texttt{ematson@purdue.edu}\\
}

\begin{document}
\maketitle

\begin{abstract}
Hierarchical Multi-Agent Systems provide convenient and relevant ways to analyze, model, and simulate complex systems composed of a large number of entities that interact at different levels of abstraction. In this paper, we introduce HAMLET (Hierarchical Agent-based Machine LEarning plaTform), a hybrid machine learning platform based on hierarchical multi-agent systems, to facilitate the research and democratization of geographically and/or locally distributed machine learning entities. The proposed system models a machine learning solutions as a hypergraph and autonomously sets up a multi-level structure of heterogeneous agents based on their innate capabilities and learned skills. HAMLET aids the design and management of machine learning systems and provides analytical capabilities for research communities to assess the existing and/or new algorithms/datasets through flexible and customizable queries. The proposed hybrid machine learning platform does not assume restrictions on the type of learning algorithms/datasets and is theoretically proven to be sound and complete with polynomial computational requirements. Additionally, it is examined empirically on 120 training and four generalized batch testing tasks performed on 24 machine learning algorithms and 9 standard datasets. The provided experimental results not only establish confidence in the platform's consistency and correctness but also demonstrate its testing and analytical capacity.  
\end{abstract}

% keywords can be removed

\section{Introduction}
Machine Learning (ML) is a particularly prevalent branch of Artificial Intelligence that is becoming increasingly relevant to real-life applications including, but not limited to economics \cite{athey2018impact}, education \cite{fedushko2019predicting}, agriculture \cite{liakos2018machine}, drug discovery \cite{vamathevan2019applications}, and medicine \cite{rajkomar2019machine}. Such a rapid growth, fueled by either the emergence of new data/applications or the challenges in improving the generality of the solutions, demands straightforward and easy access to the state of the art in the field such that both the researchers and the practitioners would be able to analyze, configure, and integrate machine learning solutions in their own tasks.

Thanks to ubiquitous computing, the last decade has witnessed an explosion in the volume and dimension of data together with the number of machine learning approaches working on them. Knowledge discovery and learning from large and geo-distributed datasets are the targets of the distributed data mining and machine learning approaches in which the concentration is on eschewing over-engineering of machine-learning models by humans and while improving efficacy and scalability by applying both algorithmic innovation and high-performance computing (HPC) techniques to distribute workloads across several machines \cite{verbraeken2020survey}. Additionally, due to the high diversity of the ML problems and solutions contributed by thousands of multi-disciplinary researcher communities around the world, it is becoming exhausting, if not impossible, for both experts and non-experts to keep track of the state of the art. To overcome this hardship, a growing number of research endeavors and commercial technologies are being devised, such as Auto-sklearn \cite{feurer2015efficient}, MLR \cite{bischl2016mlr},  Rapidminer \cite{hofmann2016rapidminer}, OpenML \cite{OpenML2013}, and Google BigQuery ML \cite{bisong2019google} to name a few.

This paper strives to propose a platform that provides an organizational scheme for storing, training, and testing machine learning algorithms and data in a decentralized format. The suggested platform is based on multi-agent systems and is meant to be open in the sense that it is not limited to be used with a pre-determined set of machine learning components. Multi-agent systems (MASs) offer a number of general advantages with respect to computer-supported cooperative working, distributed computation, and resource sharing. Some of these advantages are \cite{wooldridge2009introduction}: (1) decentralized control, (2) robustness, (3) simple extendibility, and (4) sharing of expertise and resources. The decentralized control is, arguably, the most signiﬁcant feature of MAS that serves to distinguish such systems from distributed or parallel computation approaches. Decentralized control implies that individual agents, within a MAS, operate in an autonomous manner and are, in some sense, self-deterministic. Robustness, in turn, is a feature of the decentralized control where the overall system continues to operate even though some of the agents crash. Decentralized control also supports extendibility in the sense that additional functionality can be added simply by including more agents. The advantages of sharing expertise and resources are self-evident and have been heavily used in our work. 

Similar to ML ecosystem, there has been an immense growth in the size and complexity of Multi-Agent Systems (MASs) during the last decades \cite{odell2004metamodel}. Although MASs are considered today to be well-suited to analyze, model, and simulate complex systems, in cases where there is a large number of entities interacting at different levels of abstraction, they fail to faithfully represent complex behaviors with multiple granularities. To deal with this problem, hierarchical systems, as an organizational structure, have attracted the attention of MAS researchers. Today many of its contributions to many application ranging from manufacturing systems \cite{maturana1999metamorph}, transports \cite{burckert1998transportation}, cooperative work \cite{adam2000homascow} or yet radio mobile mesh dimensioning \cite{rodriguez2003towards} are apparent.

The advantages offered by MAS, together with their intrinsic cooperation and coordination abilities, are particularly applicable to machine learning tasks and knowledge discovery in data (KDD) where a considerable collection of tools and techniques are prevalent \cite{czarnowski2013machine}. The integration of multi-agent technologies and machine learning techniques is often called agent mining \cite{cao2009agent} and can be conducted in two directions, i.e. either using machine learning methods to design and train intelligent agents (also called multi-agent learning) or using multi-agent systems to enhance machine learning or data mining processes (also called agent-driven machine learning). To put it briefly, machine learning and data mining tasks can benefit from the multi-agent systems by \cite{zhang2005agents, ryzko2020modern}:
\begin{itemize}
	\item maintaining the autonomy of data sources,
	\item facilitating interactive and distributed approaches,
	\item providing flexibility in the selection of sources,
	\item improving scalability and supporting distributed data sources,
	\item supporting the use of strategies in learning and mining processes, and
	\item inherently enabling collaborative learning
\end{itemize}

This paper focuses on addressing the ever-increasing challenges of organizing, maintaining, analyzing, and democratizing the use of machine learning resources. As mentioned before, this is particularly done by designing and developing a distributed and extensible machine learning platform, called HAMLET, that provides the following main contributions:
\begin{itemize}
	\item It facilitates the organization of machine learning algorithms and data sets in a multi-level similarity-based architecture.
	\item It enables storing and conducting machine learning tasks at both local and global scales.
	\item It democratizes the use of machine learning resources through a simple and intuitive query design.
	\item It facilitates sharing machine learning tasks and results.
	\item It simplifies the analysis of machine learning resources by providing automated training and testing algorithms.
	\item It not only gives the freedom of using a wide variety of machine learning algorithms, but also provides researchers and contributors with the flexibility of applying customized privacy, integrity, and access strategies.
\end{itemize}

HAMLET tries to address and facilitates the machine learning tasks initiated by both novice and advanced users. With the help of the proposed agent-based model and its corresponding algorithms, a beginner user can utilize the platform for finding answers to questions such as: ``what is the accuracy of testing SVC on iris dataset?'', ``which algorithm has the best (or worst) clustering homogeneity score on the digits dataset?'', etc. An expert user, on the other hand, can utilize the system to contribute new algorithms and datasets to the community, run various analyses on the available machine learning resources, analyzing the effect of different hyper-parameter values on the performance of algorithms, implement automated algorithm selection models that utilizes HAMLET, etc.   

The proposed approach employs the concepts of organizational multi-agent systems in which the agents represent the existing machine learning elements such as algorithms, data sets, and models together with management and reporting units. The hierarchical organization of the agents in the proposed system is dynamically and distributedly built  based on the represented capabilities of the agents and during machine learning training and testing tasks. The used capability and skill-based architecture of the agents also provides the above mentioned flexibility in adding more sophisticated decision making processes alongside the basic ML tasks. The suggested platform is evaluated comprehensively using both theoretical and empirical approaches. By the theoretical methods, we prove that a HAMLET-based system makes right decisions, reports correct results, and handles resource shortcomings properly. The theoretical approach also accompanies time and space complexity analyses of its algorithms to make sure that its overloads do not exceed its unique features. In a highly dynamic and communicative environment of multi-agents systems, such theoretical analyses provides an additional assurance that the system operates within its expected limits provided that its assumptions are satisfied. The presented empirical results, on the other hand, follows two primary objectives: demonstrating how the platform is expected to be used through example queries and their corresponding results; and showing its flexibility, coherence, and ease of use on practical real-world use cases. We have implemented the framework in Python programming language and utilized the machine learning libraries and resources commonly used in the community.  

The rest of this paper is organized as follows. First, we review some of the noteworthy work relevant to this research in section 2. Then, we present a detailed description of the proposed platform in section 3. This is followed by section 4 which assesses the correctness and the performance of the proposed model theoretically. Sections 5 highlights the development aspects of the platform and demonstrates its flexibility and capabilities through real-world machine learning tasks. And finally, section 6 concludes the paper and provides suggestion for future work.

\section{Related Work}

There are a considerable number of reports in the literature about the integration of multi-agents and machine learning systems. In this section, we primarily focus on the application of agent-based techniques in different aspects of a machine learning and data mining life cycle, and we review some of the works that are relatively more related to our research. 

One of the earliest works is the research by Kargupta et al. \cite{kergupta1997scalable}. In this work, the authors have proposed Parallel Data Mining Agents (PADMA), a data mining system that uses software agents for local data accessing and analysis, together with a web-based interface for interactive data visualization. PADMA was originally proposed for and tested on a text classification task though its was extended later and successfully used in medical applications \cite{kargupta1997web}. Our proposed platform differs from PADMA in multiple ways. For example, in PADMA the mining agents are local to the sites that data resides but in HAMLET, the agents representing data and algorithms are different entities that communicate over a network. On the other hand, unlike HAMLET, PADMA's structure is single level where all the agent are connected to and managed by a single facilitator. Such a simplistic structure is prone to scalability issues and limits the application of the platform for small sets of problems and specific tasks. Gorodetsky et al. claimed that the core problem in agent-based machine learning and data mining is not the algorithms themselves -- in many cases these are well understood -- but instead the most appropriate mechanisms to allow agents to collaborate \cite{gorodetsky2003multi}. They presented a distributed architecture for classification problems and a set of protocols for a multi-agent software tool. Their suggested architecture comprises two main components to handle source-based data and meta-data -- along with a set of specialized agents -- to handle queries and classification tasks. Dissimilar to HAMLET, their suggested toolkit is not meant to be general purpose covering various data mining and machine learning tasks at the same time. Peng et al. give an interesting comparison between single-agent and multi-agent text classiﬁcation in terms of a number of criteria including response time, quality of classiﬁcation, and economic/privacy considerations \cite{peng2001comparison}. Their collaboration model for the multi-agent setting is based on the work presented in \cite{raje1998designing} according to which, the agents do not initiate document-based cooperation until they are done with their individual classification tasks. As expected, their reported results are in favor of a multi-agent approach. In \cite{tozicka2007framework}, an abstract agent-based distributed machine learning and data mining framework is proposed. The work utilizes a meta-level description for each agent to keep track of its learning process and exchange it with the peers. Such meta-information helps the agents reason about their learning process and hence to improve the results. Unlike this work, HAMLET does not make any contributions on improving the quality of the learning algorithms but focuses more on the organization of ML elements and automating the machine learning tasks.

Agent technology has also been employed in meta-data mining, the combination of results of individual learning agents. Perhaps, the earliest most mature agent-based meta-learning systems are: Java Agent for Meta-learning (JAM) \cite{stolfo1997jam}, a colletive data mining-based experimental system called BODHI \cite{kargupta1999collective}, and Papyrus \cite{bailey1999papyrus}. Basically, all these systems try to combine local knowledge to optimize a global objective. In JAM, a number of learners are simultaneously trained on a number of data subsets, and the results are combined via a meta-learning technique. In BODHI, on the other hand, the primary goal is to provide an extensible system that hosts and facilitates the communications between mobile data mining agents through a three-level hierarchy. In contrast to JAM and BODHI, Papyrus can move not only models but also data from site to site when such a strategy is desired. Papyrus is a specialized system that is designed for clustering while JAM and BODHI are designed for data classiﬁcation. Apart from their specializations in limited sets of applications, these systems differ from HAMLET in structural aspects as well. For example, BODHI uses a hierarchical structure like HAMLET; however, its structure comprises limited and fixed three levels, whereas HAMLET's architecture does not have any predetermined arrangements. On the other hand, the multi-level architecture of BODHI is based on the role of the agents, i.e. whether they are mining agents or facilitators; however, the levels in HAMLET are formed during the life time of the systems and based on the learning capabilities of the agents. And last but not least, in most of the similar platforms that tend to perform agent-based machine learning tasks, the structure is configured and set at the design time based on the objectives of the system. In HAMLET, however, the system starts with basic requirements and sets of protocols and dynamically evolves during its life time based on the ML tasks that are conducted on it.

There is very limited work reported on generic approaches. One example is EMADS (the Extendable Multi-Agent Data Mining Framework) \cite{albashiri2008emads}. EMADS has been evaluated using two data mining scenarios: Association Rule Mining (ARM) and Classiﬁcation. This framework can ﬁnd the best classiﬁer providing the highest accuracy with respect to a particular data set. EMADS uses ﬁxed protocols and since the mining tasks are managed by a single site, called mediator, it is not easily scalable to cover large number of learning agents. Our proposed method, on the other hand, not only provides the analytical information about the best learners but also is capable of conducting batch learning tasks, is not limited to classification, and is easily scalable to hold thousands of elements. Liu et al. proposed a hierarchical and parallel model called DRHPDM that utilized a centralized and distributed mining layers for the data sources according to their relevance and a global processing unit to combine the results \cite{liu2011distributed}. In their work, the MAS has aided their approach to realize a flexible cross-platform mining task. The proposed architecture relies on a central GPU unit to determine the relevance of the data sources, forming local mining groups, and processing the final models. Unlike HAMLET which is capable of performing various machine learning tasks with different configuration parameters, DRHPDM merely focuses on a single mining tasks at a time. Furthermore, it does not clearly specify how the central unit would cope with the use of customized learning algorithms at large scales.  In \cite{chaimontree2012framework}, a multi-agent-based clustering framework called MABC is suggested. This framework utilizes four types of agents, namely user, data, validation, and clustering, to perform the clustering task in two phases: firstly generating the initial configuration and then improving them through the negotiation among the agents. This work differs from our proposed platform in the sense that it is used only for clustering with predetermined set of algorithms, and concentrates more on improving the clustering results than providing analytical information. ActoDatA (Actor Data Analysis) \cite{lombardo2019multi} is another framework that utilizes an agent-based architecture to provide a set of software tools for distributed data mining and analysis. Its architecture is based on five types of agents that are responsible for various tasks such as mining, analysis, pre-processing, data acquisition, and interacting with users. ActoDataA's objectives differ from our proposed platform as they are primarily focused on classification tasks and automating different phases of its workflow.

There are numerous other work in the literature that utilize agent technologies and multi-agent systems to improve mining and learning results in specific applications. The research report in \cite{bosse2017incremental} introduces a lightweight mobile agent platform called JavaScript Agent Machine (JAM) that couples machine learning with multi-agent systems. The primary goal of this platform is to provide a seamless integration of IoT and cloud-based mobile environments using agent technologies. Although the main concentration of the work is on learning for earthquake and disaster monitoring, and JAM is used only as a simulation platform, it's methodology and features, such as algorithm and data separation -- together with its various connectivity capabilities -- can be adopted for other distributed learning applications as well. In \cite{yong2019multi}, Yong et al. have proposed a multi-agent architecture that leverages probabilistic machine learning by identifying the sources of uncertainty. The main point of concentration in their work is the cyber-physical manufacturing systems; and they have experimented their approach in real-time monitoring of hydraulic systems using a Bayesian Neural Network. The work reported in \cite{tounsi2020csmas} has
employed multi-agent coordination to present a mining process for predicting credit risks in banking. In \cite{ganapathy2012intelligent}, on the other hand, the authors have designed an intrusion detection system that delegates various tasks such as data collection, data pre-processing, selection, outlier detection, and classification to different software agents. And finally, in \cite{geng2021agent}, the authors have suggested a clustering framework for satellite networks that employs fixed three layers of agents. As these researches are not directly related to objectives of our paper, we do not delve into them further.

\section{Proposed Approach}
This section presents the details of the proposed platform. We first model a machine learning system as a hypergraph and then, we use this graph to design a hierarchical multi-agent system capable of handling various machine learning queries.
\subsection{Problem Formalization}
\label{sec:formalization}
The principal components of any machine learning system are the implemented algorithms, the stored datasets, and the visualization and/or report generators. In the proposed platform, we concentrate on the first two parts though it can be easily extended to support as many additional sections as required. Let's assume that each algorithm or dataset entity in such a system is represented by a node. These nodes are connected by training/testing queries initiated by end-users. Since the models might involve more than two entities, we model the entire system as a hypergraph. The nodes of the hypergraph are the data/algorithm entities, and the hyperedges are the models that are built during training/testing procedures. Formally, we denote the system hypergraph by tuple $S\left<X,M\right>$ where $M$ is the set of built models and $X$ is the set of all available data and algorithms, that is, $X=\{a_1, a_2,\dots,a_N,d_1, d_2,\dots,d_L\}$ where $a_i$ and $d_j$ represent an specific algorithm and data, respectively. Figure~\ref{fig:hyp_org} depicts an example system composed of three algorithms: $\{svm, k\text{-}means(\text{KM}), c4.5\}$; and two datasets: $\{iris, wine\}$ that are used in three models $m_{svm\text{-}wine}, m_{c4.5\text{-}iris}$, and $m_{k\text{-}means\text{-}wine}$. In other words, $X=\{svm, k\text{-}means(\text{KM}), c4.5, iris, wine\}$ and $M=\{m_{svm\text{-}wine}, m_{c4.5\text{-}iris}\}$. In this research, we use the multi-modal representation of the hypergraph in which a hyperedge is translated into a new type of node that is connected to all the entities it contains. Figure~\ref{fig:hyp_alt} depicts the multi-modal representation of the aforementioned example of a machine learning system.  

\begin{figure}
	\centering
	\begin{subfigure}[b]{0.47\textwidth}
		\includegraphics[width=\textwidth]{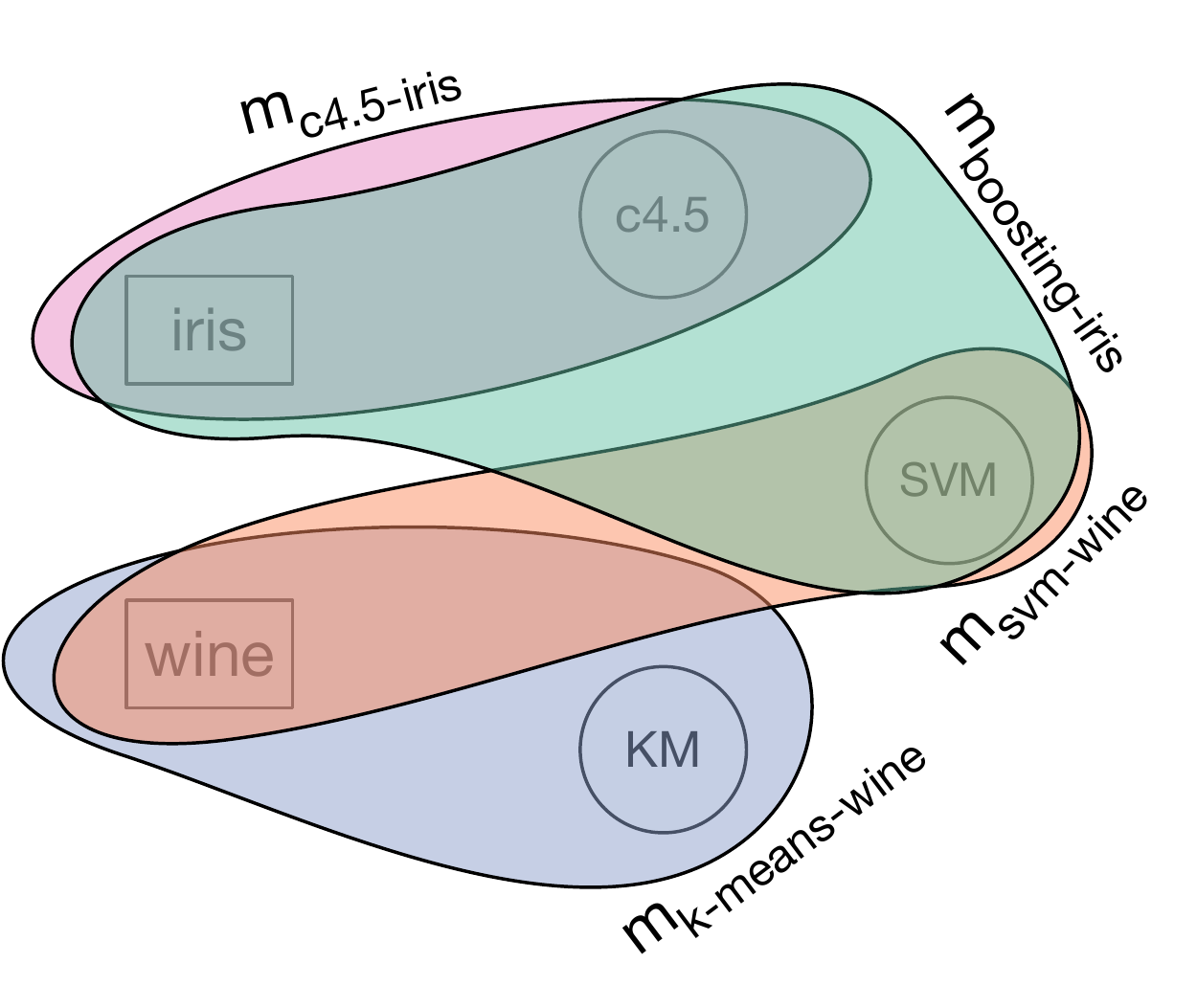}
		\caption{~}
		\label{fig:hyp_org}
	\end{subfigure}
	\begin{subfigure}[b]{0.35\textwidth}
		\includegraphics[width=\textwidth]{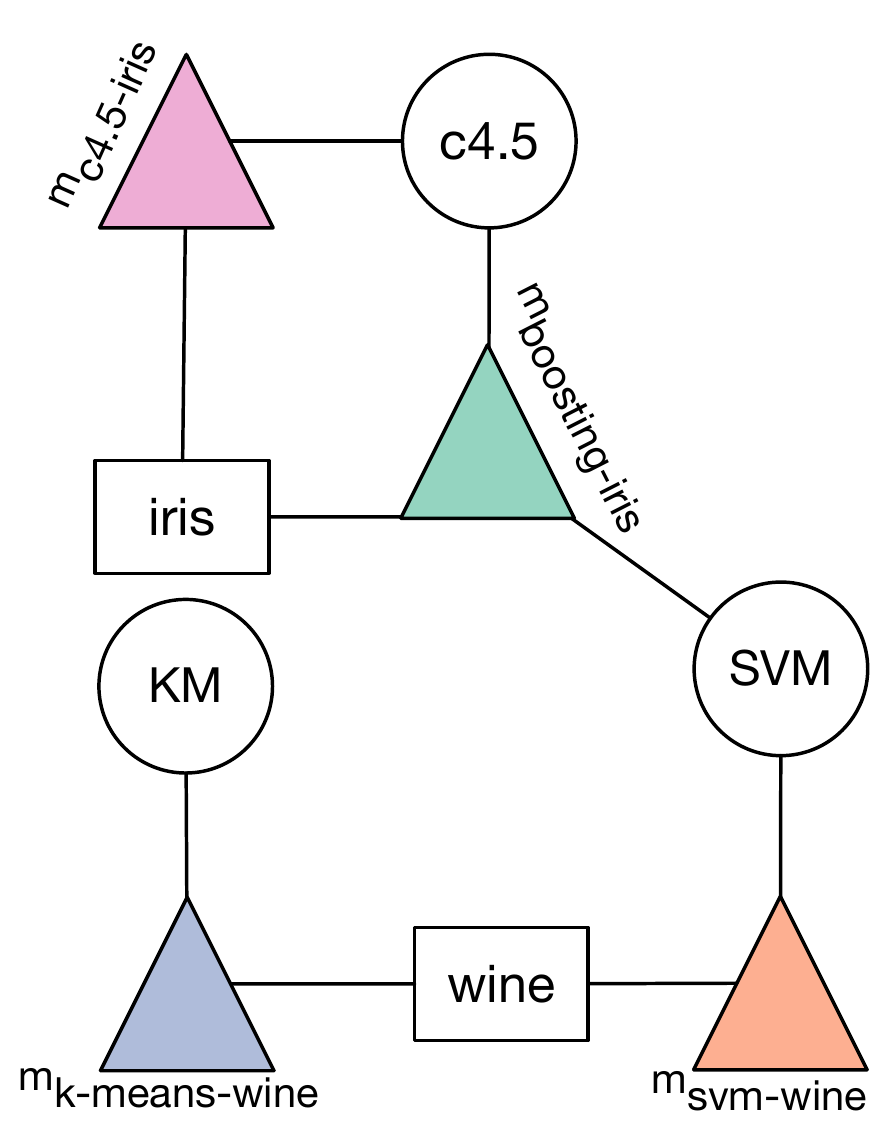}
		\caption{~}
		\label{fig:hyp_alt}
	\end{subfigure}
	\caption{An example machine learning system as a hypergraph}
	\label{fig:hyp_all}
\end{figure}

Before continuing to the detailed multi-agent modeling of the system, we introduce a set of definitions and assumptions together with the notations that are going to be used throughout this paper:
\begin{itemize}
	\item \textit{Algorithm} refers to any machine learning method, e.g. clustering or classification, in its abstract form without taking any specific data into account. An algorithm is denoted by tuple $a_i\left<name, type, P_{a_i}\right>$ where $P_{a_i}=\{(p_i,v_i)\}$ denotes the set of parameters that is used for the configuration of the algorithm, and $type\in \{$\textit{classification, clustering, regression, \dots}$\}$ specifies the category that this algorithm can belong to. The $type$ term has an organizational purpose and does not affect the functionality of the algorithm.
	\item \textit{Dataset} refers to all pre-stored sets of instances used for machine learning tasks. In this paper, we assume the data is already available and we do not contribute any method to data gathering or preparation. Similar to an algorithm, a dataset is represented by tuple $d_i\left<name,type,P_{d_i}\right>$ where $P_{d_i}$ specifies a set of access parameters in the dataset and $type$ denotes the chain categories that the data belongs to, as described above.
	\item \textit{Model} refers to the working implementation of a machine learning method based on the above-mentioned algorithms and datasets. For the sake of simplicity, we concentrate only on the models that implement a single algorithm on a single dataset in this paper. Such a model is denoted by tuple $m_i\left<name, a, d, P_{m_i}\right>$ where $a$ and $d$ are respectively the algorithm and dataset on which the model is based; and $P_{m_i}=\{(p_i,v_i)\}$ specifies the model-specific configuration parameters.   
	\item \textit{Query} refers to the machine learning requests that an end-user sends to the system. Throughout this paper, we assume that any query, training or testing, is specified by tuple $q_i\left<id, \Lambda_i, \Delta_i, O_i\right>$ where $id$ uniquely identifies the query; $\Lambda_i=\{(a_j,P_{a_j})\}$ and $\Delta_i=\{(d_k,P_{d_k})\}$ respectively denote the sets of algorithms and data that are used; and finally, $O$ specifies the configuration of the output that the end-user is interested in. For instance, the tuple 
	$$
	\left<00, \{(svm,\{(kernel,rbf)\}),(c4.5,\{\})\},\{(iris,\{\})\},\\\{\textit{type=test, format=visual,} measures=\{accuracy\}\} \right>
	$$
	specifies a query that tests algorithms $svm$, with \textit{rbf} kernel, and $c4.5$, with default parameters, on $iris$ dataset, with default parameters, and visually reports the accuracy as the result. It is worth noting that the number of different machine learning operations in a single query $q_i$, by this definition, is $|\Lambda_i|\times|\Delta_i|$ where $|...|$ denotes the set cardinality.
\end{itemize}

\subsection{Multi-agent Modeling}
The holistic view of the proposed model has a hierarchical structure of the agents that are specialized in managing data or processing a machine learning task. We use the aforementioned multi-modal representation of the machine learning hypergraph to express the structure of the multi-agent hierarchy, in terms of its components and the communication links. More specifically, each modal of the graph constructs a sub-tree that holds all the similar data/algorithm agents. For instance, figure~\ref{fig:hyp_tree} depicts the hierarchical modeling of the example in figure~\ref{fig:hyp_all}. Please note that we have used two different styles (solid and dashed lines) to draw the links between the structural components in this figure. This is for the sake of clarity in exhibiting the tree-like relationship among the components of the algorithm (ALG) and data (DATA) sub-structures. 

\begin{figure}
	\centering
	\includegraphics[width=.5\textwidth]{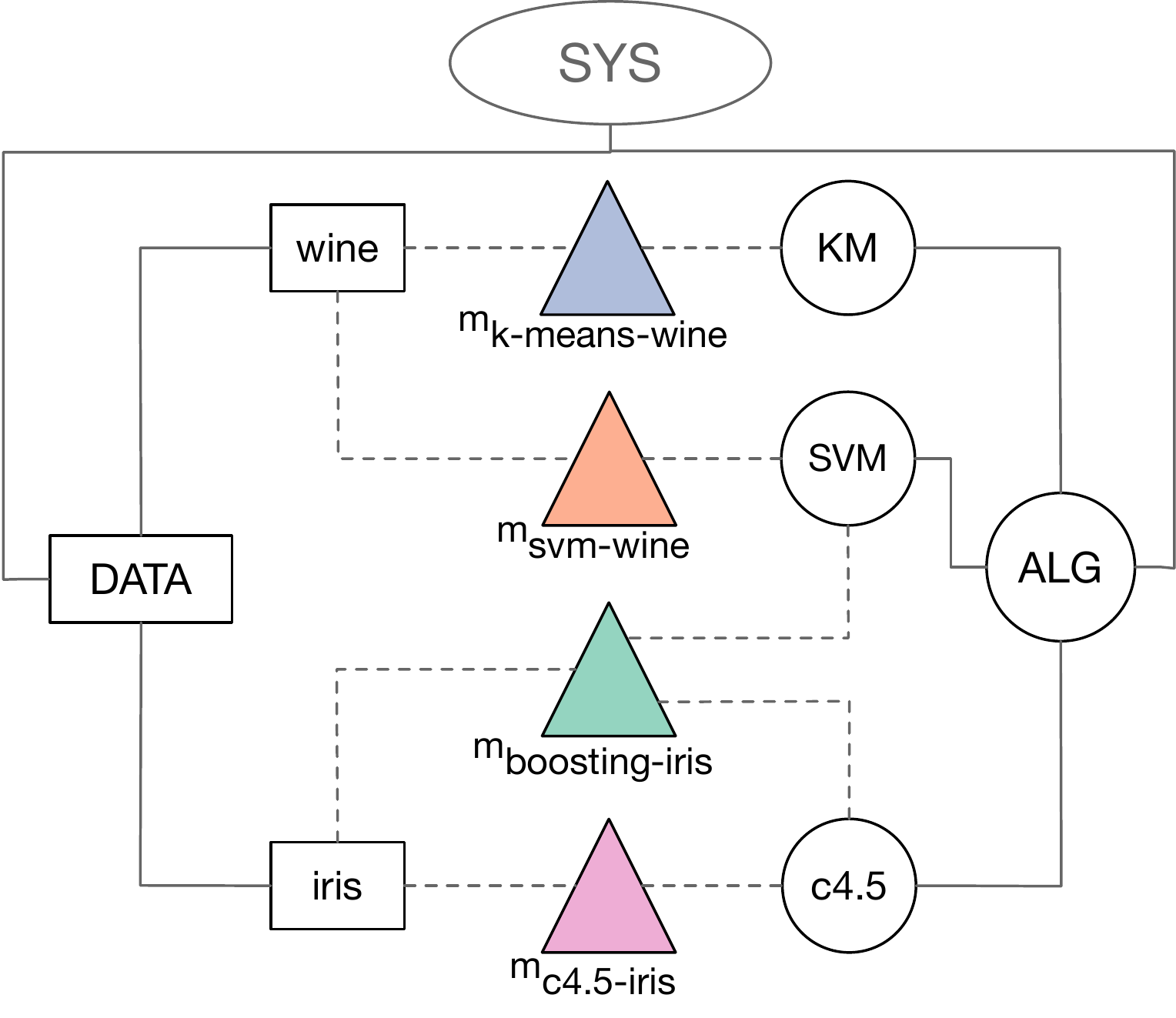}
	\caption{The hierarchical representation of the machine learning hypergraph example.}
	\label{fig:hyp_tree}
\end{figure}

It should also be emphasized that the depicted example is a very simplified system with only two levels of hierarchy and no specialization of the algorithms. In real-world scenarios, depending on the training/testing queries, the structure will get complicated and will allow more connection types. Figure~\ref{fig:hyp_tree_comp} tries to show an extended version of figure~\ref{fig:hyp_tree} system to reflect its flexibility. 

\begin{figure}
	\centering
	\includegraphics[width=.6\textwidth]{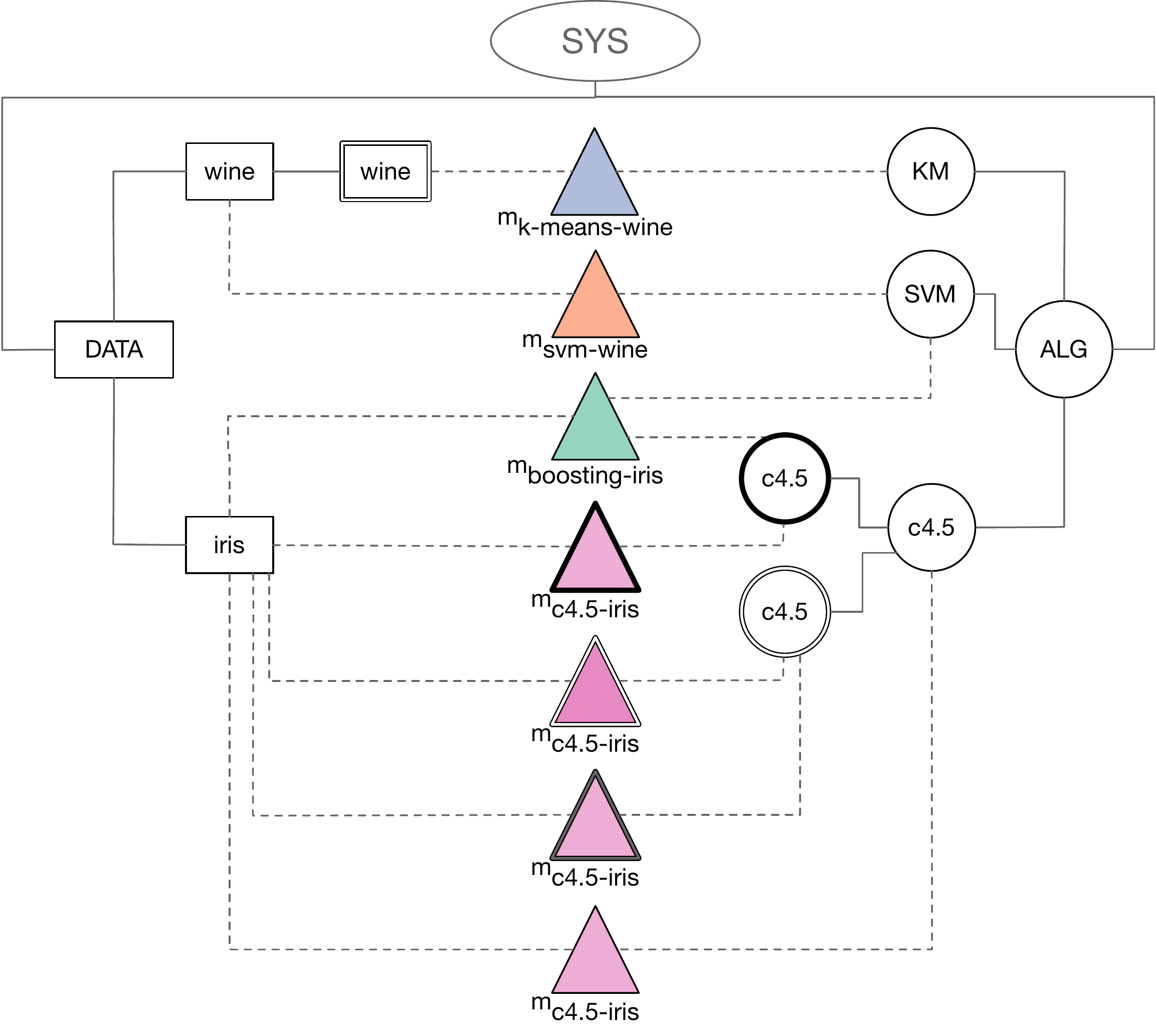}
	\caption{The hierarchical representation of the extended machine learning example.}
	\label{fig:hyp_tree_comp}
\end{figure}

The high flexibility and the multi-level structure of the proposed method requires an adaptable multi-agent organization model. In this paper, we use a self-similar and hierarchical structure called a holarchy to model the entire proposed machine learning platform. The next section provides a short introduction to holon and holarchies together with the details about how we have adopted them in our solution.

\subsubsection{The Holonic Machine Learning Model}
\paragraph{Holonic Multi-agent Systems}~\\
The concept of a holon is central to this paper and therefore a definition of it seems to be helpful before we continue. The term holon was introduced by Arthur Koestler \cite{koestler1968ghost} in order to explain the self-similar structure of biological and social systems. He made two key observations:
\begin{itemize}
	\item These systems evolve and grow to satisfy increasingly complex and changing needs by creating stable ``intermediate'' forms that are self-reliant and more capable than the initial systems.
	\item In living and organizational systems, it is generally difficult to distinguish between ``wholes'' and ``parts''. Put it another way, almost every distinguishable element is simultaneously a whole (an essentially autonomous body) and a part (an integrated section of a larger, more capable body).
\end{itemize}
The concepts of holonic systems have been successfully adopted and used in the design of organizational multi-agent systems. For instance, \cite{cossentino2010aspecs} introduces ASPECS, a step-by-step software process for modeling and engineering complex systems at different levels of details using a holonic organizational meta-model. GORMAS \cite{argente2009gormas}, on the other hand, provides analytical and design methodologies for open virtual organizations, including holonic agent-based systems. In multi-agent systems, the vision of holons is much closer to that of recursive or composed agents. A holon constitutes a way to gather local and global, individual, and collective points of view. Therefore, a holon is a self-similar structure composed of other holons as sub-structures. This hierarchical structure composed of holons is called a holarchy. Depending on the level of observation, a holon can be seen either as an autonomous atomic entity or as an organization of holons. In other words, a holon is a whole-part construct that is not only composed of other holons but it is, at the same time, a component of a higher-level holon.

In a holonic multi-agent system, we can distinguish between two main types of holons, namely head and body holons. All of the holons, which are a member of another holon (called super-holon), are considered to be body holons. These holons can be either atomic or composite and are performing the tasks that have been delegated to them. On the other hand, head holons act as the representatives of the holons they are members of. In other words, holons are observable to the outside world by means of these representatives. As the representatives, head holons manage the holons’ communication with the outside of the holon and coordinates the body holons in pursuit of the goals of the holon. Inside a holon, the force that keeps the heads and bodies is their commitments to the goal of the holon. It is worth noting that in this commitment, the relationships among the agents and holons are formed at runtime, in contrary to classical methods such as object-oriented programming in which they are expressed at the code level. More formally, according to \cite{fischer2003holonic}, for a MAS with the set $A$ of agents at time $t$, the set $H$ of all holons is defined recursively as follows:
\begin{itemize}
	\item Every instantiated agent can be considered as an atomic holon.
	\item $h = (\textit{Head}, \textit{Sub-holons}, C) \in H$, where $\textit{Sub-holons}\in H\setminus \emptyset$ is the set of holons that participate in $h$; $\textit{Head}\subseteq \textit{Sub-holons}$ is a non-empty set of holons that has the aforementioned managerial responsibilities; and $C\subseteq Commitments$ defines the relationship inside the holon and is agreed on by all holons $h'\in \textit{Sub-holons}$ at the time of joining the holon $h$. These commitments keep the members inside the holon.
\end{itemize}

According to the definitions above, a holon $h$ is observed by the outside world of $h$ like any other agent in $A$. Only at a closer inspection, it may turn out that $h$ is constructed from (or represents) a set of agents. Like traditional agents, any holon has a unique identification. This facilitates communication among the holons by just sending messages to their addresses. Handling these messages is one of the responsibilities of the head holon(s). Given the holon $h = (Head, \{h_1, ..., h_n\}, C)$, we call $h_1, ..., h_n$ the sub-holons of $h$, and $h$ the super-holon of $h_1, ..., h_n$. The set $Body= \{h_1, ..., h_n\}\setminus Head$ (the complement of Head) is the set of  sub-holons that are not allowed to represent holon $h$. Holons are allowed to engage in several different super-holons with noncontradictory goals.  

\paragraph{The Adoption of the Holonic Concepts}~\\
Based on the formalization of the machine learning problem that we presented in section~\ref{sec:formalization}, we define a new holon for each of the structural components of the machine learning hierarchy. Generally, these holons can be categorized as follows:
\begin{itemize}
	\item \textbf{SysH}, as the system holon, acts as the representative of the entire machine learning system to the outside world. This holon is the ultimate super-holon that includes all of the other parts of the system.
	\item \textbf{AbsH}, as the set of abstract holons, acts as the container for all holons of the same functionality. These types of composite holons do not directly perform any machine learning task or manage a dataset, and hence are not connected to any model holon (ModH). Instead, they play a managerial role in handling the queries and the results. There are at least two pre-defined abstract holons in the proposed platform: \textit{abstract data} and \textit{abstract algorithm} holons that while are the sub-holons of the \textit{SysH} act as the ultimate parent of all existing data and algorithm holons in the system, respectively. These holons are specified by ``DATA'' and ``ALG'' in figure~\ref{fig:hyp_tree_comp}. The \textit{type} parameter of an algorithm or data tuples, can be used to define a hierarchy of the abstract holons. That is, the system might have other abstract holons for categorization and managerial purposes, such as an abstract holon to hold all the decision tree algorithms or to manage all the categorical data. If this is the case, they should be properly placed in the hierarchy based on their relationship with the other abstract and concrete holons. As we explain later, we only use the default abstract holons, together with the ones that are dynamically created by the system, as we prefer to not enforce any restrictions to the definition of queries. 
	\item \textbf{DatH} refers to the set of non-abstract holons that hold and manage datasets. A data holon is atomic and is part of an AbsH. Moreover, when a dataset is not used in the training of any machine learning model, its corresponding DatH is not connected to any model holon, ModH.
	\item \textbf{AlgH} refers to the set of non-abstract holons that contain the implementation and configuration details of a data mining algorithm. Similar to a DatH, an AlgH is atomic and might or might not maintain links with model holons.
	\item \textbf{ModH}, as the set of model holons, corresponds to a realization of a machine learning algorithm used on a specific dataset. A ModH is an atomic holon constructed by an AlgH and is the sub-holon of a DataH in addition to its creator AlgH. It is also assumed that a ModH does not interact with the other model holons. 
\end{itemize}

In order to show the hierarchical relationship between these holons, let's assume that $\tensor*[^U_t]{H}{^l_i}$ denotes the $i$-th holon of type $t\in T=\{s,a,d,m\}$ at level $l$ of the holarchy where $s, a, d$, and $m$ stand for \textit{system, algorithm, data}, and \textit{model} respectively; and $U$ is the set containing the indexes of its super-holons. Numbering the levels from top of the holarchy starting from 0, we have: 
\begin{equation}
	SysH= {}_{s}^{\emptyset}H^0_0
\end{equation}
\begin{equation}
	AbsH= \{{}_{t}^{\{u\}}H^{l}_i:t\in \{a,d\}\; \wedge\; l>0\; \wedge\; {}_{t}^{\{u^\prime\}}H^{l-1}_u\in (AbsH\cup SysH)\}
\end{equation}
\begin{equation}
	DatH= \{{}_{d}^{\{u\}}H^{l}_i:l>1\; \wedge\; {}_{d}^{\{u^\prime\}}H^{l-1}_u\in (DatH\cup AbsH)\}
\end{equation}
\begin{equation}
	AlgH= \{{}_{a}^{\{u\}}H^{l}_i:l>1\; \wedge\; {}_{a}^{\{u^\prime\}}H^{l-1}_u\in (AlgH\cup AbsH)\}
\end{equation}
\begin{equation}
	ModH= \{{}_{m}^{\{u,u^\prime\}}H^{l}_i:l>2\; \wedge\; {}_{a}^{\{x\}}H^{l-1}_u \in AlgH\; \wedge\; {}_{d}^{\{z\}}H^{l-1}_{u^\prime} \in DatH\}
\end{equation}
and hierarchical relationship between the holons are defined as follows:
\begin{equation}
	\label{eq:1}
	{}_{s}^{\emptyset}H^0_0=\{{}_{a}^{\{0\}}H^1_1, {}_{d}^{\{0\}}H^1_1\}
\end{equation}
\begin{equation}
	\label{eq:2}
	{}_{t}^{U}H^l_i=\{{}_{t}^{\{i\}}H^{l+1}_j:  l\ge 1\; \wedge\; t\in \{a,d\}\}\; \cup\; \{{}_{m}^{V}H^k_j: i\in V\}
\end{equation}
Please note that, these statements merely define the holonic relationships and do not identify the complete set of members that a holon might have. In other words, a holon might potentially comprise additional agents/holons that help it carry out the internal processes. We will clearly specify such additional members whenever we assume and use any.

Apart from its members, a holon comprises several other parts that enable it to process and answer the queries. The general internal architecture that is commonly used by the above-defined holons is shown in figure~\ref{fig:holon_arc}. In this figure, the arrows show the direction of the information provided between the components; the \textit{Head Agent} (HA), administers the inter- and intra- holon interactions; the \textit{Knowledge Base} (KB) stores the data or the details of the algorithm; the \textit{Directory Facilitator} (DF) maintains the information of and accesses to all of the super-holons and members, including sub-holons and the other members; the \textit{Memory} (ME) component stores the history of the results and useful management logs; the \textit{Skills \& Capabilities} (S\&C) element refer to the abilities that distinguish a specific holon from the similar ones; and finally, the \textit{Other Members} (OM) part contains the utility agents that are employed by the head to handle the internal tasks, such as query agents, result agents, etc. It should be noted that the holarchical relationships are maintained using DF, and there is no designated component to hold the sub-holons. This is for the sake of keeping the holon as small as possible and make it to be easily distributed over multiple devices. Furthermore, we assume that the identity and category information of the holon is handled by its HA. Depending on the type of the holon, the suggested architecture might miss some parts. Table~\ref{tbl:parts} provides the included components of each holon type in the proposed platform.
\begin{figure}
	\centering
	\includegraphics[width=.3\textwidth]{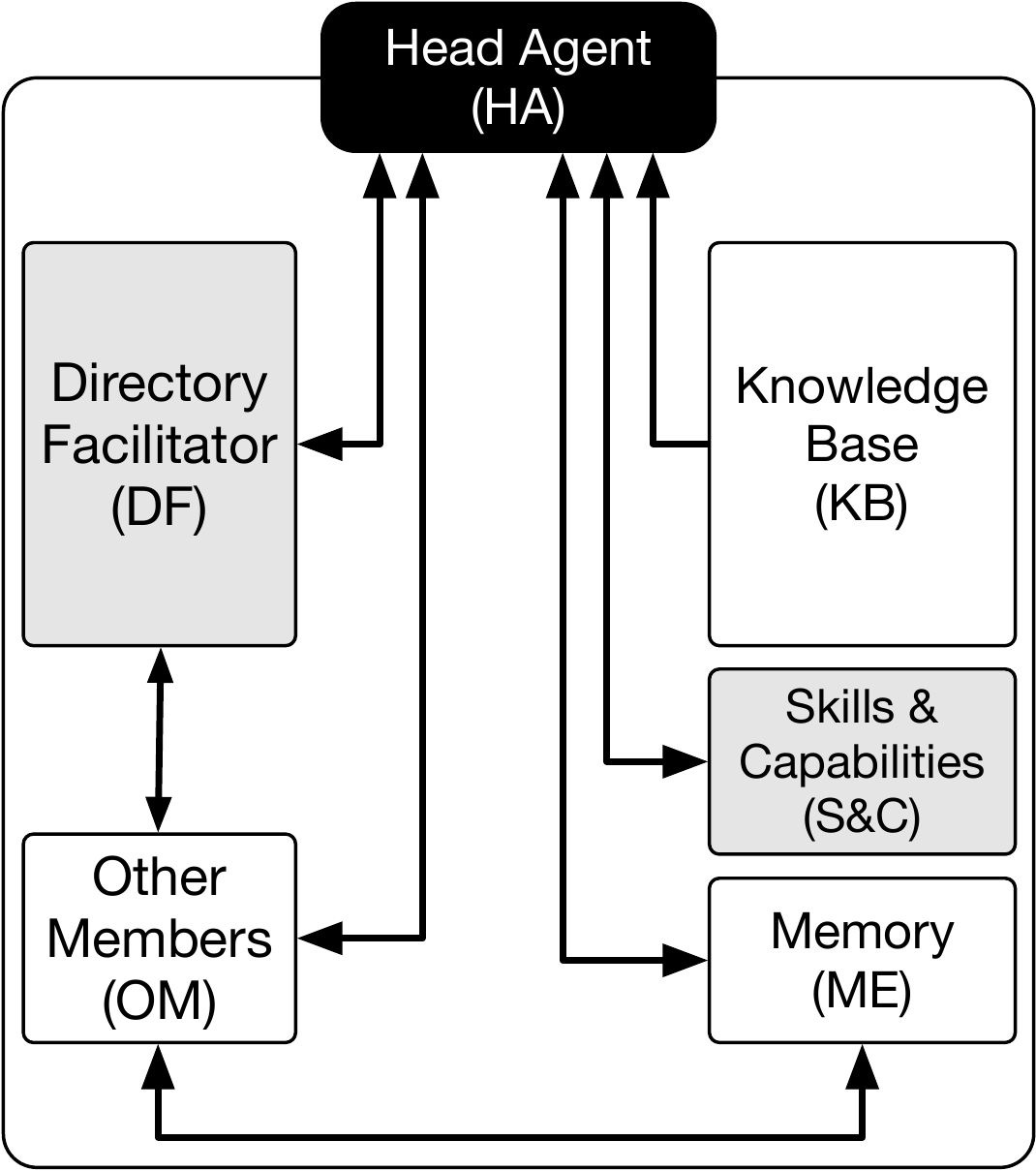}
	\caption{The general architecture of a holon.}
	\label{fig:holon_arc}
\end{figure}

\begin{table}\centering
	\caption{The parts that each holon type includes.}
	\label{tbl:parts}
		
		\begin{tabular}{@{}lcccccc@{}}\toprule
			&HA & DF & OM & KB & S\&C & ME\\ \midrule
			SysH & $\bullet$ & $\bullet$ & $\bullet$ & -- & -- & $\bullet$ \\
			AbsH & $\bullet$ & $\bullet$ & -- & -- & $\bullet$ & $\bullet$\\
			AlgH & $\bullet$ & $\bullet$ & $\bullet$ & $\bullet$ & $\bullet$ & $\bullet$ \\
			DatH & $\bullet$ & $\bullet$ & $\bullet$ & $\bullet$ & $\bullet$ & $\bullet$ \\
			ModH & $\bullet$ & $\bullet$ & -- & $\bullet$ & $\bullet$ & $\bullet$ \\
			\bottomrule
		\end{tabular}
	\end{table}
	
	As it was mentioned above, the skills and capabilities component distinguishes a holon from its counterparts. In fact, this entity plays a canonical part in the mining operations of the holon and specifying its position inside of the holarchy. In the proposed platform, these terms are defined as follows:
	\begin{itemize}
		\item \textbf{Capability}, denoted by $C$, refers to the innate machine learning ability of a holon, considering itself and all of its members. We use the configuration parameters of the data/algorithm that holon represents to define its set of capabilities. In case of atomic holons, $C_{{}_{t}^{U}H^l_i} = P_{t_i}$ where $P_{t_i}$ indicates the configuration parameters of the corresponding algorithm/data/model entity. For composite abstract holons, however, the capability is defined as the parametric sum (discussed later in definition~\ref{def:ps}) of the capabilities of all of its non-model sub-holons. 
		\item \textbf{Skill} refers to the set of the specific expertise of the holon. Unlike capabilities that exist intrinsically from the birth of the holon, skills are acquired as soon as the holon is involved in a practical machine learning operation, i.e. the time a child of \textit{ModH} type is spawned. That being said, we update the skill set of \textit{ModH} holons and the other type of holons differently. Formally, if the skill set of the atomic holon ${}_{t}^{U}H^l_i$ is denoted by $S_{{}_{t}^{U}H^l_i}$, we have:
		\begin{equation}
			\label{eq:skill-model}
			S_{{}_{m}^{U}H^l_i}=\bigcup_{u^\prime\in U} C_{{}_{t}^{Y}H^{l^\prime}_{u^\prime}};\quad t\in \{a,d\}
		\end{equation}
		\begin{equation}
			\label{eq:skill-other}
			S_{{}_{t}^{U}H^l_i}=\bigcup_{Y}\left\{\{C_{{}_{m}^{Y}H^{\lambda}_j}; i\in Y\}\setminus C_{{}_{t}^{U}H^l_i}\right\}
		\end{equation}
		Additionally, for the case of an abstract data/algorithm holon, the skill set is the union of the skills of its sub-holons. That is:
		\begin{equation}
			\label{eq:skill-abstract}
			S_{{}_{t}^{u}H^l_i}=\bigcup_{u^\prime} S_{{}_{t}^{\{i\}}H^{l+1}_{u^\prime}}
		\end{equation}
		By way of explanation, the skill set of the model holon is the union of the capabilities of its super-holons, and the skill set of the atomic data/algorithm holons is the union of the skills of its model sub-holons without its own capability set. Finally, the skill set of a composite abstract holon comprises the combination of subordinate skills. The toy example of figure~\ref{fig:skill-exm} demonstrates how the capabilities and skills are set in a partial holarchy.
		
		\begin{figure}
			\centering
			\includegraphics[width=\textwidth]{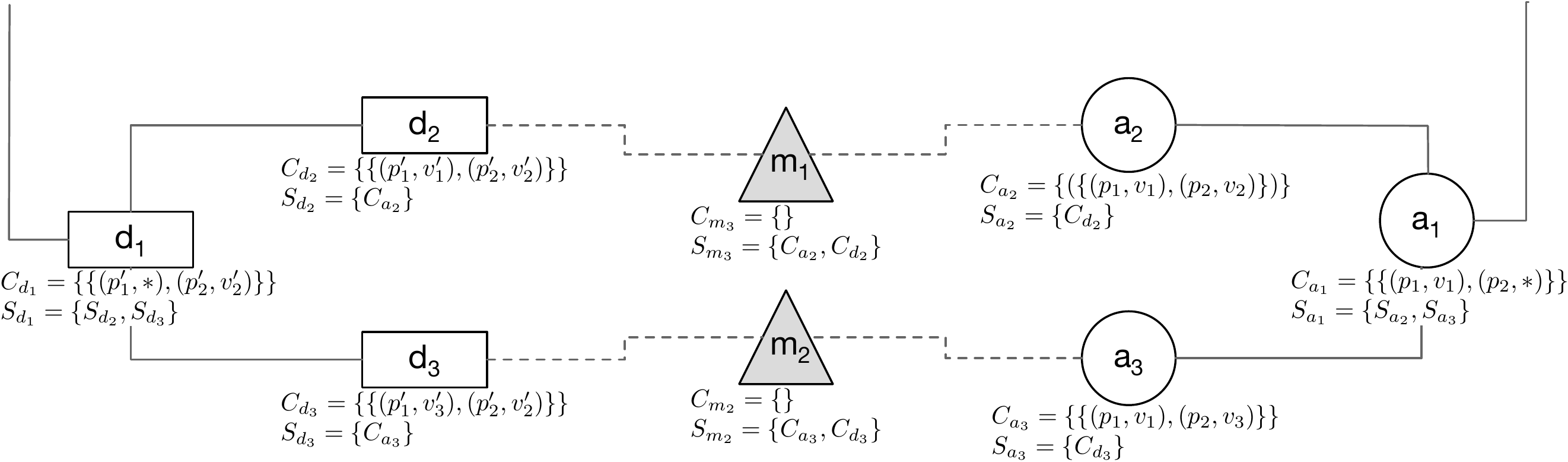}
			\caption{Capabilities and skills in a simple holarchy example.}
			\label{fig:skill-exm}
		\end{figure}
	\end{itemize}
	It should be emphasized that these terms are defined with respect to machine learning tasks and does not take the intrinsic abilities of the other members of the holon, such as processing the queries, generating reports, etc.
	
	\paragraph{The Construction of the Holarchy}~\\
	One of the key steps in utilizing a holonic multi-agent model is the initial construction of the hierarchy. The pattern of the holon arrangement in the holarchy plays a critical role in the performance of the system and its dynamic adaptation to the changes of the environment\cite{esmaeili2017socially,esmaeili2016impact, esmaeili2019towards}. Although a holonic structure can be designed and hand-arranged by an expert, there are numerous research endeavors in the literature that have concentrated on the automatic organization of holonic multi-agent systems -- socially based method proposed in \cite{esmaeili2017socially, esmaeili2019towards}, the RIO\cite{hilaire2000formal} based approach in \cite{hilaire2008adaptative}, and the Petri net-based model reported in \cite{ciufudean2011artificial} to name a few. In this paper, The holarchy is initially composed of a \textit{SysH} and two \textit{AbsH}s, namely ALG and DATA, to accommodate all algorithms and data of the system respectively. Then it dynamically continues growing as new machine learning queries are sent to the system. The detailed algorithm makes use of a few new operators and symbols defines as follows:
	
	\begin{definition}[Parametric Set]\label{def:pset}
		A parametric set $P$ is a set of parameter-value ordered pairs $(p_i,v_i)$ such that $\exists (p_j, v_j)\in P, p_i=p_j\Longrightarrow v_i=v_j$. In other words, there is only one pair for any parameter $p_i$ in the set.
	\end{definition}
	
	\begin{definition}[Parametric General Symbol]\label{def:pg}
		A parametric general symbol, denoted by $\star$, is a placeholder for all the available values of a parameter. For instance, the ordered pair $(learning\_rate,\star)$ implies the set of all available values for parameter $learning\_rate$. Throughout this paper, we call a pair general if it the symbol appears as its value; and similarly, we call a set general if it contains at least one general pair.
	\end{definition}
	
	\begin{definition}[Parametric Congruence]\label{def:pe}
		Two ordered pairs $(p,v)$ and $(p^\prime,v^\prime)$ are called parametric congruent, written as $(p, v)\accentset{\star}{\cong}(p^\prime,v^\prime)$, if and only if $p = p^\prime$. Similarly, congruence for two parameter sets $P$ and $P^\prime$ is defined as follows:
		\begin{equation}
			P\accentset{\star}{\cong}P^\prime \iff |P|=|P^\prime|\; \wedge\; \forall (p,v)\in P\; \exists (p^\prime,v^\prime)\in P^\prime: (p,v)\accentset{\star}{\cong}(p^\prime,v^\prime)
		\end{equation}
	\end{definition}
	
	\begin{definition}[Parametric Inequality]\label{def:pin}
		The ordered pair $(p, v)$ is parametrically less than or equal to $(p^\prime, v\prime)$, denoted by $(p,v) \accentset{\star}{\le}(p^\prime, v\prime)$, if and only if they are congruent and $v = v^\prime \vee v = \star \vee v^\prime = \star$. Similarly, for two sets $P$ and $P^\prime$ we have:
		\begin{equation}
			P\accentset{\star}{\le}P^\prime \iff \forall (p,v)\in P\; \exists (p^\prime,v^\prime)\in P^\prime: (p,v)\accentset{\star}{\le}(p^\prime,v^\prime)
		\end{equation}
		Please note that, $P$ and $P\prime$ sets do not necessarily need to be parametric congruent. For instance, $ \emptyset \accentset{\star}{\le} \{(p_1,v_1),(p_2, v_2)\};\allowbreak ~~ \{(p_1,v_1)\} \accentset{\star}{\le} \{(p_1,*),(p_2, v_2)\}; ~~ \{(p_1,v_1),(p_2, *)\} \accentset{\star}{\le} \{(p_1,*),(p_2, v_2)\}$  all yield true, whereas $\{(p_1,v_1),(p_2, v_2)\} \accentset{\star}{\le} \{(p_1,*),\allowbreak (p_2, v_3)\}$ results in false.
	\end{definition}

	\begin{definition}[Parametric Sum]\label{def:ps}
		Parametric sum, denoted by $\accentset{\star}{+}$, is a binary operator defined on parametric congruent pairs or sets as follows:
		\begin{equation}
			(p,v) \accentset{\star}{+} (p,v^\prime) = \begin{cases}
				(p,v) & \text{if $v = v^\prime$}\\
				(p,\star) & \text{if $v \ne v^\prime$}
			\end{cases}
		\end{equation}
		\begin{equation}
			P \accentset{\star}{+} P^\prime = \bigcup_{\substack{(p,v)\in P\\(p,v^\prime)\in P^\prime}}\{(p,v) \accentset{\star}{+} (p,v^\prime)\}
		\end{equation}
		Furthermore, parametric sum has the additive identity property over sets, i.e $\emptyset \accentset{\star}{+} P^\prime =P^\prime$.
	\end{definition}
	
	\begin{definition}[Parametric Similarity Ratio]\label{def:pd}
		Parametric similarity ratio, denoted by $\accentset{\star}{\sim}$, is a bivariate function that quantifies the similarity between two parametric congruent pairs or sets, in terms of their common values. The range of $\accentset{\star}{\sim}$ lies in $(0,1]$ and is predefined for all parameter values. In this paper, we use the following definition:
		\begin{equation}
			\accentset{\star}{\sim}((p,v), (p,v^\prime))=\begin{cases}
				1 & \text{if $v=v^\prime$}\\
				\alpha & \text{if $v\ne v^\prime \wedge (v=\ast \veebar\: v^\prime=\ast)$}\\
				\beta & \text{if $v\ne v^\prime \wedge (v\ne\ast \wedge\: v^\prime\ne\ast)$}\\
			\end{cases}
		\end{equation}
		such that, $0<\beta<\alpha<1$. The parametric difference of two parameter sets $P$ and $P^\prime$ is defined as follows:
		\begin{equation}
			\accentset{\star}{\sim}(P, P^\prime)=\frac{\displaystyle \sum_{v=v^\prime}\accentset{\star}{\sim}((p,v), (p,v^\prime)) + \prod_{v\ne v^\prime}\accentset{\star}{\sim}((p,v), (p,v^\prime))}{\vert P\vert} 
		\end{equation}
	\end{definition}
	where $(p,v)\in P; (p,v^\prime)\in P^\prime$; and $\vert\dots\vert$ denotes the set cardinality. That being said, the possible values of the parametric similarity score of two sets lie in range $(\frac{\beta^{|P|}}{|P|}, 1]$. In this paper, we have set $\alpha=0.5$ and $\beta=0.1$.
	
	It can be easily shown that the operators explained in definitions \ref{def:pe}, \ref{def:ps}, and \ref{def:pd} all commutative property.
	
	As it was stated before, the construction of the holarchy begins with the initial SYS = \{DATA, ALG\} holons. For now, let's assume that the user query is properly processed by a utility member of SYS holon and is sent to its both sub-holons. Furthermore, at the moment, we assume that all the queries are for building the holarchy composed of the data and algorithm holons, and we postpone the creation of model holons until we present the details of training and testing processes. Finally, since the construction of the structures under the algorithm and data super-holons are the same, we explain the holonification method for the algorithms only. 
	
	The very first thing that the abstract ALG holon checks, upon receiving the query from the SYS, is the name of the algorithm to direct the request to a proper path down the holarchy. For this purpose, it initiates a bidding process based on Contract Net Protocol (CNP)\cite{cnp1980} and asks its sub-holons for their proposals. The proposal of its immediate sub-holon $h={}_{a}^{\{1\}}H_{i}^2$ is basically the result of $\accentset{\star}{\sim}((name,{a^q_{name}}), (name,h_{name}))$ where $a^q_{name}$ and $h_{name}$ are the names of the new algorithm requested to be added by the query and the name of the holon $h$ respectively. Having received the proposals, the ALG holon chooses the sub-holon with proposal value 1 -- there will be only one such proposal -- to forward the query to. If there is no such proposal, ALG spawns a new holon to represent the new algorithm's specifications. Algorithm~\ref{alg:start} presents the details of the process. The names of the variables and functions are chosen to be self-explanatory, and we have left comments wherever further explanations are needed. 
	
	\begin{algorithm}[H]
		\label{alg:start}
		\SetAlgoLined
		\KwIn{$a_i\left<name_{a_i}, type_{a_i}, P_{a_i}\right>$~~~~\tcp*[h]{an algorithm}}
		$\textsf{best} \gets \textsf{myself}$\tcp*{myself refers to the holon that calls this function}
		\ForEach{${}_{a}^{\{1\}}H_{i}^2 \in {}_{a}^{\{0\}}H_{i}^1$}{
			$\textsf{proposal} \gets \textsf{CFP}({}_{a}^{\{1\}}H_{i}^2,(name,name_{a_i}))$\tcp*{Call For Proposal}
			\If{$\textsf{proposal} = 1$}{
				$\textsf{best} \gets {}_{a}^{\{1\}}H_{i}^2$\;
				break\;
			}
		}
		\tcp{ask the ``best'' holon to add algorithm $a_i$ as new holon}
		\textsf{Ask(best, ``ADD'', $a_i$)}\;\label{alg:1:ask} 
		\caption{The new algorithm initiation function.}
	\end{algorithm}
	
	When an algorithm holon is asked to add a new holon (line~\ref{alg:1:ask} of algorithm~\ref{alg:start}), it runs the \textsf{AddAlgorithm} function, in which, a slightly different set of steps are followed. The first difference pertains to the way proposals are generated and chosen. Unlike before, the holons use their capability sets to calculate their similarity scores, and the one with the greater value is chosen. Furthermore, if the chosen holon is atomic, a new super-holon is created to contain both the chosen and the new holon. Otherwise, a new sub-holon is created for the algorithm under the current holon, or the new algorithm info will be passed down recursively until it is handled properly. The details of the process is presented in algorithm~\ref{alg:add}. In this algorithm, the arguments of function \textsf{CreateNewHolon} (lines~\ref{alg:2:createholon1}, \ref{alg:2:createholon2} and \ref{alg:2:createholon3}) are the name of the holon, the reference to its super-holon, the set of its capabilities, and finally the set of its skills, respectively. Additionally, there are a few utility functions that are called regularly in this algorithm. Function \textsf{Ask} in lines~\ref{alg:2:ask1}, \ref{alg:2:ask2}, and \ref{alg:2:ask3}, sends a request to the holon, specified in the first argument, to run an operation, specified in the second argument, based on the information that are provided in the remaining arguments. Upon receiving such a request, the holons call appropriate functions and answer the request accordingly. Functions \textsf{JoinHolon} and \textsf{UpdateCapability}, on the other hand, are used by the holon to properly move from one super-holon to another, and update its capability lists respectively. The details of these two functions are presented in algorithm~\ref{alg:join-update}. In this algorithm, function \textsf{Inform} is used to let a holon know about some facts. In the case of the presented functions, it is used to inform the super-holon about the updates in capabilities so that the super-holon makes the necessary updates accordingly. In all of these algorithms, it is assumed that holon ${}_{a}^{\{u\}}H_{i}^{l>0}$ is the one that calls and uses the functions.
	
	\begin{algorithm}
		\caption{The new algorithm addition function.}
		\label{alg:add}
		\SetAlgoLined
		\KwIn{ 
			$a_i\left<name_{a_i}, type_{a_i}, P_{a_i}\right>$~~~~\tcp*[h]{an algorithm}
		}
		\SetKwFunction{FMain}{AddAlgorithm}
		\SetKwProg{Fn}{Function}{}{}
		\Fn{\FMain{\textsf{$a_i$}}}{
			\tcp{myself refers to the holon that calls this function}
			\uIf(\tcp*[h]{I am ALG holon}){\textsf{myself.LEVEL $=1$}}{
				$\textsf{CreateNewHolon(}name_{a_i}, \textsf{myself}, P_{a_i}, \{\})$\;\label{alg:2:createholon1}
			}
			\uElseIf{\textsf{AmIAtomic()}}{
				\eIf(\tcp*[h]{The algorithm already exists}){$\boxed{\accentset{\star}{\sim}}(C_{myself}, P_{a_i}) = 1$}{
					\Return
				}{
					$\textsf{new\_super} \gets \textsf{Ask(myself.SUPER, ``CREATE-HOLON'', myself.C, myself.S)}$\;\label{alg:2:ask1}
					$\textsf{new\_holon} \gets \textsf{CreateNewHolon(}name_{a_i}, \emptyset, P_{a_i}, \{\}\textsf{)}$\;\label{alg:2:createholon2}
					\textsf{JoinHolon(new\_super)}\;
					\textsf{Ask(new\_holon, ``JOIN'', new\_super)}\;\label{alg:2:ask2}
				}
			}
			\Else{
				$\textsf{best} \gets \textsf{myself}$\;
				$\varsigma_{\max} \gets \boxed{\accentset{\star}{\sim}}(C_{myself}, P_{a_i})$\;
				\ForEach{${}_{a}^{\{i\}}H_{j}^{l+1} \in {}_{a}^{\{u\}}H_{i}^{l}$}{
					$\textsf{proposal} \gets \textsf{CFP}({}_{a}^{\{i\}}H_{j}^{l+1},P_{a_i})$\tcp*{Call For Proposal}
					\If{$\textsf{proposal} > \varsigma_{\max}$}{
						$\varsigma_{\max} \gets \textsf{proposal}$\;
						$\textsf{best} \gets {}_{a}^{\{i\}}H_{j}^{l+1}$\;
					}
				}
				\eIf{\textsf{best = myself}}{
					$\textsf{CreateNewHolon(myself.NAME, myself, } P_{a_i}, \{\})$\;\label{alg:2:createholon3}
					$\textsf{UpdateCapability(}P_{a_i}\textsf{)}$\;
				}{
					\textsf{Ask(best, ``ADD'', $a_i$)}\;\label{alg:2:ask3}
				}
			}
		}
		\textbf{End Function}
	\end{algorithm}
	
	\begin{algorithm}
		\caption{The algorithm for joining a holon and updating the capability.}
		\label{alg:join-update}
		\SetAlgoLined
		% \KwIn{ 
			%     $new\text{-}super\text{-}holon$
			% }
		\SetKwFunction{FJoin}{JoinHolon}
		\SetKwProg{Fn}{Function}{}{}
		\Fn{\FJoin{\textsf{$new\text{-}super\text{-}holon$}}}{
			\tcp{myself refers to the holon that calls this function}
			\textsf{Update(``SUPER'', $new\text{-}super\text{-}holon$)}\;
			\textsf{Inform($new\text{-}super\text{-}holon$, ``CAPABILITY'', myself.C)}\;
		}
		\textbf{End Function}
		\BlankLine
		\BlankLine
		%  \KwIn{ 
			%         $new\text{-}capability$
			%     }
		\SetKwFunction{FUpCap}{UpdateCapability}
		\SetKwProg{Fn}{Function}{}{}
		\Fn{\FUpCap{\textsf{$new\text{-}capability$}}}{
			\tcp{myself refers to the holon that calls this function}
			\If(\tcp*[h]{I am not ALG holon}){myself.LEVEL > 1}{
				$\textsf{myself.C} \gets \textsf{myself.C}\; \boxed{\accentset{\star}{+}}\; new\text{-}capability$\;
				\textsf{Inform(myself.SUPER, ``CAPABILITY'', myself.C)}\;
			}
		}
		\textbf{End Function}
	\end{algorithm}
	
	Figure~\ref{fig:addition-exm} demonstrates a step by step process of algorithms~\ref{alg:start} and \ref{alg:add} in a simple example case. For the sake of clarity, we have shown only the names and the values of the configuration parameters for each input algorithm (in red color). Furthermore, the identity and the name of the created holons are given in \textit{id:name} format inside of the nodes to help readers understand the order of the created holons. In part \ref{fig:addition-exm-a} of this figure, the ALG holon is asked to add algorithm $X$ with parameter values $\{a,b,c,d\}$ to its holarchy. Since ALG does not have any subs, a new holon is created as its sub-holon to represent algorithm $X$ (part \ref{fig:addition-exm-b}). When the system is asked to add algorithm $Y$ with parameter values $\{o,p,q\}$, ALG calls for proposal from its sub-holons, \textit{1:X} in this example, and since the proposal value is not 0 (due to the dissimilarity of its name), the new sub-holon \textit{2:Y} is created. In part \ref{fig:addition-exm-c}, the resulted holarchy is requested to add a new algorithm with name $X$ and parameter values $\{a,e,c,d\}$. First of all, ALG locates the sub-holon that pertains to the algorithms of name $X$, and then forwards the algorithm specifications to that holon. Upon receiving the request, holon \textit{1:X} calculates the parametric similarity ($\accentset{\star}{\sim}$ between its capabilities and the parameter set of the algorithm. The value \textcolor{blue}{\fbox{0.77}} printed in a box above the corresponding node in part \ref{fig:addition-exm-c} represents the value of similarity. Since there are not sub-holons to ask for their proposals, super-holon \textit{3:X} is created and both \textit{1:X} and the newly created holons for that algorithm \textit{4:X} join that super-holon (part \ref{fig:addition-exm-d}), and the capabilities of the holons are updated as explained in algorithm~\ref{alg:add}. The remaining parts of the figure show the same set of steps to handle three more new incoming algorithm info queries, and hence, for the sake of space, we do not explain them further.
	
	\begin{figure}[!htbp]
		\centering
		\begin{subfigure}[b]{0.3\textwidth}
			\centering
			\includegraphics[height=.15\textheight]{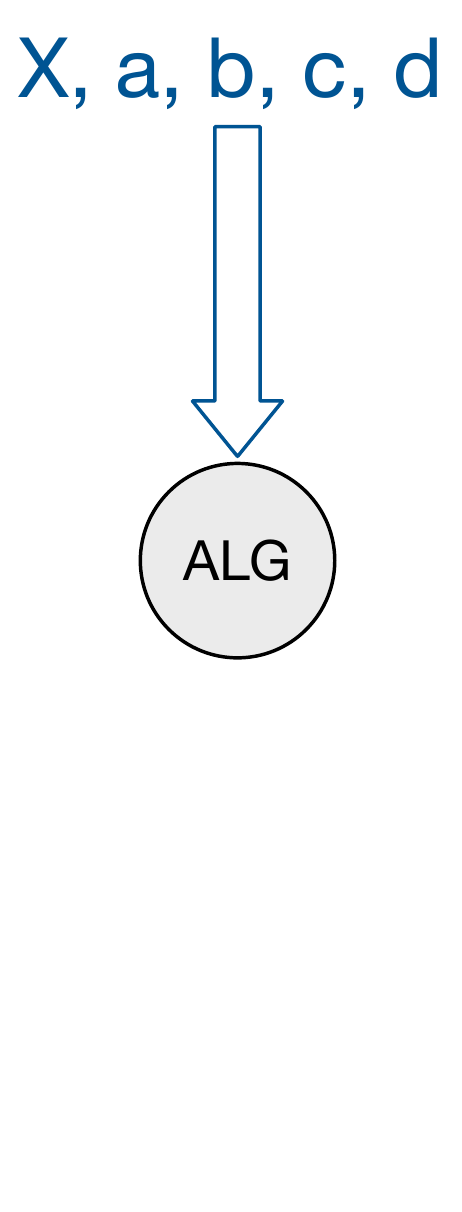}
			\caption{~}
			\label{fig:addition-exm-a}
		\end{subfigure}
		\begin{subfigure}[b]{0.3\textwidth}
			\centering
			\includegraphics[height=.15\textheight]{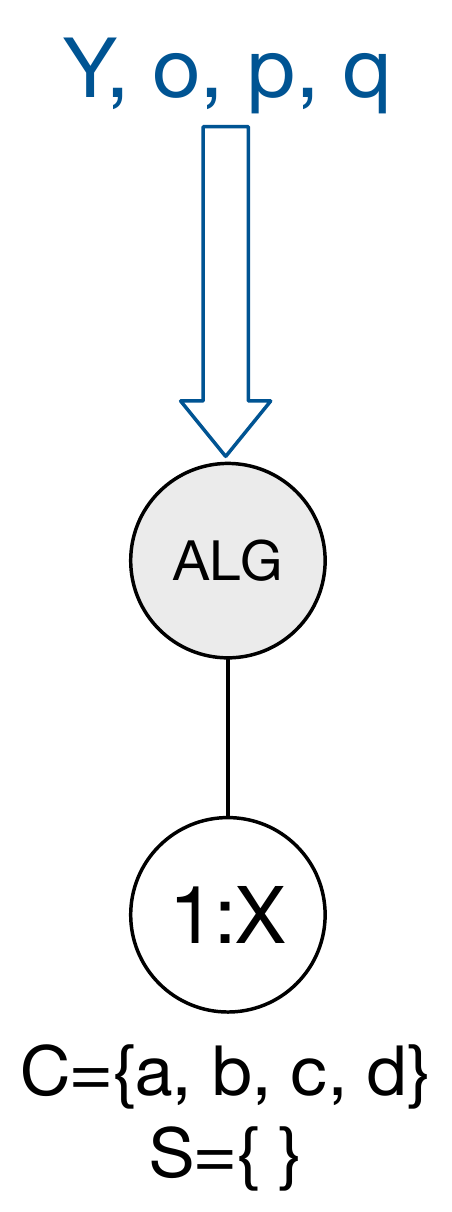}
			\caption{~}
			\label{fig:addition-exm-b}
		\end{subfigure}
		\begin{subfigure}[b]{0.3\textwidth}
			\centering
			\includegraphics[height=.15\textheight]{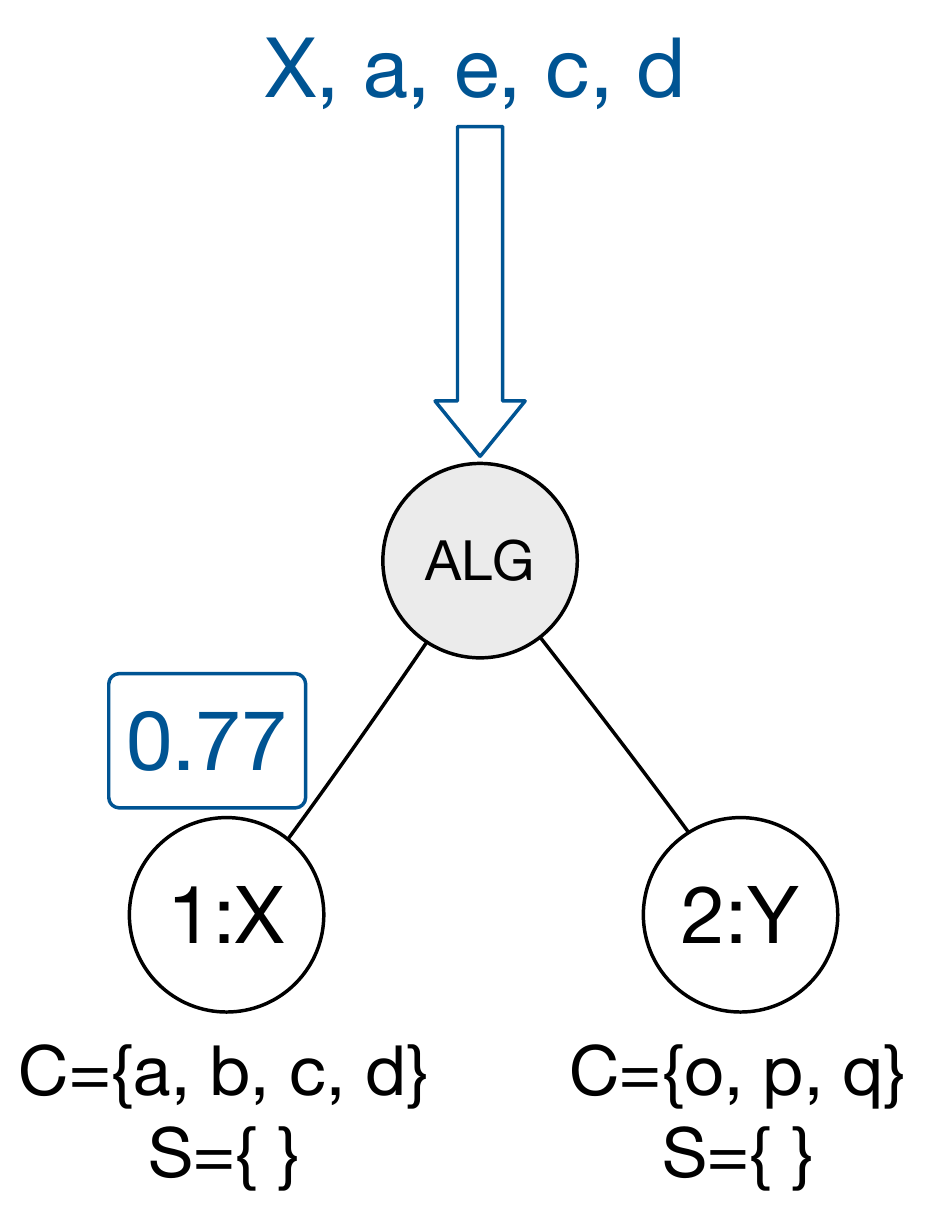}
			\caption{~}
			\label{fig:addition-exm-c}
		\end{subfigure}\\
		\begin{subfigure}[b]{0.4\textwidth}
			\centering
			\includegraphics[height=.2\textheight]{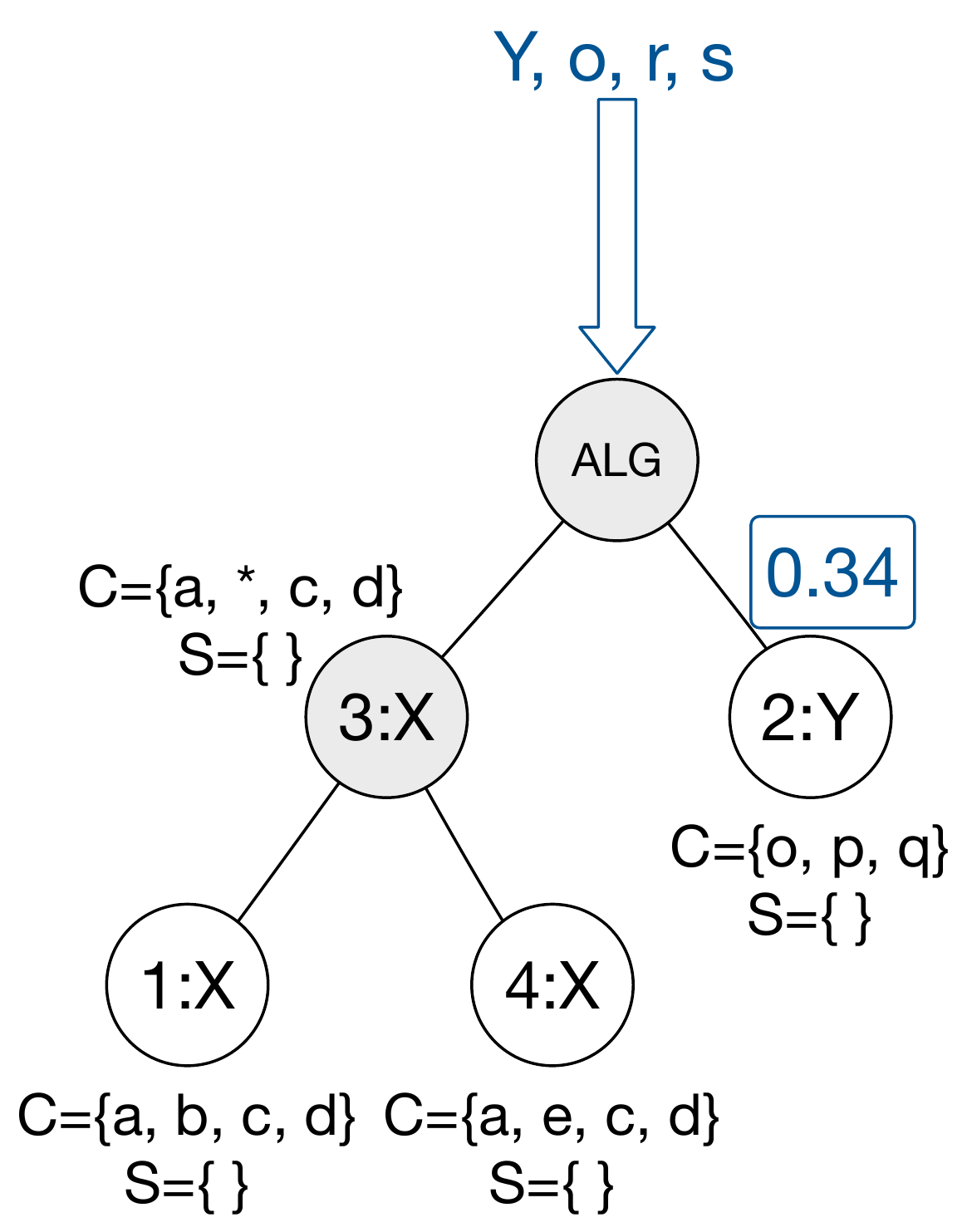}
			\caption{~}
			\label{fig:addition-exm-d}
		\end{subfigure}
		\begin{subfigure}[b]{0.4\textwidth}
			\centering
			\includegraphics[height=.2\textheight]{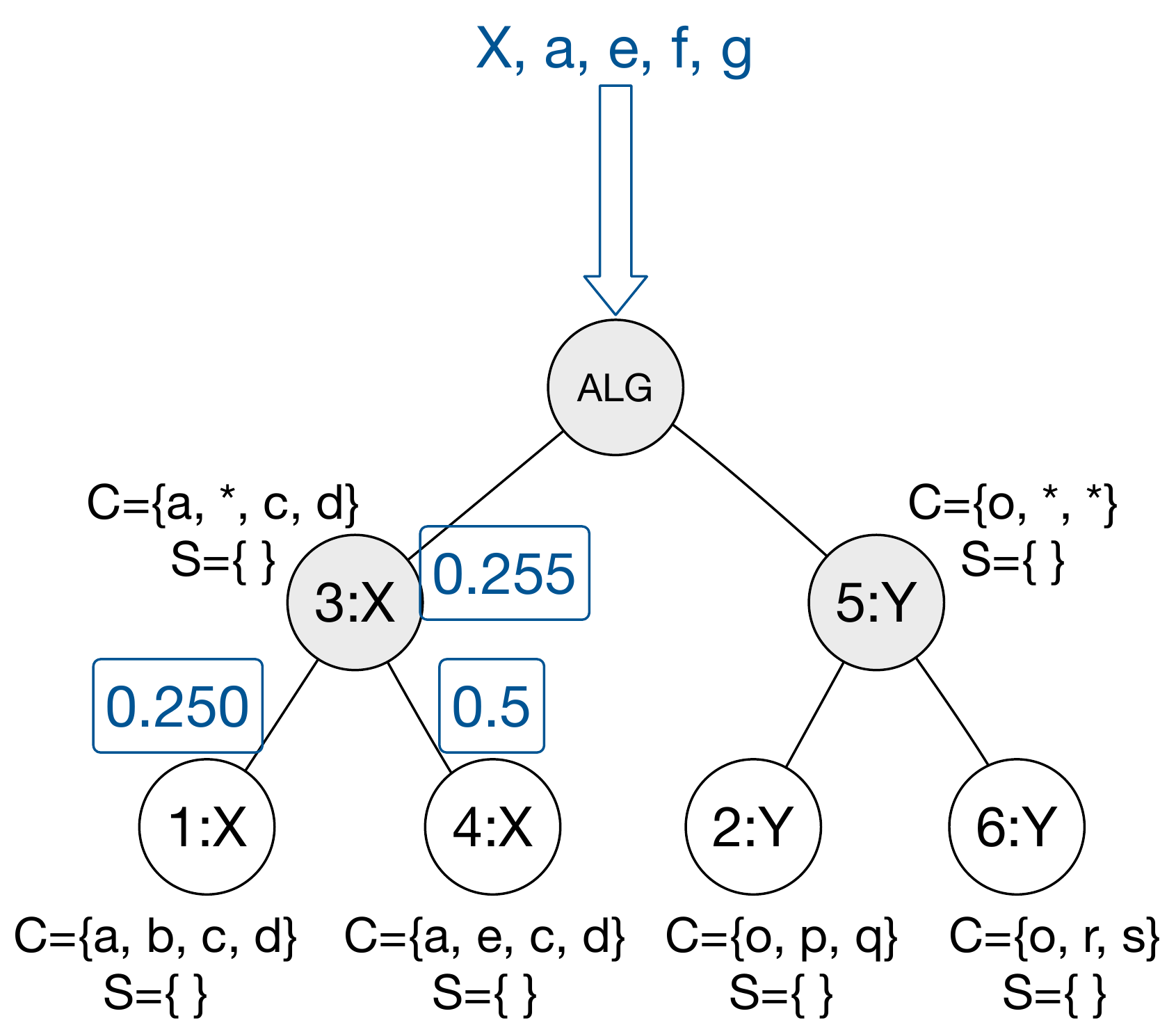}
			\caption{~}
			\label{fig:addition-exm-e}
		\end{subfigure}\\
		\begin{subfigure}[b]{0.4\textwidth}
			\centering
			\includegraphics[height=.25\textheight]{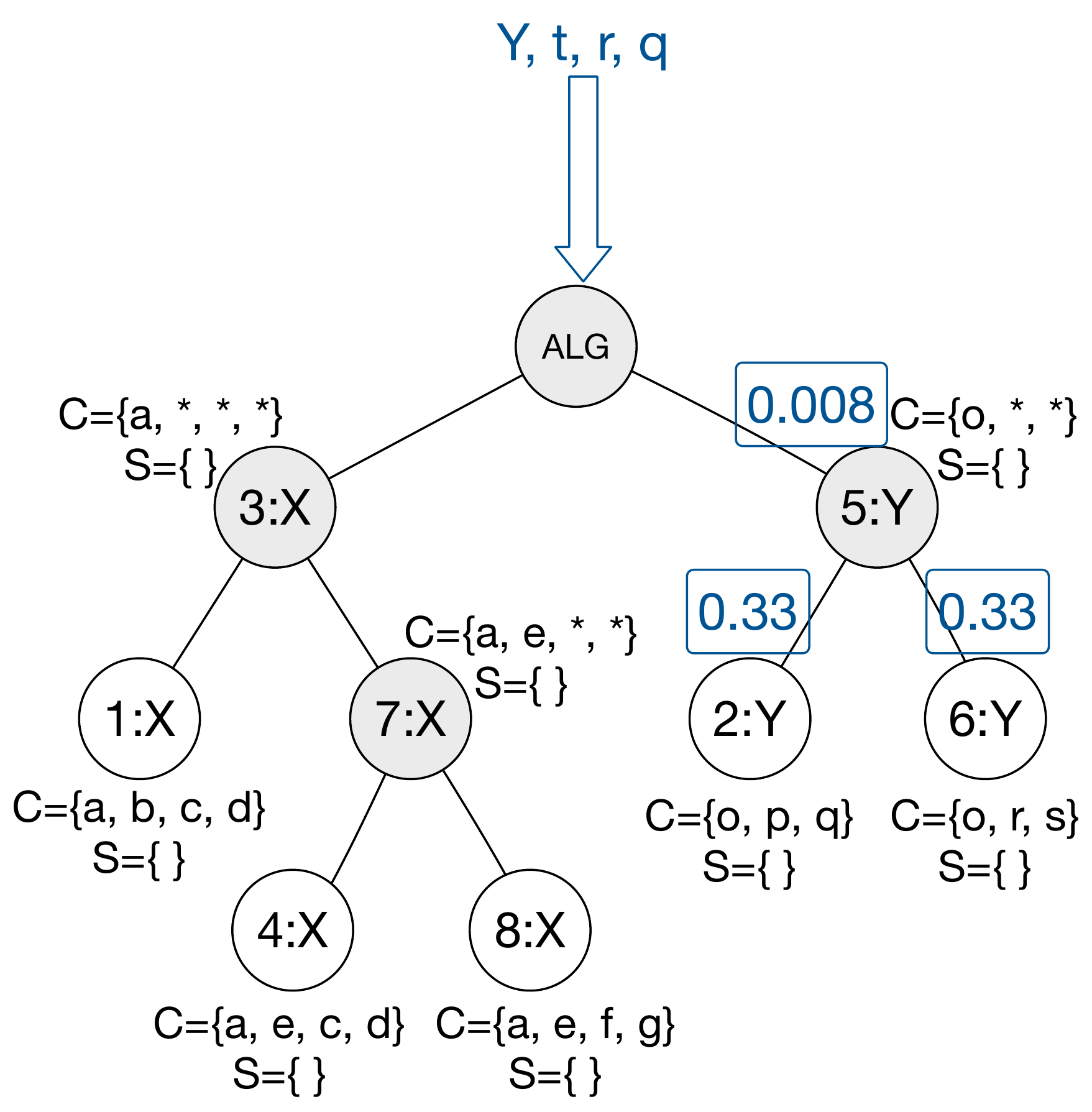}
			\caption{~}
			\label{fig:addition-exm-f}
		\end{subfigure}
		\begin{subfigure}[b]{0.4\textwidth}
			\centering
			\includegraphics[height=.25\textheight]{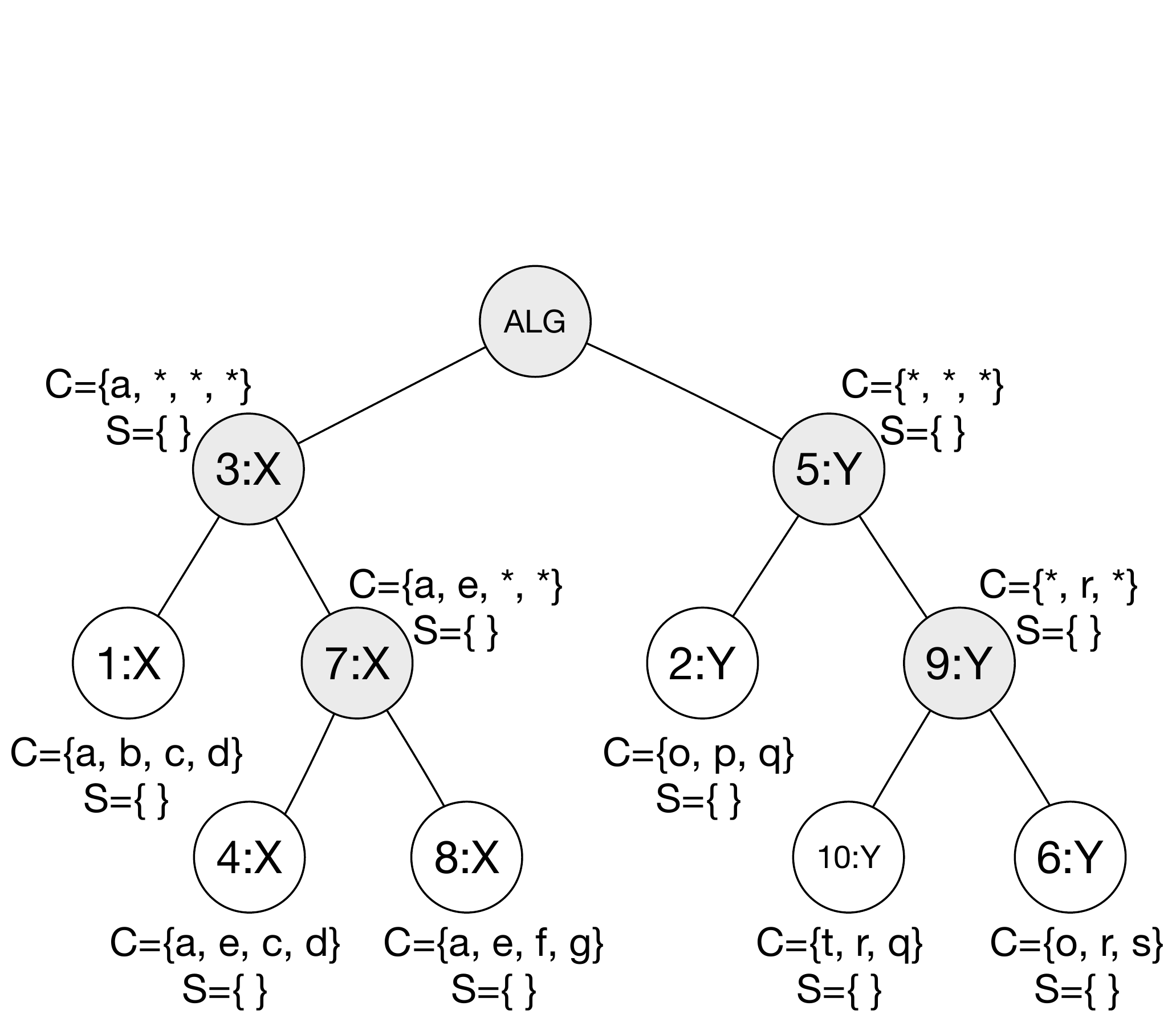}
			\caption{~}
			\label{fig:addition-exm-g}
		\end{subfigure}
		\caption{The step by step demonstration of algorithms~\ref{alg:start} and \ref{alg:add} in an example.}
		\label{fig:addition-exm}
	\end{figure}
	
	\subsubsection{Training and Testing}
	
	In the previous section, we discussed how the holarchy is built as new algorithms or data are added to the system. The important system component that was not taken into consideration, in the aforementioned process, is the model holon. As we have mentioned earlier, the model holons are to represent a practical realization of applying an algorithm on a dataset, therefore, there was no need for their creation as we merely added the definitions of algorithms to the holarchy. In this section, we present the details of the holarchical growth and alternations that happen when we try to train or test an algorithm on specific datasets.
	
	\paragraph{Training}
	By training, we mean creating a machine learning model that is tuned to answer queries about a specific dataset. In the proposed platform, when the system is asked to train a particular algorithm on a specified dataset, one of the following cases happens:
	\begin{itemize}
		\item The holarchy contains no holons that represent the data/algorithm or both of them. In this case, if enough information is provided about the missing component(s), they are added to the holarchy together with the requested model.
		\item The holarchy contains both the data and the algorithm but not the model. In this case, a model holon is created and properly linked to the data and algorithm holons.
		\item The holarchy contains the model, i.e. it has been trained before, with exactly the same configurations. Here, the system might inform the user about the duplication and provide some already available information.
	\end{itemize}
	
	In the remainder of this section, we assume that each training query contains only one algorithm-data pair information, and whenever there is a need to grow the holarchy, all the information about the data and the algorithm is available. Moreover, we skip the details of the interactions with the user in duplication cases.
	
	To deal with the first two cases of the aforementioned list at once, we follow a procedure very similar to the algorithms of inserting a new component, presented in the previous section. Strictly speaking, two passes will be carried out in the holarchy to train the model. In the first pass, which is presented in detail in algorithm~\ref{alg:train1st}, the holarchy is searched to locate the holons representing the data/algorithm, and in case either or none of the data or algorithm is available, the missing component(s) are added to the holarchy.  When the holon representing the algorithm is found/created, it is asked to spawn an empty model holon (lines~\ref{alg:4:spawn0}, \ref{alg:4:spawn1}, \ref{alg:4:spawn2} and \ref{alg:4:spawn3} of the algorithm). Function \textsf{SpawnModel()} creates the model holon properly, stores its address for the current query in the memory, and initiates a recursive address updating procedure in all of the super-holons in the path towards SYS. These addresses are used by all holons in the second phase of the training algorithm to send the commands directly to the appropriate destination without searching for it again. Please remember from before that the model holons are only created by the algorithm holons, i.e. AlgH, and in the case of DatH holons, the destination will be the atomic data holon that will provide the required data for the query. The details of spawn and address update functions are presented in algorithm~\ref{alg:spawn-address}. It is worth noting that algorithm~\ref{alg:train1st} is the same as algorithm~\ref{alg:add} with a few slight changes that are printed in red color to be distinguished easily. Therefore, these two functions can be easily merged into one in practice.    
	
	\begin{algorithm}
		\caption{The first pass of training process.}
		\label{alg:train1st}
		\SetAlgoLined
		\KwIn{ 
			$a_i\left<name_{a_i}, type_{a_i}, P_{a_i}\right>$~~~~\tcp*[h]{an algorithm}
		}
		\SetKwFunction{FMain}{TrainFirstPass}
		\SetKwProg{Fn}{Function}{}{}
		\Fn{\FMain{\textsf{$a_i$}}}{
			\tcp{myself refers to the holon that calls this function}
			\uIf(\tcp*[h]{I am ALG holon}){\textsf{myself.LEVEL $=1$}}{
				$\textsf{new\_holon} \gets \textsf{CreateNewHolon(}name_{a_i}, \textsf{myself}, P_{a_i}, \{\})$\;
				\textcolor{red}{\textsf{Ask(new\_holon, ``SPAWN'', ``MODEL'')}\tcp*{only in AlgHs}}\label{alg:4:spawn0}
			}
			\uElseIf{\textsf{AmIAtomic()}}{
				\eIf(\tcp*[h]{The algorithm already exists}){$\boxed{\accentset{\star}{\sim}}(C_{myself}, P_{a_i}) = 1$}{
					\textcolor{red}{\textsf{SpawnModel()}\tcp*{only in AlgHs}}\label{alg:4:spawn1}
					% \textcolor{red}{\textsf{InformAddress([])}\tcp*{only in DatHs}}\label{alg:4:address2}
				}{
					$\textsf{new\_super} \gets \textsf{Ask(myself.SUPER, ``CREATE-HOLON'', myself.C, myself.S)}$\;
					$\textsf{new\_holon} \gets \textsf{CreateNewHolon(}name_{a_i}, \emptyset, P_{a_i}, \{\}\textsf{)}$\;
					\textsf{JoinHolon(new\_super)}\;
					\textsf{Ask(new\_holon, ``JOIN'', new\_super)}\;
					\textcolor{red}{\textsf{Ask(new\_holon, ``SPAWN'', ``MODEL'')}\tcp*{only in AlgHs}}\label{alg:4:spawn2}
					% \textcolor{red}{\textsf{Ask(new\_holon, ``INFORM'', ``ADDRESS'')}\tcp*{only in DatHs}}\label{alg:4:address1}
				}
			}
			\Else{
				$\textsf{best} \gets \textsf{myself}$\;
				$\varsigma_{\max} \gets \boxed{\accentset{\star}{\sim}}(C_{myself}, P_{a_i})$\;
				\ForEach{${}_{a}^{\{i\}}H_{j}^{l+1} \in {}_{a}^{\{u\}}H_{i}^{l}$}{\label{alg:4:cfp}
					$\textsf{proposal} \gets \textsf{CFP}({}_{a}^{\{i\}}H_{j}^{l+1},P_{a_i})$\tcp*{Call For Proposal}
					\If{$\textsf{proposal} > \varsigma_{\max}$}{
						$\varsigma_{\max} \gets \textsf{proposal}$\;
						$\textsf{best} \gets {}_{a}^{\{i\}}H_{j}^{l+1}$\;
						\textbf{break}\;
					}
				}
				\eIf{\textsf{best = myself}}{
					$\textsf{new\_holon} \gets \textsf{CreateNewHolon(myself.NAME, myself, } P_{a_i}, \{\})$\;
					$\textsf{UpdateCapability(}P_{a_i}\textsf{)}$\;
					\textcolor{red}{\textsf{Ask(new\_holon, ``SPAWN'', ``MODEL'')}\tcp*{only in AlgHs}}\label{alg:4:spawn3}
					% \textcolor{red}{\textsf{Ask(new\_holon, ``INFORM'', ``ADDRESS'')}\tcp*{only in DatHs}}\label{alg:4:address1}
				}{
					\textcolor{red}{\textsf{Ask(best, ``TRAIN FIRST PASS'', $a_i$)}\;}\label{alg:4:ask1}
				}
			}
		}
		\textbf{End Function}
	\end{algorithm}
	
	\begin{algorithm}
		\caption{The algorithm for spawning a model holon and propagating its address to the top of holarchy.}
		\label{alg:spawn-address}
		\SetAlgoLined
		\SetKwFunction{FSpawn}{SpawnModel}
		\SetKwProg{Fn}{Function}{}{}
		\Fn{\FSpawn{}}{
			\tcp{myself refers to the holon that calls this function}
			$\textsf{new\_model} \gets \textsf{CreateModel(myself.NAME, myself, myself.C,\{\})}$\;
			$\textsf{InformAddress(qid}_i\textsf{,new\_model.ADDRESS)}$\;
		}
		\textbf{End Function}
		\BlankLine
		\BlankLine
		\SetKwFunction{FInfAdd}{InformAddress}
		\SetKwProg{Fn}{Function}{}{}
		\Fn{\FInfAdd{\textsf{\textit{query\text{-}id, address}}}}{
			\tcp{myself refers to the holon that calls this function}
			$\textsf{StoreAddress(query\text{-}id, address)}$\;
			$\textsf{Ask(myself.SUPER, "INFORM-ADDRESS",}$\\$\hfill\textsf{query\text{-}id, myslef.ADDRESS)}$\;
		}
		\textbf{End Function}
	\end{algorithm}
	
	The second phase of the training procedure begins as soon as the SysH holon is informed about the addresses of the training query from both ALG and DATA sub-holons. In the second pass, the SysH asks the ALG and DATA holons to initiate training, provided the address of the corresponding model/data holon (companion). Receiving this request, each holon forwards the request to the address that it has stored in its memory in the previous pass. This will continue until the request reaches the proper destination, i.e. the newly created model holon or the atomic data holon. Upon receiving the request, the data holon is configured to provide access to the specified model holon when needed. On the other hand, the model holon, as soon as it receives the training command, communicates with the data holon through the address that has been provided and then starts to train its inherited algorithm on the provided data. Two points should be noted. Firstly, the model holon joins the data holon and updates its capabilities, only after it successfully communicated the data holon and granted access to the data. Secondly, the skills of the model holon and all of its super-holons are updated after the training procedure finishes successfully, i.e. the algorithm is successfully trained on the data. This update can happen on the way the training results are sent back to the SysH holon. The details of the second pass are presented in algorithm~\ref{alg:train2nd}. Although it is out of the scope of this paper, it is worth remarking that the two-pass training process facilitates control mechanisms and integrity checks, especially when the data/algorithms are provided by third parties. Furthermore, the model holons, being the shared member of an algorithm and a data holon, help the system properly keep track of and handle the changes should any is made in the definition or access of algorithm/data holons in the entire system.
	
	\begin{algorithm}
		\caption{The second pass of the training process.}
		\label{alg:train2nd}
		\SetAlgoLined
		\SetKwFunction{FMain}{TrainSecondPass}
		\SetKwProg{Fn}{Function}{}{}
		\Fn{\FMain{\textsf{$query\text{-}id,target\text{-}address, companion\text{-}address$}}}{
			\tcp{myself refers to the holon that calls this function}
			\eIf{\textsf{target-address = myself.ADDRESS}}{
				\textsf{StartTraining(companion-address)}\;
			}{
				$\textsf{address} \gets \textsf{GetAddress(query-id)}$\;
				\textsf{Ask(address, ``TRAIN SECOND PASS'', \\\hfill query-id, companion-address)}\;
			}
		}
		\textbf{End Function}
	\end{algorithm}
	
	As it has been mentioned before, the proposed platform is designed in such a way that no operation blocks the flows of the past or future ones. In other words, while a specific agent/holon is busy processing a query, the other parts are open to accepting new requests without waiting for the previous results. This behavior is largely managed by the communication of the holons as described before, together with the local updates that each holon might make as needed. For instance, the first pass of a new training query that is being processed just before the second pass of another query begins can easily invalidate the previous address updates due to the potential structural changes it may cause. To overcome this problem, whenever a new data/algorithm holon is created, the holons in the vicinity of the change update their addresses accordingly. Figure~\ref{fig:address-update-exm} demonstrates the way addresses are updated. In part~\ref{fig:address-update-exm-a}, the target address entries of holon \textbf{0}'s memory for queries $q_1$ and $q_2$ are pointing to atomic holon \textbf{1}. In part~\ref{fig:address-update-exm-b}, the new holon \textbf{3} is created and inserted because of query $q_3$. As a result, the memory entries of holon \textbf{0} is updated to point to the newly created super-holon \textbf{2}, and holon \textbf{2} per se, points to the holon \textbf{1} now. As soon as the first pass of $q_3$ finishes, the corresponding addresses (shown in green color) are created for this query as explained before.
	
	\begin{figure}
		\centering
		\begin{subfigure}[b]{0.3\textwidth}
			\centering
			\includegraphics[height=.15\textheight]{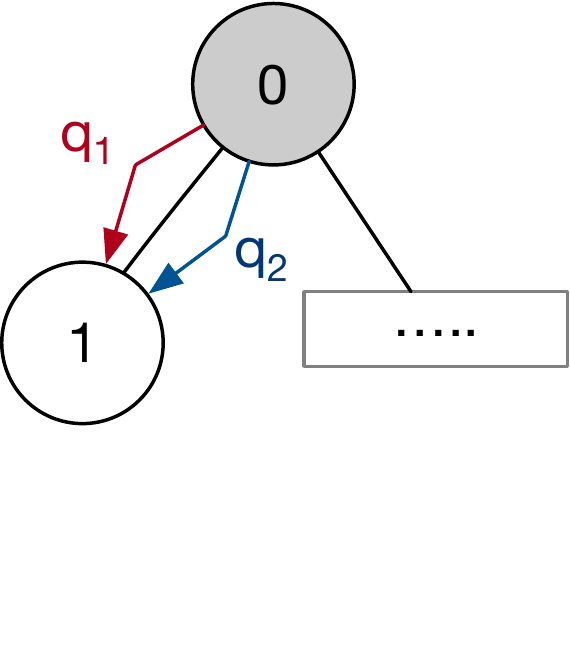}
			\caption{~}
			\label{fig:address-update-exm-a}
		\end{subfigure}
		\begin{subfigure}[b]{0.3\textwidth}
			\centering
			\includegraphics[height=.15\textheight]{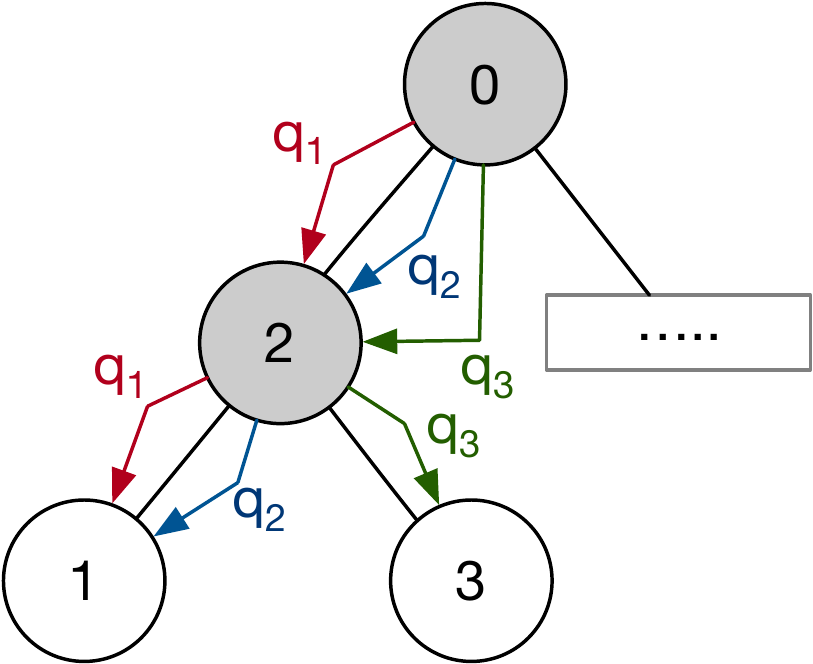}
			\caption{~}
			\label{fig:address-update-exm-b}
		\end{subfigure}
		\caption{An example demonstrating the way the access addresses are updated when a new holon is created.}
		\label{fig:address-update-exm}
	\end{figure}
	
	\paragraph{Testing}\label{sec:testing}
	We define testing identically with how it is contemporarily used in the machine learning/data science communities, which is a metric-based assessment of the efficacy of a trained system against particular pre-defined and ground-truthed datasets. The prerequisite of exerting a testing operation is defined by each atomic holon. Without loss of generality, the following testing method assumes that the test data is already in the holarchy and its information is explicitly provided. This needs a process very similar to the one we used in the first pass of the training algorithm (algorithm~\ref{alg:train1st}) to insert and/or retrieve the address of the test dataset.
	
	The testing process launches by the SYS holon passing the testing information and criteria to the ALG holon. As we would like to process and perform the operations in a batch, we allow the use of the parametric general symbol (defined in \ref{def:pg}) in the query. Receiving the request, each holon compares the criteria, consisting of the names and configuration parameters of the data and algorithms, with its own capabilities and skills, based on the definition of parametric inequality (definition~\ref{def:pin}). Formally, let $(a_j,P_{a_j})$ and $(d_j,P_{d_j})$ be the specifications of the testing algorithms and data, respectively. As soon as holon ${}_{a}^{\{u\}}H_i^l$ receives a testing request from its super-holon, it checks if the following statement is true: 
	\begin{equation}\label{eq:testing-check}
		\left((name,a_j)\accentset{\star}{\le} (name, name_{{}_{a}^{\{u\}}H_i^l})\right) \wedge \left(P_{a_j} \accentset{\star}{\le} C_{{}_{a}^{\{u\}}H_i^l}\right) \wedge \left(\exists s \in S_{{}_{a}^{\{u\}}H_i^l}: s = d_j\right)
	\end{equation}
	where the first part of the statement ensures that the process is at the correct sub-holarchy; the second term checks whether the available capabilities can cover the requested testing parameters, and finally, the third statement assures that the dataset of the same family has been introduced to the holarchy in the training phase. If this statement yields false, the holon informs its super-holon; otherwise, it sends the requests to its AlgH sub-holons and collects their answers to report to the super-holon. Algorithm~\ref{alg:test} provides the details of the testing mechanism. It is worth mentioning that line~\ref{alg:7:ask} does not imply a blocking process in collecting the results. In practice, the requests are sent and are collected later based on the identity of the testing query.
	
	\begin{algorithm}
		\caption{The testing process.}
		\label{alg:test}
		\SetAlgoLined
		\SetKwIF{If}{ElseIf}{Else}{if~(\endgraf}{\endgraf)~then}{else if}{else}{end if}%
		\KwIn{ 
			$q_j\left<id, (a_j,P_{a_j}), (d_j,P_{d_j}), O\right>$
		}
		\SetKwFunction{FMain}{Test}
		\SetKwProg{Fn}{Function}{}{}
		\Fn{\FMain{\textsf{$q_j\left<id, (a_j,P_{a_j}), (d_j,P_{d_j}), O\right>$}}}{
			\tcp{myself refers to the holon that calls this function}
			$\textsf{results} \gets \emptyset$\;
			\eIf{\textsf{AmIModel()}}{
				$\textsf{results} \gets \textsf{Perform(``TEST'', (}$ $d_j,P_{d_j} \textsf{), O)}$ \;
			}{
				\If{\begin{tabular}{@{\hspace*{1.5em}}l@{}}
						$ \left((name,a_j)\accentset{\star}{\le} (name, name_{{}_{a}^{\{u\}}H_i^l})\right) \wedge \left(P_{a_j} \accentset{\star}{\le} C_{{}_{a}^{\{u\}}H_i^l}\right) \wedge$\\\hfill $\left(\exists s \in S_{{}_{a}^{\{u\}}H_i^l}: s = d_j\right) \vee \left(\textsf{myself.LEVEL =} 1\right)$\\
					\end{tabular}
				}{
					\ForEach{${}_{\{a,m\}}^{\{i\}}H_{k}^{l+1} \in {}_{a}^{\{u\}}H_{i}^{l}$}{
						$\textsf{results} \gets \textsf{results}\; \bigcup\; \textsf{Ask}\left({}_{a}^{\{i\}}H_k^{l+1}, \textsf{``TEST''}, q_i\left<id, (a_j,P_{a_j}), (d_j,P_{d_j}), O_i\right>\right)$\;\label{alg:7:ask}
					}
				}
			}
			\Return \textsf{Inform(myself.SUPER, ``RESULTS'', results)}\;
		}
		\textbf{End Function}
	\end{algorithm}
	
	A parametric general symbol is not limited to be used only in the algorithm specification of the query. To support the symbol for data specification, a process very similar to algorithm~\ref{alg:test} should be utilized to collect all the data first and pass them to the testing procedure. In other words, the data holons will collect the information of the datasets that match the query, instead of the testing results and send them back to their super-holons recursively. 
	
	\section {Theoretical Analysis}
	Previous sections presented the details of the proposed distributed platform, without providing a further evaluation of its correctness and efficiency. This section delves into the theoretical analysis of platform in terms of computational complexities and the correctness of the presented algorithms 
	\subsection{Correctness}
	To make sure that the proposed platform is working correctly, we need to prove that all of its algorithms, i.e. training and testing, produce correct answers. We use soundness and completeness as the measures of correctness and, assuming the use of correctly implemented machine learning algorithms and flawless datasets, we prove them for both of the proposed training and testing procedures. This is done by a set of lemmas and theorems proven in the remainder of this section. Figure~\ref{fig:proof-conf} depicts the relationship between the parameters that the following lemmas and corollaries assume.
	
	\begin{figure}
		\centering
		\begin{subfigure}[b]{0.3\textwidth}
			\centering
			\includegraphics[height=.15\textheight]{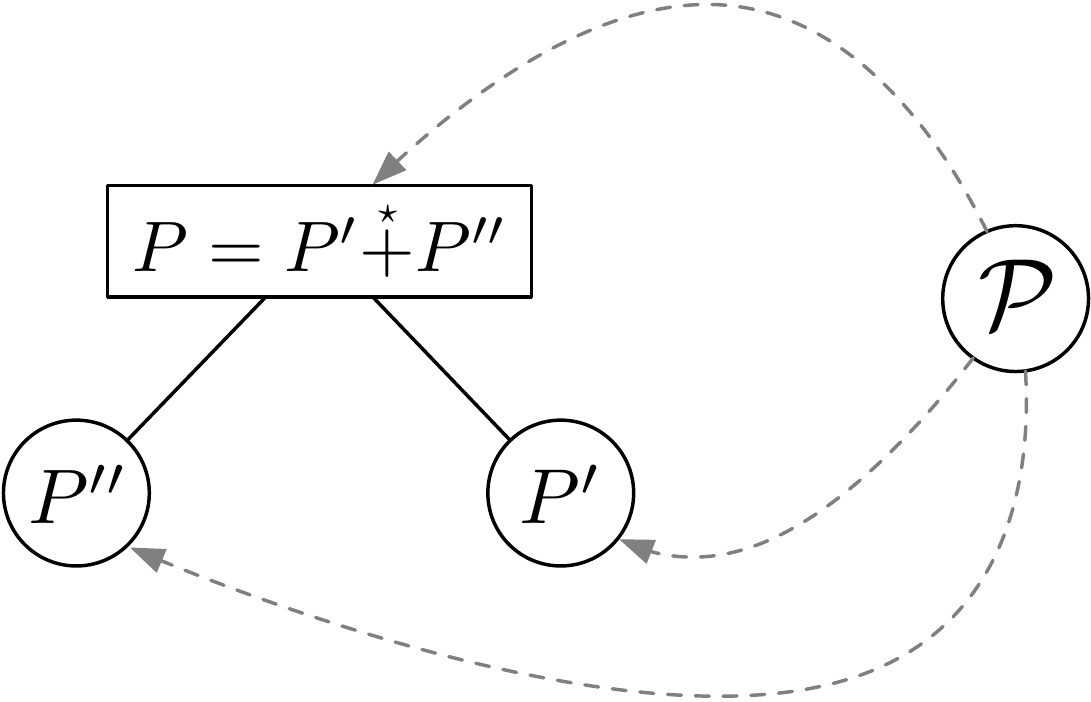}
			\caption{~}
			\label{fig:proof-conf-a}
		\end{subfigure}\hspace{2em}
		\begin{subfigure}[b]{0.35\textwidth}
			\centering
			\includegraphics[height=.35\textheight]{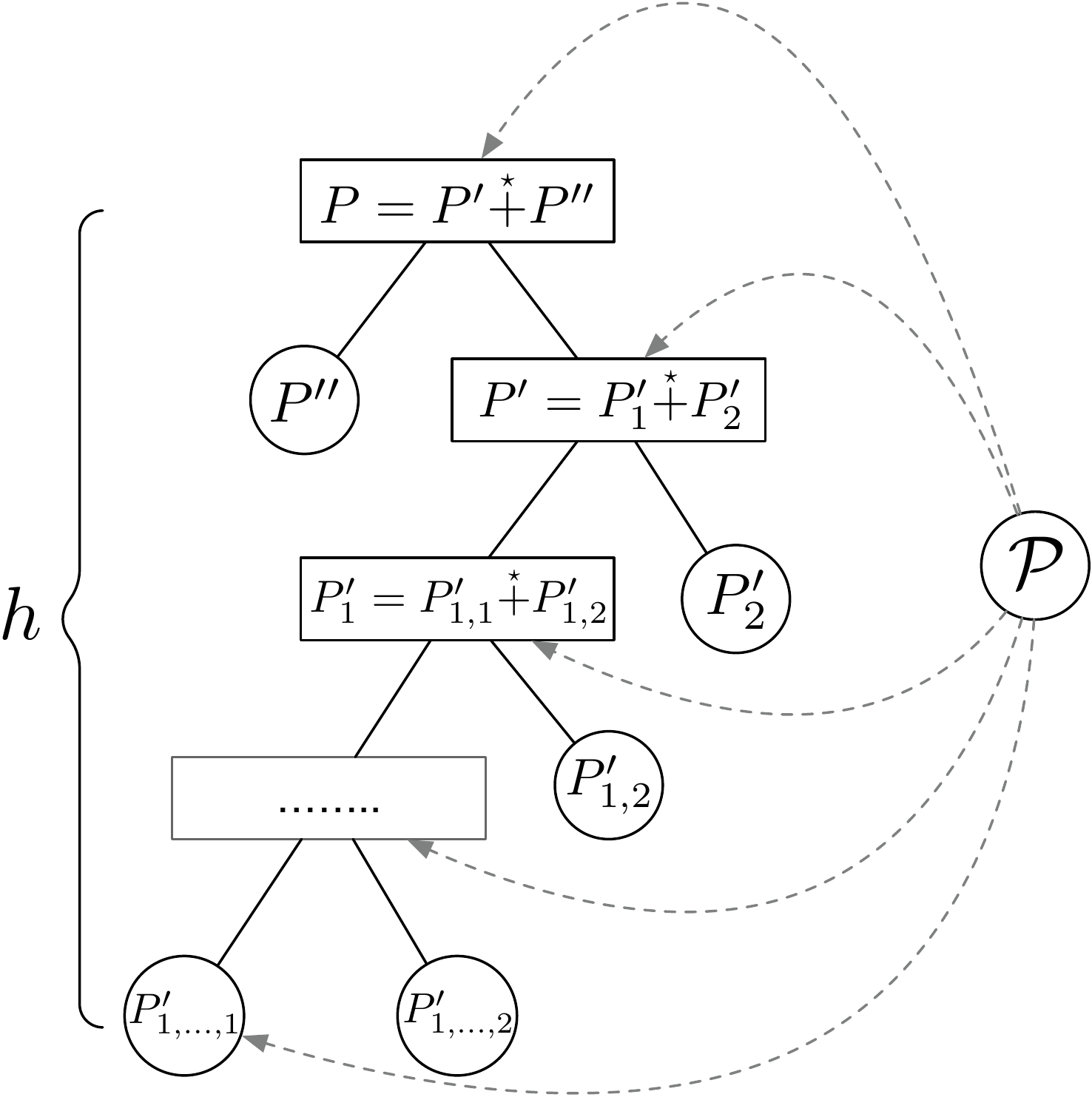}
			\caption{~}
			\label{fig:proof-conf-b}
		\end{subfigure}
		\caption{The relationships between the parameter sets used in the lemmas and corollaries. The dashed lines represent the multiple ways that a query with configuration parameter set $\mathcal{P}$ is processed by each agent in the hierarchy.}
		\label{fig:proof-conf}
	\end{figure}
		
	\begin{lemma}\label{lem:2}
		If $P, P', P''\in\mathscr{P}$ such that $P'\ne P''\text{ and } P=P'\accentset{\star}{+}P''$, where $\mathscr{P}$ is the set of all possible congruent sets of the same size, then for any non-general parametric set $\mathcal{P}\in\mathscr{P}$ (see figure~\ref{fig:proof-conf-a}), we have $\accentset{\star}{\sim}(\mathcal{P}, P)<\accentset{\star}{\sim}(\mathcal{P}, P')$ if and only if $|\mathcal{P}\cap P|<|\mathcal{P}\cap P'|$. In other words, this lemma ensures that the similarity ratio, $\accentset{\star}{\sim}$, correctly assigns a greater value for the capability of an agent that is more specialized in handling an incoming query.
	\end{lemma}
	\begin{proof}
		Assuming $\mathcal{P}=\{(\rho_i,\nu_i\}$, let's define $L=\{(\rho_i,v_i)\in P : v_i\ne \nu_i \wedge (v_i=\ast \veebar\: \nu_i=\ast)\}$,  $M=\{(\rho,v_i) : v_i\ne \nu_i \wedge (v_i\ne\ast \wedge \nu_i\ne\ast)\}$, and $L'$ and $M'$ sets similarly. Since $\forall P,P'\in\mathscr{P},\quad P\cap P' = \{(\rho,\nu): (\rho,\nu)\in P,P'\}$ and as all the sets are congruent, $\accentset{\star}{\sim}(\mathcal{P}, P)<\accentset{\star}{\sim}(\mathcal{P}, P')$ means:
		\begin{equation}
			\begin{aligned}
				\displaystyle \sum_{\substack{\nu_i=v_i^\prime\\i}}\accentset{\star}{\sim}((\rho_i,\nu_i), &(\rho_i,v_i^\prime)) + \prod_{\substack{\nu_i\ne v_i^\prime\\i}}\accentset{\star}{\sim}((\rho_i,\nu_i), (\rho_i,v_i^\prime)) <\\ & \displaystyle \sum_{\substack{\nu_i=v_i''\\i}}\accentset{\star}{\sim}((\rho_i,\nu_i), (\rho_i,v_i'')) + \prod_{\substack{\nu_i\ne v_i''\\i}}\accentset{\star}{\sim}((\rho_i,\nu_i), (\rho_i,v_i''))
			\end{aligned}
		\end{equation}
		\begin{align}
			\Longrightarrow
			|\mathcal{P}\cap P| + \alpha^{|L|}\beta^{|M|}<|\mathcal{P}\cap P'| + \alpha^{|L'|}\beta^{|M|'}
		\end{align}
		This needs that either $|\mathcal{P}\cap P|<|\mathcal{P}\cap P'|$, which proves the claim, or :
		\begin{equation}
			\begin{aligned}
				|\mathcal{P}\cap P|=|\mathcal{P}\cap P'|\quad &\Longrightarrow\quad \alpha^{|L|}\beta^{|M|}<\alpha^{|L'|}\beta^{|M'|}\\
				&\Longrightarrow\alpha^{|L|-|L'|}<\beta^{|M'|-|M|}
			\end{aligned}
		\end{equation}
		
		On the other hand, due to definitions~\ref{def:pe} and \ref{def:pd}, we must have $|L|+|M|=|L'|+|M'|$, $|L|<|L'|$, and $|M|>|M'|$. This requires that there are more general pairs in $P$ than in $P'$, which is not possible according to the definition~\ref{def:ps}.
		
		The second part of the proof is trivial, because of the fact that $|\mathcal{P}\cap P|>1$ and $|\mathcal{P}\cap P'|>1$ and as they are the dominant factors in the calculation of the parametric similarity ratio, we can write:
		\begin{equation}
			\begin{aligned}
				|\mathcal{P}\cap P|<|\mathcal{P}\cap P'|&\Longrightarrow|\mathcal{P}\cap P| + \alpha^{|L|}\beta^{|M|}<|\mathcal{P}\cap P'| + \alpha^{|L'|}\beta^{|M|'}\\&\Longrightarrow \accentset{\star}{\sim}(\mathcal{P}, P)<\accentset{\star}{\sim}(\mathcal{P}, P')
			\end{aligned}
		\end{equation}
	\end{proof}
	
	\begin{corollary}\label{cor:21}
		If $P, P', P''\in\mathscr{P}$ such that $P'\ne P''\text{ and } P=P'\accentset{\star}{+}P''$, where $\mathscr{P}$ is the set of all possible congruent sets of the same size, then for any non-general parametric set $\mathcal{P}\in\mathscr{P}$ (see figure~\ref{fig:proof-conf-a}), only one of the statements $\accentset{\star}{\sim}(\mathcal{P}, P)<\accentset{\star}{\sim}(\mathcal{P}, P')$ or $\accentset{\star}{\sim}(\mathcal{P}, P)<\accentset{\star}{\sim}(\mathcal{P}, P'')$ will be true. To put it another way, according to this corollary, only one of the children of an agent node can have a similarity ratio larger than its parent. That is a parent agent can safely stop processing the proposals from its subordinates as soon as it receives one that has a value, i.e. similarity ratio, greater than its own. 
	\end{corollary}
	\begin{proof}
		Let's assume that both $\accentset{\star}{\sim}(\mathcal{P}, P)<\accentset{\star}{\sim}(\mathcal{P}, P')$ and $\accentset{\star}{\sim}(\mathcal{P}, P)<\accentset{\star}{\sim}(\mathcal{P}, P'')$ are true. According to lemma~\ref{lem:2}, this means that $|\mathcal{P}\cap P|<|\mathcal{P}\cap P'|$ and $|\mathcal{P}\cap P|<|\mathcal{P}\cap P''|$. This means both of the following statements are true:
		\begin{align}
			\exists (p_i,v'_i\ne *)\in P' \quad &:\quad (p_i,v'_i)\in \mathcal{P} \wedge (p_i,v'_i)\not\in P\\
			\exists (p_i,v''_i\ne *)\in P'' \quad &:\quad (p_i,v''_i)\in \mathcal{P} \wedge (p_i,v''_i)\not\in P
		\end{align}
		There are two possible cases: (i) $v'_i=v''_i$ or (ii) $v'_i\ne v''_i$. According to definition~\ref{def:ps}, the first case is not possible, because that will cause the corresponding pair to appear in set $P$ as well. On the other hand, the second case means that we have two different values for the the same parameter in set $\mathcal{P}$, which contradicts the definition of parametric sets (definition~\ref{def:pset}). Therefore, we conclude that the claim of the corollary is correct.  
	\end{proof}
	\begin{lemma}\label{lem:4}
		If $P, P', P''\in\mathscr{P}$ such that $P'\ne P''\text{ and } P=P'\accentset{\star}{+}P''$, where $\mathscr{P}$ is the set of all possible congruent sets of the same size, then for any parametric set $\mathcal{P}\in\mathscr{P}$ (figure~\ref{fig:proof-conf-a}), if $\mathcal{P}\accentset{\star}{\le}P'$ and/or $\mathcal{P}\accentset{\star}{\le}P''$ then $\mathcal{P}\accentset{\star}{\le}P$. 
	\end{lemma}
	\begin{proof}
		According to definition~\ref{def:pin}, we have:
		\begin{align}
			\mathcal{P}\accentset{\star}{\le}P'\Longrightarrow \forall (\rho,\nu)\in\mathcal{P} : (\rho,\nu)\in P' or (\rho,\ast)\in P'
		\end{align}
		On the other hand, based on the membership of $(\rho,\nu)$ in $P''$ and definition~\ref{def:ps}, only one of the following cases will happen.
		\begin{align}
			(\rho,\nu)\in P'\wedge (\rho,\nu)\in P'' &\Longrightarrow (\rho,\nu)\in P\Longrightarrow\mathcal{P}\accentset{\star}{\le}P\\
			(\rho,\nu)\in P'\wedge (\rho,\nu)\not\in P'' &\Longrightarrow (\rho,\ast)\in P\Longrightarrow\mathcal{P}\accentset{\star}{\le}P\\
			(\rho,\ast)\in P'\wedge (\rho,\nu)\in P'' &\Longrightarrow (\rho,\ast)\in P\Longrightarrow\mathcal{P}\accentset{\star}{\le}P\\
			(\rho,\ast)\in P'\wedge (\rho,\nu)\not\in P'' &\Longrightarrow (\rho,\ast)\in P\Longrightarrow\mathcal{P}\accentset{\star}{\le}P
		\end{align}
		It can be seen that regardless of the membership of $(\rho,\nu)$ in $P''$ the claim is proven. We can use a similar process for $\mathcal{P}\accentset{\star}{\le}P''$ and finally show the correctness of this lemma.
	\end{proof}
	
	\begin{corollary}\label{cor:41}
		If in $P=P'\accentset{\star}{+}P''$, any of the sets on right hand side are the result of recursively applying operator $\accentset{\star}{+}$ on two or more other parametric sets (figure~\ref{fig:proof-conf-b}), i.e. 
		\begin{equation}\label{eq:cor41:1}
			\begin{aligned}
				P=\left(P'_1\accentset{\star}{+}P'_2 \right)\accentset{\star}{+}P'' &= \left(\left(P'_{1,1}\accentset{\star}{+}P'_{1,2}\right)\accentset{\star}{+}P'_2\right)\accentset{\star}{+}P'' =\dots\\&= \left(\left(\dots\left(P'_{\underbrace{1,...,1}_h}\accentset{\star}{+}P'_{\underbrace{1,...,1,2}_h}\right)\accentset{\star}{+}\dots\right)\accentset{\star}{+}P'_2\right)\accentset{\star}{+}P''
			\end{aligned}
		\end{equation}
		then
		\begin{equation}
			\mathcal{P}\accentset{\star}{\le}P'_{\underbrace{1,...,1}_h}\Longrightarrow \mathcal{P}\accentset{\star}{\le}P
		\end{equation}
		That is, if there is at least one atomic/leaf agent at the bottom of the hierarchy that can fulfill a received query, the request will be properly directed to that agent through the internal agent nodes.
	\end{corollary}
	\begin{proof}
		Based on the relationship between the sets, according to equation~\ref{eq:cor41:1}, we can write:
		\begin{equation}
			\begin{aligned}
				\mathcal{P}\accentset{\star}{\le}P'_{\underbrace{1,...,1}_h}\Longrightarrow \mathcal{P}\accentset{\star}{\le}P'_{\underbrace{1,...,1}_{h-1}}&\Longrightarrow\dots\Longrightarrow \mathcal{P}\accentset{\star}{\le}P'_{1,1}\\&\Longrightarrow \mathcal{P}\accentset{\star}{\le}P'_{1}\Longrightarrow \mathcal{P}\accentset{\star}{\le}P'\Longrightarrow \mathcal{P}\accentset{\star}{\le}P
			\end{aligned}
		\end{equation}
	\end{proof}
	\begin{corollary}\label{cor:42}
		If $P, P', P''\in\mathscr{P}$ such that $P'\ne P''\text{ and } P=P'\accentset{\star}{+}P''$, where $\mathscr{P}$ is the set of all possible congruent sets of the same size, then for any parametric set $\mathcal{P}\in\mathscr{P}$ (figure~\ref{fig:proof-conf-a}), if $\mathcal{P}\accentset{\star}{\not\le}P'$ and $\mathcal{P}\accentset{\star}{\not\le}P''$ then $\mathcal{P}\accentset{\star}{\not\le}P$. In other words, if there is no chance that a received query be fulfilled by an agent at the bottom of the hierarchy, none of the parent agents will be able to fulfill it either. Hence, the query will be blocked as early as possible by the parents.  
	\end{corollary}
	\begin{proof}
		According to definitions~\ref{def:pin} and \ref{def:ps}, we have:
		\begin{align}
			\mathcal{P}\accentset{\star}{\not\le}P',P''\Longrightarrow\exists (\rho,\nu)\in\mathcal{P} &: (\rho,\nu)\not\in P',P'' \wedge (\rho,\ast)\not\in P',P''\\
			(\rho,\nu)\not\in P',P''\Longrightarrow (\rho,\nu)\not\in P  &\textbf{~~~and~~~} (\rho,\ast)\not\in P',P''\Longrightarrow (\rho,\ast)\not\in P
		\end{align}
		Consequently, 
		\begin{equation}
			\exists (\rho,\nu)\in\mathcal{P} : (\rho,\nu)\not\in P \wedge (\rho,\ast)\not\in P \Longrightarrow\mathcal{P}\accentset{\star}{\not\le}P
		\end{equation}
	\end{proof}
	\begin{theorem}[Soundness and Completeness of the Training Algorithm]
		Giving the training algorithm all the information it needs to operate, it will provide a correct result whenever there exists one and a proper warning otherwise.    
	\end{theorem}
	\begin{proof}
		Taking another deep look at the training algorithm, we notice that it always ends with adding new holons (data/algorithm/model), when it is needed, and returning the training results to the SYS holon. Consequently, to prove the soundness and completeness of the proposed method, we just need to prove that firstly, the first pass of the method will only add the components if they are not already in the holarchy (no duplication); and secondly, the final training command is correctly directed to the model and data holons so that they start the fitting procedure. 
		
		To prove the first claim, let's assume the first pass will result in duplicate holons of the same settings at different parts of the holarchy. Based on the fact that the capability of a holon at any node is a parametric sum of all of its non-model sub-holons, duplication might occur if at any holon above the current holon in the holarchy, the training algorithm, starting at line~\ref{alg:4:cfp}, makes a wrong choice and directs the query to a wrong sub-holon. Since the proposals are made based on calculating the parametric similarity ratio, and also according to lemma~\ref{lem:2} and corollary~\ref{cor:21}, it is guaranteed that this will not happen. That is, the holon's choice to forward the training query will always be correct, leading to inserting any new holon at its best place and preventing duplication during the first pass of the training procedure.
		
		The second claim of this proof is guaranteed by code. As it was discussed in the details of the second pass of the training algorithm, any new structural changes in the holarchy, as the result of new training queries, is followed by updating all the references in the memory of the holons at the vicinity of the change (see figure~\ref{fig:address-update-exm}).
	\end{proof}
	
	\begin{theorem}[Soundness and Completeness of the Testing Algorithm]
		Giving the testing algorithm all the information it needs to operate, it will provide a correct result whenever there exists one and a proper warning otherwise.
	\end{theorem}
	\begin{proof}
		As it can be found out in the testing algorithm, the key decision at any holon of the holarchy is made based on equation~\ref{eq:testing-check}. Therefore, to show that the testing algorithm is sound and complete, we must prove that equation~\ref{eq:testing-check} properly determines whether the holarchy is capable of answering the query. On account of the training steps, all of the holons with level $\ge 2$ in the holarchy share the same name. As a result, at any node, if the name of the query algorithm is not the same as the name of the holon (the first part of equation~\ref{eq:testing-check}), it means there will be no answer through that holon, thus the holon blocks the flow of the query to its sub-holons and returns an empty set as the response. Likewise, according to equations~\ref{eq:skill-model}, \ref{eq:skill-other}, and \ref{eq:skill-abstract}, the skills(the trained data) of any AlgH holon is the union of all of its sub-holons' skills. Consequently, if the skills of holon do not satisfy the requirements of the query, the third part of the equation makes the holon properly respond to the super-holon.
		
		In case that both name and skill checks pass, the final decision is made by the second part of equation~\ref{eq:testing-check}. According to the training algorithm, the capabilities of any holon at level $\ge 2$ of the holarchy is the parametric sum of the capabilities of its sub-holon. By lemma~\ref{lem:4} and corollary~\ref{cor:41}, on the other hand, the holon will forward the test query to its sub-holons if any of its accessible subordinate atomic holons is capable of performing the test. Hence, the algorithm will find the right destination to execute the testing operation and return the appropriate result. Furthermore, pursuant to corollary~\ref{cor:42}, the testing process will be stopped and returned suitably utilizing the aforementioned equation. Hence, it would be impossible, for the testing method, to return the wrong result.
	\end{proof}
	\subsection{Complexity}
	This section discusses the computational complexity of the proposed platform based on space and computational time criteria. This is done for both training and testing algorithms separately.
	\subsubsection{Training Algorithm} 
	The training algorithm has two vertical passes in the holarchy. Since the passes in each of the data and algorithm sub-holarchies are carried out in parallel, we first take one of them into consideration, and model it as a tree, in which, the sub-holons of a holon are forming the children of its corresponding node. Figure~\ref{fig:complexity-worst-tree} depicts two trees corresponding two example holarchies that we delve into in our analysis. Let's assume that the maximum number of children (sub-holons) that an algorithm/data node in such a tree has is denoted by $b_a$ and $b_d$ respectively. Similarly, assume that the current number of the leaf nodes (atomic holons) in each of the algorithm/data sub-trees is respectively shown by $n_a$ and $n_d$. In the worst-case scenario, during the first pass %{\textcircled{\small 1}}
	, the holons need to wait for all of their sub-holons' proposals before they choose the best one. This scenario is similar to the one in which we check all of the children of a tree node before expanding the last one, as depicted by arrows in each picture of figure~\ref{fig:complexity-worst-tree}. For the sake of brevity, let's focus on the algorithm sub-tree first. In a complete tree (figure~\ref{fig:complexity-comp}), the height will be:
	\begin{equation}
		h_a=\log_{b_a}{n_a}
	\end{equation}
	since at each level we check $b_a$ nodes, the time complexity of the first pass of the training algorithm will be:
	\begin{equation}
		\mathcal{O}(b_a\cdot h_a)=\mathcal{O}(b_a\cdot \log_{b_a}{n_a})
	\end{equation}
	In an extreme case where $b_a-1$ nodes of each level are terminal, the height would be: 
	\begin{equation}
		h_a\frac{n_a-b_a}{b_a-1}+1=\frac{n_a-1}{b_a-1}    
	\end{equation} 
	and therefore, the complexity of the first pass becomes:
	\begin{equation}
		\mathcal{O}(b_a\cdot \frac{n_a-1}{b_a-1})=\mathcal{O}(n_a)
	\end{equation}
	Similarly, both of above mentioned tree layout on the data sub-tree will yield $\mathcal{O}(b_d\cdot \log_{b_d}{n_d})$ and $\mathcal{O}(n_d)$ respectively. Regarding the second pass of the algorithm, since it only needs to follow the addresses without further checking the children, the time complexity for the same tree layouts, would be $\mathcal{O}(\log_{b_a}{n_a})$ and $\mathcal{O}(n_a)$ for the algorithm sub-tree and $\mathcal{O}(\log_{b_d}{n_d})$ and $\mathcal{O}(n_d)$ for the data sub-tree, respectively. Considering the deepest tree layout (figure~\ref{fig:complexity-deep}) and the fact that the second pass starts after the first pass finishes, the worst-case time complexity of the training algorithm for the algorithm sub-holarchy, would be:
	\begin{equation}
		\mathcal{O}(n_a+n_a)=\mathcal{O}(n_a)
	\end{equation}
	Finally, taking both of DATA and ALG sub-holarchies into consideration, and recalling the fact that training at each sub-holarchy is run in parallel, the worst-case time complexity of the training algorithm would be:
	\begin{equation}
		\mathcal{O}(\max(n_a,n_d))
	\end{equation}
	
	\begin{figure}
		\centering
		\begin{subfigure}[b]{0.4\textwidth}
			\centering
			\includegraphics[height=.25\textheight]{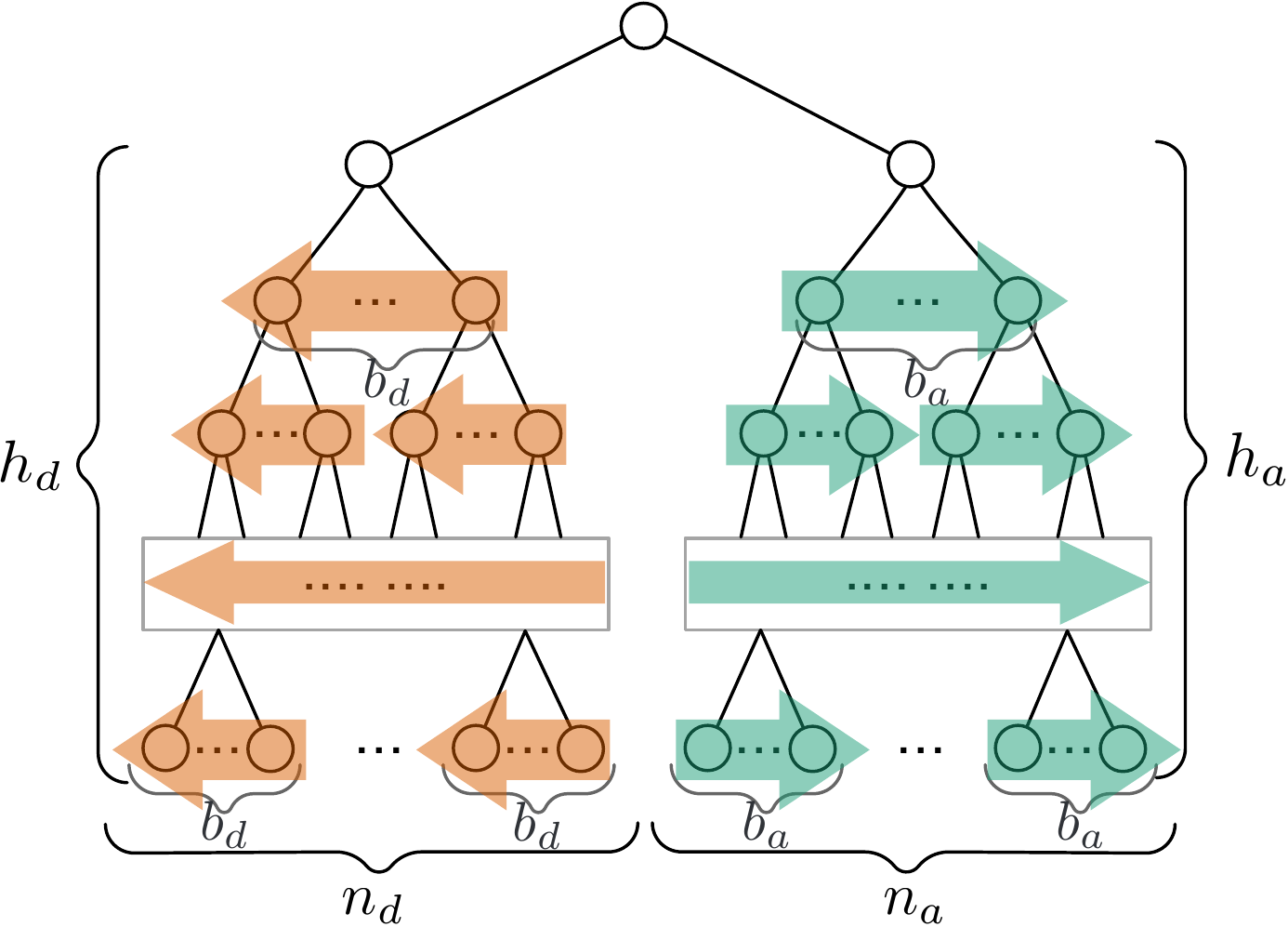}
			\caption{a complete tree}
			\label{fig:complexity-comp}
		\end{subfigure}\hspace{.1\textwidth}%
		\begin{subfigure}[b]{0.4\textwidth}
			\centering
			\includegraphics[height=.25\textheight]{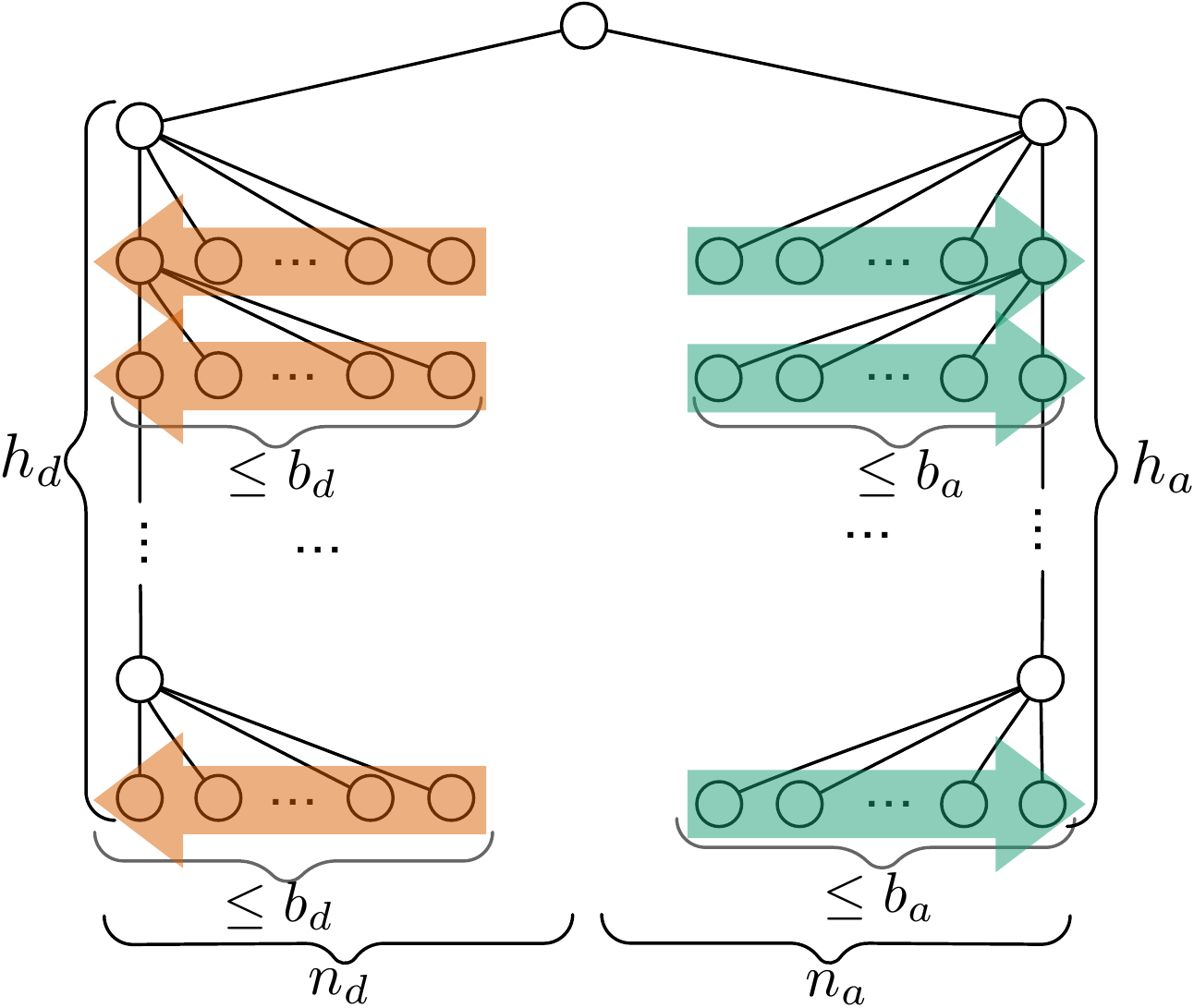}
			\caption{a deep tree}
			\label{fig:complexity-deep}
		\end{subfigure}
		\caption{Tree representations of two extreme hierarchical structures. (a) depicts the case in which each parent has the same number of subordinates, and (b) represents the a scenario in which only of the nodes at each level has a subordinate. In both of the figures, the arrows are assumed to be the order that the nodes are processed in each level.}
		\label{fig:complexity-worst-tree}
	\end{figure}
	
	Assuming a holarchy with $n_a$ and $n_d$ atomic algorithm and data holons respectively, at any time that we run the training query, the worst-case scenario for the space complexity occurs when there is no holon representing neither of the queried algorithm nor dataset. In this case, in addition to using a fixed amount of memory in each holon, 5 new holons are created: 2 super-holons, 2 new holons for the new algorithm and dataset, and 1 model holon. Therefore, the space complexity in terms of the number of new entities would be $\mathcal{O}(1)$. On the other hand, the amount of the memory that is used by the holons during the passes of the training process is $\mathcal{O}(\max(n_a,n_d))$, as it uses a fixed amount of memory for each step of the procedure until it creates the holon and begins the procedure. It is important to note that, we have ignored the complexities of the data mining algorithm and used datasets in our calculations.
	
	\subsubsection{Testing Algorithm} 
	The computational complexity analysis of the testing algorithm is very similar to that of the training process. Again, let's assume that $b_a, b_d, n_a$, and $n_d$ respectively denote the maximum branching factor of a node in algorithm and data sections, and the total number of leaf nodes representing the trained algorithms and the stored datasets in the corresponding tree model of the holarchy. Regarding the structure, we make the same hypothesis that we made before and use the layout depicted in figure~\ref{fig:complexity-worst-tree}. For a query with non-general parameters, the first step of the testing algorithm is to locate the testing dataset and get its address. This needs two travels in the depth of the DATA sub-holarchy to find the dataset and return its address, and two travels in the ALG portion and return the results. In case of a complete tree in each of DATA and ALG sections, these can be done in $\mathcal{O}(b_d\cdot\log_{b_n}{n_d})$ and $\mathcal{O}(b_a\cdot\log_{b_a}{n_a})$ time respectively. As we need to finish the data search process before we start testing and traveling in the ALG tree, the total complexity for such a tree would be 
	\begin{equation}
		\mathcal{O}(b_d\cdot\log_{b_d}{n_d} + b_a\cdot\log_{b_a}{n_a})
	\end{equation}
	On the other hand, if only one of the nodes in each level is expanded and holds children, the height of the each DATA and ALG sub-holarchies, will be $\frac{n_d-1}{b_d-1}$ and $\frac{n_a-1}{b_a-1}$ respectively, and the worst-case required time for the passes in each will be $\mathcal{O}(\frac{n_d-1}{b_d-1})$ and $\mathcal{O}(\frac{n_a-1}{b_a-1})$ accordingly. As a result, the overall testing complexity becomes 
	\begin{equation}
		\mathcal{O}(b_d\cdot\frac{n_d-1}{b_d-1} + b_a\cdot\frac{n_a-1}{b_a-1})=\mathcal{O}(n_d+n_a)
	\end{equation}
	
	During the testing procedure, no new holon is created, but the existing ones are used. As a result, the space complexity of the testing operation will be the amount of the memory that is used by each holon. This amount is fixed, i.e. $\mathcal{O}(1)$, and for each data or algorithm component. Consequently, the space for the entire testing process will be solely a function of the number of steps in the process. In other words, for the same structural settings that we discussed above, the worst-case space complexity of the entire procedure will be $\mathcal{O}(b_d\cdot\log_{b_d}{n_d} + b_a\cdot\log_{b_a}{n_a})$ and $\mathcal{O}(b_d\cdot\frac{n_d-1}{b_d-1} + b_a\cdot\frac{n_a-1}{b_a-1})=\mathcal{O}(n_d+n_a)$.
	
	As it was stated in the beginning, the above-mentioned analysis of the testing method is for the case that a non-general query is sent to the system. As the proposed testing operation is capable of carrying out multiple sub-tasks at once, thanks to the intrinsic distributed property of multi-agent systems, we expect to save time conducting general tests. For instance, assume that in the worst case the query is intended to test all the algorithms, i.e. $\Lambda=\{(*, \{*\})\}$, on all the available datasets, i.e. $\Delta=\{(*, \{*\})\}$. Furthermore, let's hypothesize that, being of the same type each, all the existing algorithms have been previously trained on all the available datasets. With the same notations as before, such a holarchy, will have $n_d\cdot n_a$ model holons in addition to $n_d$ atomic data and $n_a$ atomic algorithm holons. The above mentioned general test query on such a holarchy, can be decomposed to $n_d\cdot n_a$ non-general test queries, and if we feed them to the system separately, the time complexity for each of the mentioned complete and non-complete holarchical structures will be
	\begin{equation}
		\mathcal{O}(n_d\cdot n_a(b_d\cdot\log_{b_d}n_d + b_a\cdot\log_{b_a}n_a))
	\end{equation} and 
	\begin{equation}
		\mathcal{O}(n_d\cdot n_a(b_d\cdot\frac{n_d-1}{b_d-1} + b_a\frac{n_a-1}{b_a-1}))=\mathcal{O}({n_d}^2\cdot n+n_d\cdot {n_a}^2)
	\end{equation} in the given order. On the other hand, having a holarchy with all its algorithms trained on all its data, implies that $n_d\le n_a$. Therefore, the worst case complexities will become $\mathcal{O}({n_a}^2(b_a\cdot\log_{b_a}n_a))$ and $\mathcal{O}({n_a}^3)$ respectively. In contrast, being able to process all the queries at once, the proposed testing method will employ all of the holons in its structure to process such a general query in parallel. Let's consider the complete holarchy structure. The total number of holons in each algorithm and data sections will be the sum of the number of atomic and composite holons in each part. Taking the number of trained models into the account, the total number of holons in the holarchy, except the SYS, DATA, and ALG will be:
	\begin{equation}
		\label{eq:total-complete}
		\begin{split}
			n_d\cdot n_a + \sum_{i=0}^{\log_{b_d}n_d}\frac{n_d}{(b_d)^i} + & \sum_{i=0}^{\log_{b_a}n_a}\frac{n_a}{(b_a)^i} =\\ & n_d\cdot n_a + \frac{b_d\cdot n_d-1}{b_d-1} + \frac{b_a\cdot n_a-1}{b_a-1}
		\end{split}
	\end{equation}
	Consequently, based on the fact that in such a dense holarchy $n_d\le n_a$, the worst case time complexity to process the given general test query will be $\mathcal{O}({n_a}^2)$, which is less than the aforementioned $\mathcal{O}({n_a}^2(b_a\cdot\log_{b_a}n_a))$. Likewise, for the non-complete holarchical tree that we used in previous analyses, on account of the fact that the total number of holons is
	\begin{equation}
		\label{eq:total-noncomplete}
		n_d\cdot n_a + b_d\cdot\frac{n_d-1}{b_d-1} + b_a\cdot\frac{n_a-1}{b_a-1}\quad,
	\end{equation}
	the worst case time complexity remains the same, that is $\mathcal{O}({n_a}^2)<\mathcal{O}({n_a}^3)$.
	
	\section{Experimental Evaluation}
	Our presented distributed machine learning model is not bounded by a specific multi-agent development framework. In fact, each module and section of the entire system can be developed using different tools; and as long as they keep following the suggested protocols and training/testing procedures, the system would demonstrate the expected behavior. For the sake of experimentation, we used Smart Python Agent Development Environment (SPADE)\footnote{SPADE's source code is available at https://github.com/javipalanca/spade} \cite{gregori2006jabber} to implement our platform \footnote{The source code of our implementation of HAMLET and the experiments are available at https://github.com/aesmaeili/HAMLET}. This is mainly due to its simple yet flexible model and the abundance of machine learning and data mining libraries available in Python programming language. 
	
	SPADE is fully FIPA compliant and supports asynchronous running of agents together with the inter-agent communications based on the Extensible Messaging and Presence Protocol (XMPP). Moreover, it provides a set of helpful features that facilitate the deployment of our platform on a network of computers. Some of such characteristics are \cite{gregori2006jabber}:
	\begin{itemize}
		\item the flexibility in inter-operating with other agent development platforms, thanks to its FIPA compliance;
		\item multi-user Conference (MUC) that provides the capability to create forums of agents;
		\item providing multiple built-in behavior models, such as cyclic, recurring, one-shot, timeout, and event-based finite state machines; and 
		\item featuring customized agent presence notification and P2P agent communication capability.
	\end{itemize}
	
	SPADE's platform and agent models are depicted in figure~\ref{fig:spade}. As it can be seen, its main platform is outlined based on the multi-agent architecture standards recommended by FIPA \cite{fipa2002arch}. To put it concisely, the platform (figure~\ref{fig:spade-platform}) is composed of four components: the Agent Management System (AMS) to supervise SPADE, the Directory Facilitator (DF) to provide information about the agents and their services, the Agent Communication Channel (ACC) to manage the communications between the agents and system components, and the XML router as the Message Transport System (MTS). On the other hand, the agent model (figure~\ref{fig:spade-agent}) comprises tasks, as the executable processes of the agent; and the message dispatcher that collects the arrived messages and redirects them to the appropriate task queues.
	
	\begin{figure}[h]
		\centering
		\begin{subfigure}[b]{0.4\textwidth}
			\centering
			\includegraphics[height=.25\textheight]{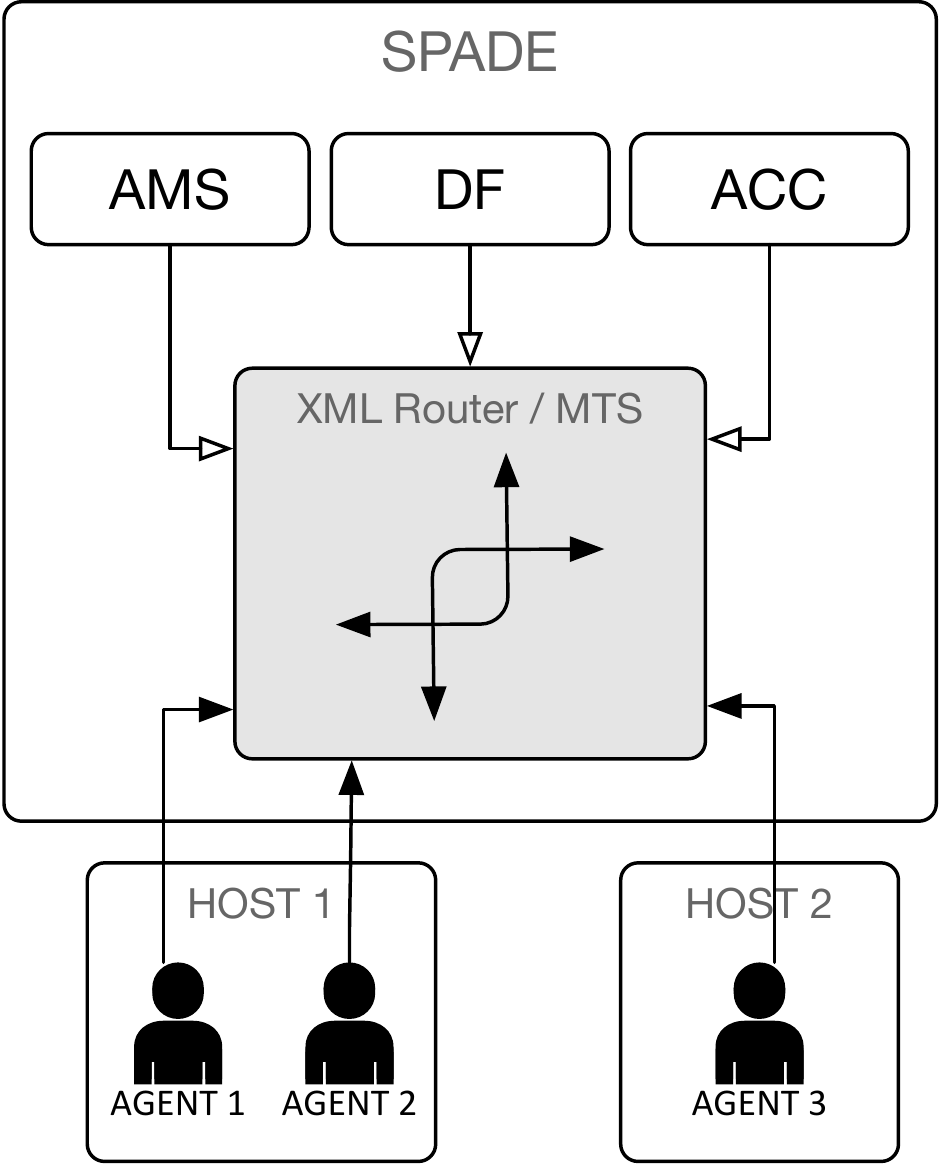}
			\caption{The platform model}
			\label{fig:spade-platform}
		\end{subfigure}\hspace{.1\textwidth}%
		\begin{subfigure}[b]{0.4\textwidth}
			\centering
			\includegraphics[height=.25\textheight]{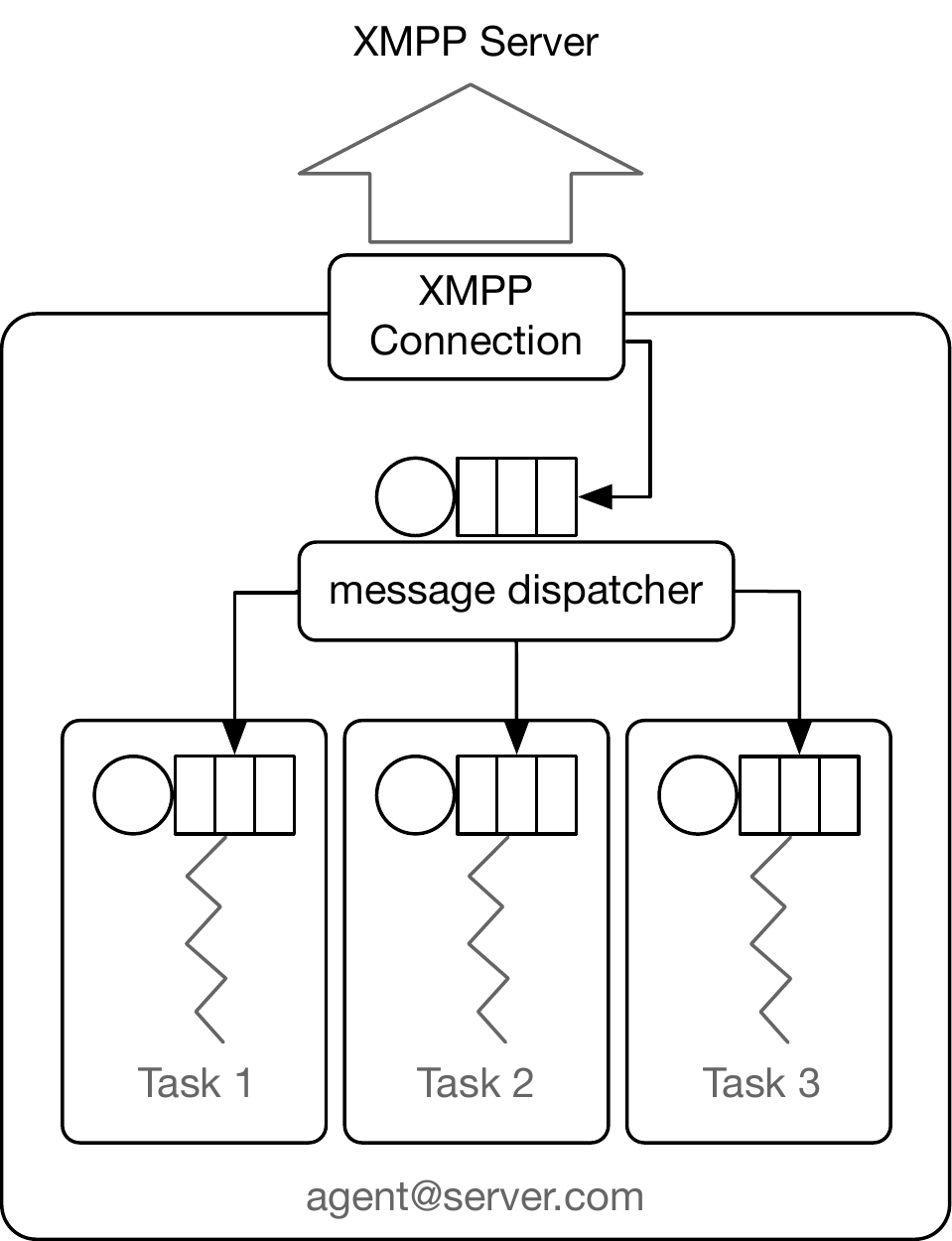}
			\caption{The agent model}
			\label{fig:spade-agent}
		\end{subfigure}
		\caption{The architecture of the SPADE agent development framework \cite{gregori2006jabber}}
		\label{fig:spade}
	\end{figure}
	
	The initial high-level outline of the the test environment, as depicted in figure~\ref{fig:experiment-setting}, is composed of the SYS, DATA, and ALG holons together with two additional PRS and VIZ agents that have the responsibilities of processing the received queries and generating plots for the results respectively. Additionally, the User agent depicted in this schema simulates the role of an external human agent and basically acts as an automatic query generator. The queries sent to the system are handed over the PRS agent to validate the basic query structures and inform about any potential errors before beginning the task. As soon as the result tensors are available after conducting a machine learning task, They are sent to the VIZ agent to plot and deliver them according to a pre-specified format.
	
	\begin{figure}[h]
		\centering
		\includegraphics[height=.2\textheight]{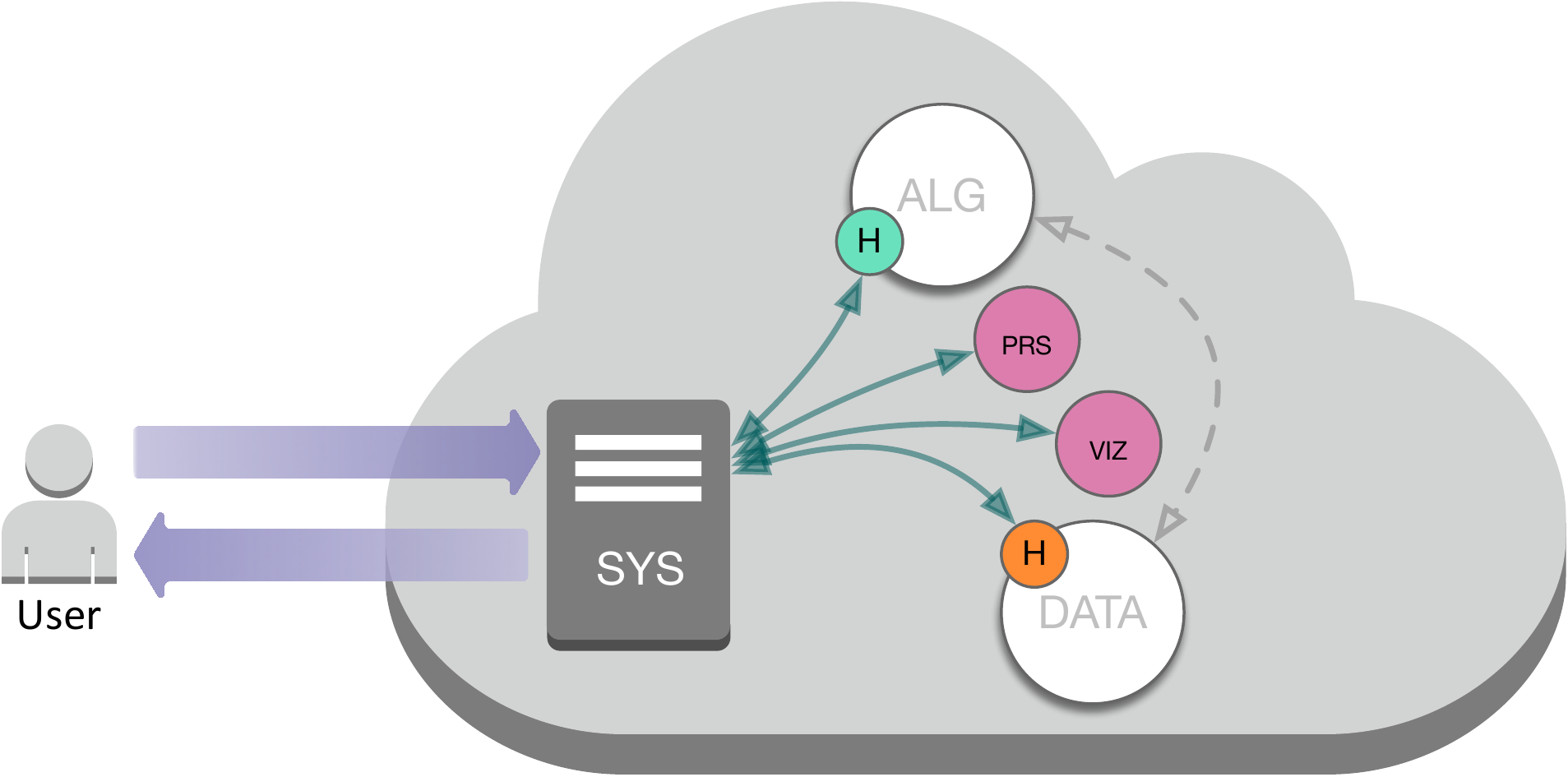}
		\caption{The initial components of the experiment.}
		\label{fig:experiment-setting}
	\end{figure}
	
	In order to develop the holonic structure, we specialized SPADE's agent model by adding the internal components, summarized in figure~\ref{fig:holon_arc} and table~\ref{tbl:parts}, to the built-in elements such as the message dispatcher. Figure~\ref{fig:hamlet-class} illustrates a high-level view of the implemented classes and their relationships with each other. The shaded area in this diagram particularly highlights the components that HAMLET implements based on its architecture. Class \emph{Holon}, as an abstract class, defines the basic data structures, properties, and behaviors common in all holon types. The children of this class try to customize the provided basic architecture and interfaces with more specific and task oriented components. Furthermore, separating different holon types in different classes helps us effectively enforce the restrictions, such the multiplicity and holarchical relationships, defined in HAMLET. The other classes presented in figure~\ref{fig:hamlet-class}, namely \emph{PRS}, \emph{VIZ}, and \emph{User}, are technically the SPADE agents with particular functionalities. It is worth noting that due to the flexible structure of HAMLET, the functionality and services of the test bed can be easily extended by adding more agents and properly connecting them to the other holons.
	
	\begin{figure}[h]
		\centering
		\includegraphics[height=.35\textheight]{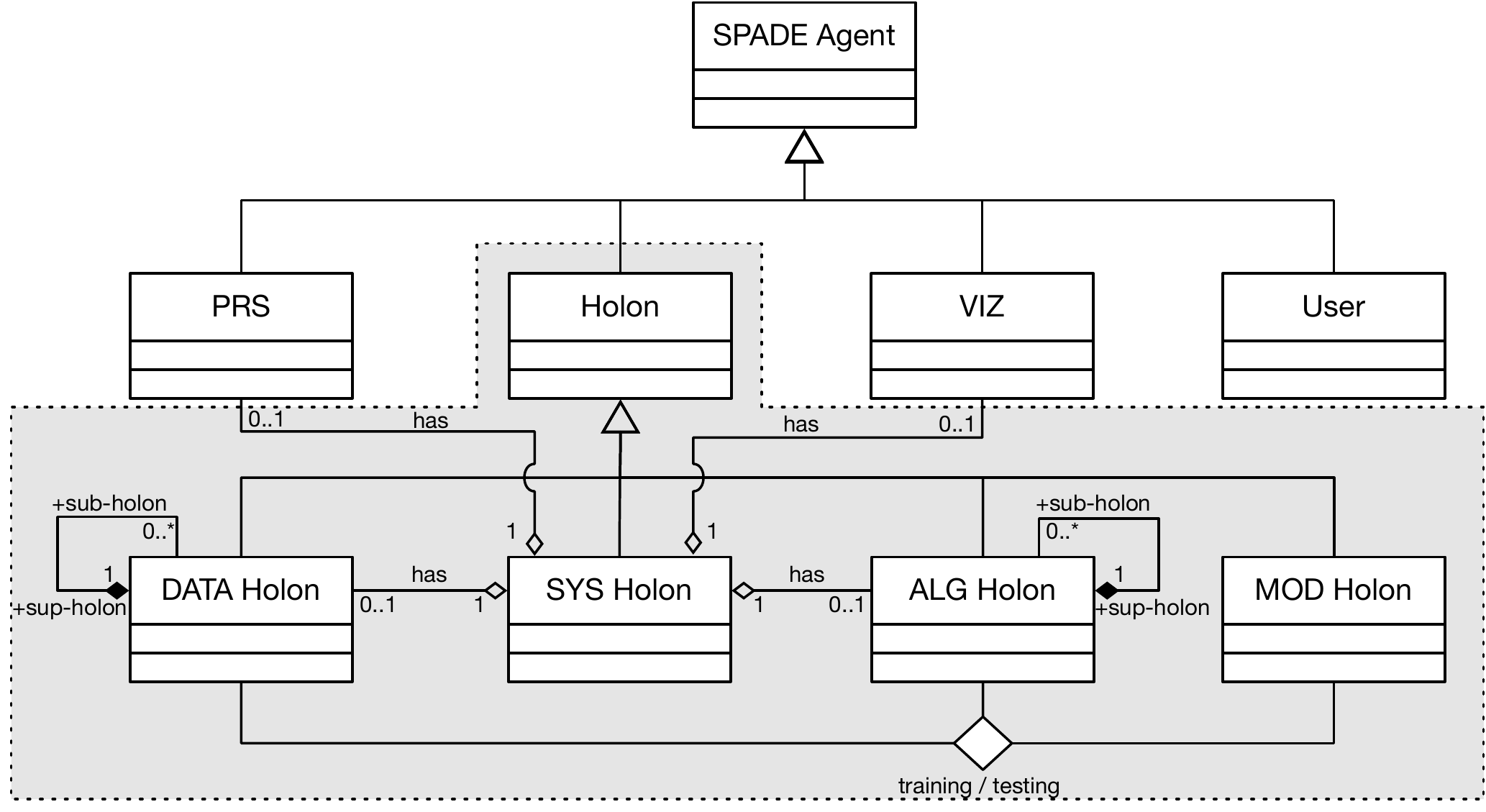}
		\caption{The classes diagram of the implemented HAMLET elements and their relationship.}
		\label{fig:hamlet-class}
	\end{figure}
	
	The experiment makes use of several classification, regression, and clustering algorithms being trained and tested on several standard datasets. All of the used algorithms are from the scikit-learn \cite{scikit-learn} library, and their details are listed in tables~\ref{tbl:alg-details} and ~\ref{tbl:data-details} respectively. In table~\ref{tbl:alg-details}, for each algorithm, the columns \textit{parameters} and \textit{id} respectively hold the list of the used parameters and the identifier that is used for the presentation purposes. Please note that, for the sake of space limitation and clarity, we have only listed the parameters that have different values from the default ones and we have listed the complete list of the parameters for each algorithm as the note in table~\ref{tbl:alg-details}. Additionally, table~\ref{tbl:data-details} lists the datasets that are used for each machine learning task. For classification and regression, we have divided each dataset into training and testing sets consisting of 60\% and 40\% of the original instances respectively. For the clustering task, however, we have utilized 100\% of the data. Last but not least, all the experiments have been carried out in a PC with Intel Core-i5 @1.6GHz CPU and 16GB RAM running Ubuntu OS and python 3.7.
	
	\newcommand{\ra}[1]{\renewcommand{\arraystretch}{#1}}
	\begin{table}[!htbp]\centering
		\caption{The details of the used algorithms.}
		\label{tbl:alg-details}
		\ra{1.3}
		\resizebox{\textwidth}{!}{
			\begin{threeparttable}
				\begin{tabular}{@{}rrrrcrrrcrrr@{}}\toprule
					& \multicolumn{3}{c}{\emph{classification}} & \phantom{abc}& \multicolumn{3}{c}{\emph{regression}} &
					\phantom{abc} & \multicolumn{3}{c}{\emph{clustering}}\\
					\cmidrule{2-4} \cmidrule{6-8} \cmidrule{10-12}
					& \emph{id} & \emph{name} & \emph{parameters} && \emph{id} & \emph{name} & \emph{parameters} && \emph{id} & \emph{name} & \emph{parameters}\\ \midrule
					& \textbf{A01} & SVC\tnote{1} & kernel=linear && \textbf{A09} & Linear\tnote{6} & defaults && \textbf{A17} & KMeans\tnote{12} & defaults\tnote{*}\\
					& \textbf{A02}& SVC& kernel=sigmoid&& \textbf{A10}& Ridge\tnote{7}& fit\_inercept=False&& \textbf{A18}& KMeans& algorithm=full\tnote{*}\\
					& \textbf{A03}& SVC& $\gamma=0.001$&& \textbf{A11}& Ridge& $\alpha=0.5$&& \textbf{A19}& MBKMeans\tnote{13}& defaults\tnote{*}\\
					& \textbf{A04}& SVC& $C=100,\gamma=0.001$&& \textbf{A12}& KRR\tnote{8}& defaults&& \textbf{A20}& DBSCAN\tnote{14}& defaults\\
					& \textbf{A05}& NuSVC\tnote{2}& defaults&& \textbf{A13}& Lasso\tnote{9}& $\alpha=0.1$&& \textbf{A21}& DBSCAN& metric=cityblock\\
					& \textbf{A06}& ComNB\tnote{3}& defaults&& \textbf{A14}& NuSVR\tnote{10}& defaults&& \textbf{A22}& DBSCAN& metric=cosine\\
					& \textbf{A07}& DTree\tnote{4}& defaults&& \textbf{A15}& NuSVR& $\nu=0.1$&& \textbf{A23}& Birch\tnote{15}& defaults\\
					& \textbf{A08}& NrCent\tnote{5}& defaults&& \textbf{A16}& ElasNet\tnote{11}& defaults&& \textbf{A24}& HAC\tnote{16}& defaults\tnote{*}\\
					\bottomrule
				\end{tabular}
				\begin{tablenotes}
					\small
					\item[1] \textbf{C-Support Vector Classification}\cite{chang2011libsvm}. Defaults:(C=1.0, kernel='rbf', degree=3, $\gamma$='scale', coef0=0.0, shrinking=True, probability=False, tol=0.001, cache\_size=200, class\_weight=None, verbose=False, max\_iter=-1, decision\_function\_shape='ovr', break\_ties=False)\cite{scikit-learn-web}
					\item[2] \textbf{Nu-Support Vector Classification}\cite{chang2011libsvm}. Defaults:(nu=0.5, kernel='rbf', degree=3, $\gamma$='scale', coef0=0.0, shrinking=True, probability=False, tol=0.001, cache\_size=200, class\_weight=None, verbose=False, max\_iter=-1, decision\_function\_shape='ovr', break\_ties=False)\cite{scikit-learn-web}
					\item[3] \textbf{Complement Naive Bayes classifier}\cite{rennie2003tackling}. Defaults:($\alpha$=1.0, fit\_prior=True, class\_prior=None, norm=False)\cite{scikit-learn-web}
					\item[4] \textbf{Decision Tree Classifier}\cite{hastie2009elements}. Defaults:(riterion='gini', splitter='best', max\_depth=None, min\_samples\_split=2, min\_samples\_leaf=1, min\_weight\_fraction\_leaf=0.0, max\_features=None, random\_state=None, max\_leaf\_nodes=None, min\_impurity\_decrease=0.0, min\_impurity\_split=None, class\_weight=None, presort='deprecated', ccp\_alpha=0.0)\cite{scikit-learn-web}
					\item[5] \textbf{Nearest Centroid Classifier}\cite{tibshirani2002diagnosis}. Defaults:(metric='euclidean', shrink\_threshold=None)\cite{scikit-learn-web}
					\item[6] \textbf{Ordinary Least Squares Linear Regression}. Defaults:(fit\_intercept=True, normalize=False, copy\_X=True, n\_jobs=None)\cite{scikit-learn-web}
					\item[7] \textbf{Ridge Regression}\cite{hoerl1970ridge}. Defaults:($\alpha=1.0$, fit\_intercept=True, normalize=False, copy\_X=True, max\_iter=None, tol=0.001, solver=auto)\cite{scikit-learn-web}
					\item[8] \textbf{Kernel Ridge Regression}\cite{murphy2012machine}. Defaults: ($\alpha$=1, kernel='linear', $\gamma$=None, degree=3, coef0=1, kernel\_params=None)\cite{scikit-learn-web}
					\item[9] \textbf{least Absolute Shrinkage and Selection Operator}\cite{tibshirani1996regression}. Defaults:($\alpha$=1.0, fit\_intercept=True, normalize=False, precompute=False, copy\_X=True, max\_iter=1000, tol=0.0001, warm\_start=False, positive=False, random\_state=None, selection='cyclic')\cite{scikit-learn-web}
					\item[10] \textbf{Nu Support Vector Regression}\cite{chang2011libsvm}. Defaults:($\nu$=0.5, C=1.0, kernel='rbf', degree=3, $\gamma$='scale', coef0=0.0, shrinking=True, tol=0.001, cache\_size=200, verbose=False, max\_iter=-1)\cite{scikit-learn-web}
					\item[11] \textbf{Elastic Net Regression}\cite{zou2005regularization}. Defaults:($\alpha$=1.0, l1\_ratio=0.5, fit\_intercept=True, normalize=False, precompute=False, max\_iter=1000, copy\_X=True, tol=0.0001, warm\_start=False, positive=False, random\_state=None, selection='cyclic')\cite{scikit-learn-web}
					\item[12] \textbf{K-Means Clustering}\cite{lloyd1982least}. Defaults:(n\_clusters=8, init='k-means++', n\_init=10, max\_iter=300, tol=0.0001, precompute\_distances='deprecated', verbose=0, random\_state=None, copy\_x=True, n\_jobs='deprecated', algorithm='auto')\cite{scikit-learn-web}
					\item[13] \textbf{Mini-Batch K-Means Clustering}\cite{sculley2010web}. Defaults:(n\_clusters=8, init='k-means++', max\_iter=100, batch\_size=100, verbose=0, compute\_labels=True, random\_state=None, tol=0.0, max\_no\_improvement=10, init\_size=None, n\_init=3, reassignment\_ratio=0.01)\cite{scikit-learn-web}
					\item[14] \textbf{Density-Based Spatial Clustering of Applications with Noise}\cite{ester1996density}. Defaults:($\epsilon$=0.5, min\_samples=5, metric='euclidean', metric\_params=None, algorithm='auto', leaf\_size=30, p=None, n\_jobs=None)\cite{scikit-learn-web}
					\item[15] \textbf{ Birch Clustering}\cite{zhang1996birch}. Defaults:(threshold=0.5, branching\_factor=50, n\_clusters=3, compute\_labels=True, copy=True)\cite{scikit-learn-web}
					\item[16] \textbf{Hierarchical Agglomerative Clustering}\cite{rokach2005clustering}. Defaults:(n\_clusters=2, affinity='euclidean', memory=None, connectivity=None, compute\_full\_tree='auto', linkage='ward', distance\_threshold=None)\cite{scikit-learn-web}
					\item[*] The number of the clusters is set equal to the number of true classes
				\end{tablenotes}
			\end{threeparttable}
		}
	\end{table}

	\begin{table}\centering
		\caption{The details of the used datasets.}
		\label{tbl:data-details}
		\ra{1.3}
		\resizebox{\textwidth}{!}{
			\begin{threeparttable}
				\begin{tabular}{@{}llcccccc@{}}\toprule
					&\emph{name} & \emph{classes/targets} & \emph{samples per class} & \emph{total samples} & \emph{dimensionality} & \emph{features}\\ \midrule
					\multicolumn{2}{l}{\textbf{classification:}}\\
					&\textit{Iris} \cite{fisher1936use}\tnote{*} & 3 & [50,50,50] & 150 & 4 & real, positive\\
					&\textit{Wine} \cite{lichman2013uci}\tnote{*}& 3& [59,71,48]& 178& 13& real, positive\\
					&\textit{Breast cancer} \cite{wolberg1994machine}& 2& [212, 358]& 569& 30& real, positive\\
					&\textit{Digits} \cite{alpaydin1998cascading}\tnote{*,$\dagger$}& 10& about 180& 1797& 64& integers [0, 16]\\
					&\textit{Art. Class.}\tnote{*,1}& 3& [300,300,300]& 900& 20& real (-7.3, 8.9)\tnote{**}\\
					&\textit{Art. Moon}\tnote{*,2}& 2&[250,250]& 500& 2& real (-1.2, 2.2)\tnote{**}\\
					\multicolumn{2}{l}{\textbf{regression:}}\\
					&\textit{Boston} \cite{harrison1978hedonic}& real [5, 50] & --& 506& 13& real, positive\\
					&\textit{Diabetes} \cite{efron2004least}& integer [25, 346]& --& 442& 10& real (-0.2, 0.2)\\
					&\textit{Art. Regr.}\tnote{3}& real (-488.1, 533.2)& --& 200& 20& real (-4,4)\\
					\bottomrule
				\end{tabular}
				\begin{tablenotes}
					\small
					
					\item[*] Also used in clustering experiments.
					\item[**] In order to prevent the problem that some algorithms have with negative features, the values are normalized into [0,1].
					\item[$\dagger$] This is a copy of the test set of the UCI ML hand-written digits datasets.
					\item[1] Artificially made using \texttt{make\_classification} function of scikit-learn library \cite{scikit-learn-web}.
					\item[2] Artificially made using \texttt{make\_moons} function of scikit-learn library \cite{scikit-learn-web}.
					\item[3] Artificially made using \texttt{make\_regression} function of scikit-learn library \cite{scikit-learn-web}.
				\end{tablenotes}
			\end{threeparttable}
		}
	\end{table}
	
	\subsection{Training}
	The first set of queries pertains to training the aforementioned classification, regression, and clustering algorithms on the available datasets, and finally adding test datasets to the system. For the sake of creating a high-loaded experiment, we trained all the algorithms on all of their corresponding datasets and depicted the structure of the resulted holarchy at the end of each task as in figure~\ref{fig:exp-holarchy-all}. The figures on the left hand side column represent the structure of the holarchy without considering the model holons, and the ones on the right hand side focus solely on the interconnection between the atomic algorithm/data holons and the created models at the lowest level of the holarchy. These two views are provided mainly for the sake of clarity. Furthermore, the sub-figures at each row pertain to the holarchy after a particular task, i.e., sub-figures~\ref{fig:exp-holarchy-classification-a} and \ref{fig:exp-holarchy-classification-b} are for after training the classification algorithms, sub-figures~\ref{fig:exp-holarchy-regression-a} and \ref{fig:exp-holarchy-regression-b} are for after training the regression algorithms, sub-figures~\ref{fig:exp-holarchy-clustering-a} and \ref{fig:exp-holarchy-clustering-b} are for after running the clustering algorithms, and finally, sub-figures~\ref{fig:exp-holarchy-data-add-a} and \ref{fig:exp-holarchy-data-add-b} are for after adding test data. There are a few additional points about the figures worth noting. Firstly, in figure~\ref{fig:exp-holarchy-regression-b}, the structure is composed of two separate communities, which is because of the fact that the classification and regression algorithms do not share any dataset during the training phase. Secondly, in figure~\ref{fig:exp-holarchy-data-add-b}, there are some DataH holons in the vicinity of communities that are not connected to any other holons. These are the test datasets that are added to the holarchy and because they are not employed in any training task, they do not hold any connections to the model holons. And finally, the nodes of the structures are positioned automatically by the visualization algorithm, therefore, the algorithm sub-structures in figures~\ref{fig:exp-holarchy-clustering-a} and \ref{fig:exp-holarchy-data-add-a} are exactly the same, though they are drawn differently.  
	\begin{figure}
		% \centering
		\begin{subfigure}[b]{0.49\textwidth}
			\centering
			\includegraphics[page=1,width=.9\textwidth]{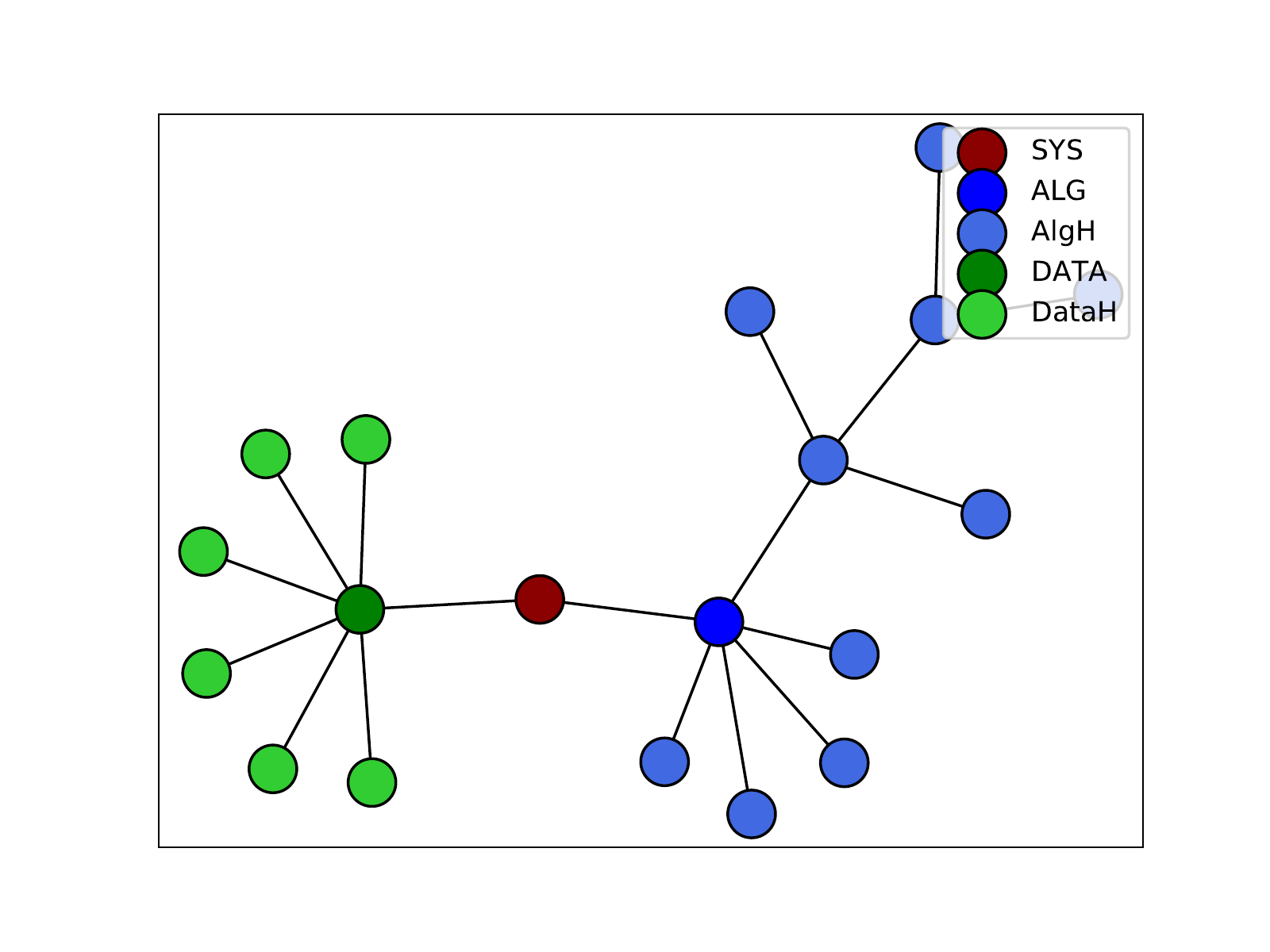}
			\vspace{-.6cm}\caption{~}
			\label{fig:exp-holarchy-classification-a}
		\end{subfigure}%\hspace{.1\textwidth}
		\begin{subfigure}[b]{0.49\textwidth}
			\centering
			\includegraphics[page=2,width=.9\textwidth]{figs/exps/Network-After-Classification.pdf}
			\vspace{-.6cm}\caption{~}
			\label{fig:exp-holarchy-classification-b}
		\end{subfigure}\\\vspace{-.13cm}
		\begin{subfigure}[b]{0.49\textwidth}
			\centering
			\includegraphics[page=1,width=.9\textwidth]{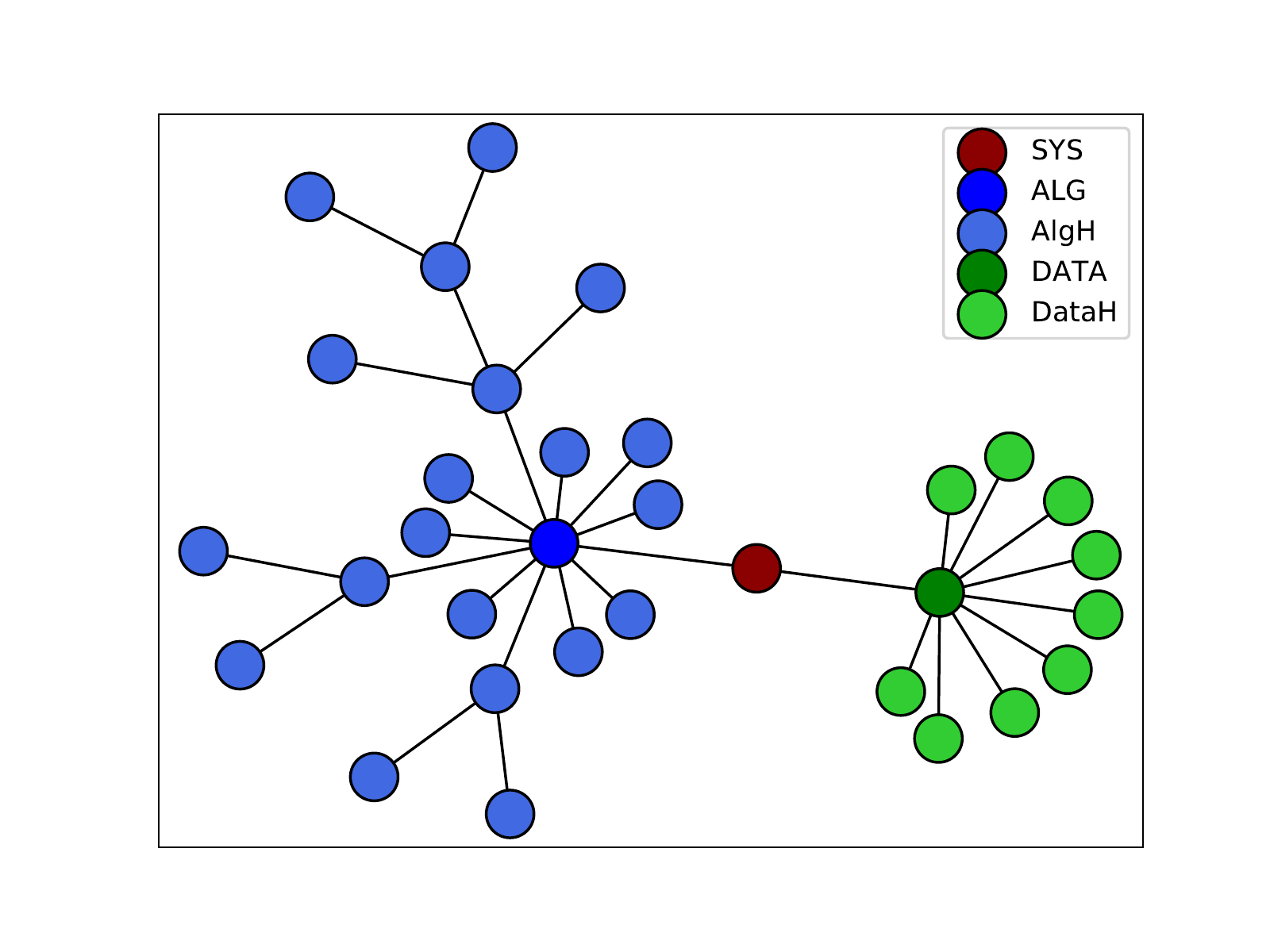}
			\vspace{-.6cm}\caption{~}
			\label{fig:exp-holarchy-regression-a}
		\end{subfigure}
		\begin{subfigure}[b]{0.49\textwidth}
			\centering
			\includegraphics[page=2,width=.9\textwidth]{figs/exps/Network-After-Regression.pdf}
			\vspace{-.6cm}\caption{~}
			\label{fig:exp-holarchy-regression-b}
		\end{subfigure}\\\vspace{-.13cm}
		\begin{subfigure}[b]{0.49\textwidth}
			\centering
			\includegraphics[page=1,width=.9\textwidth]{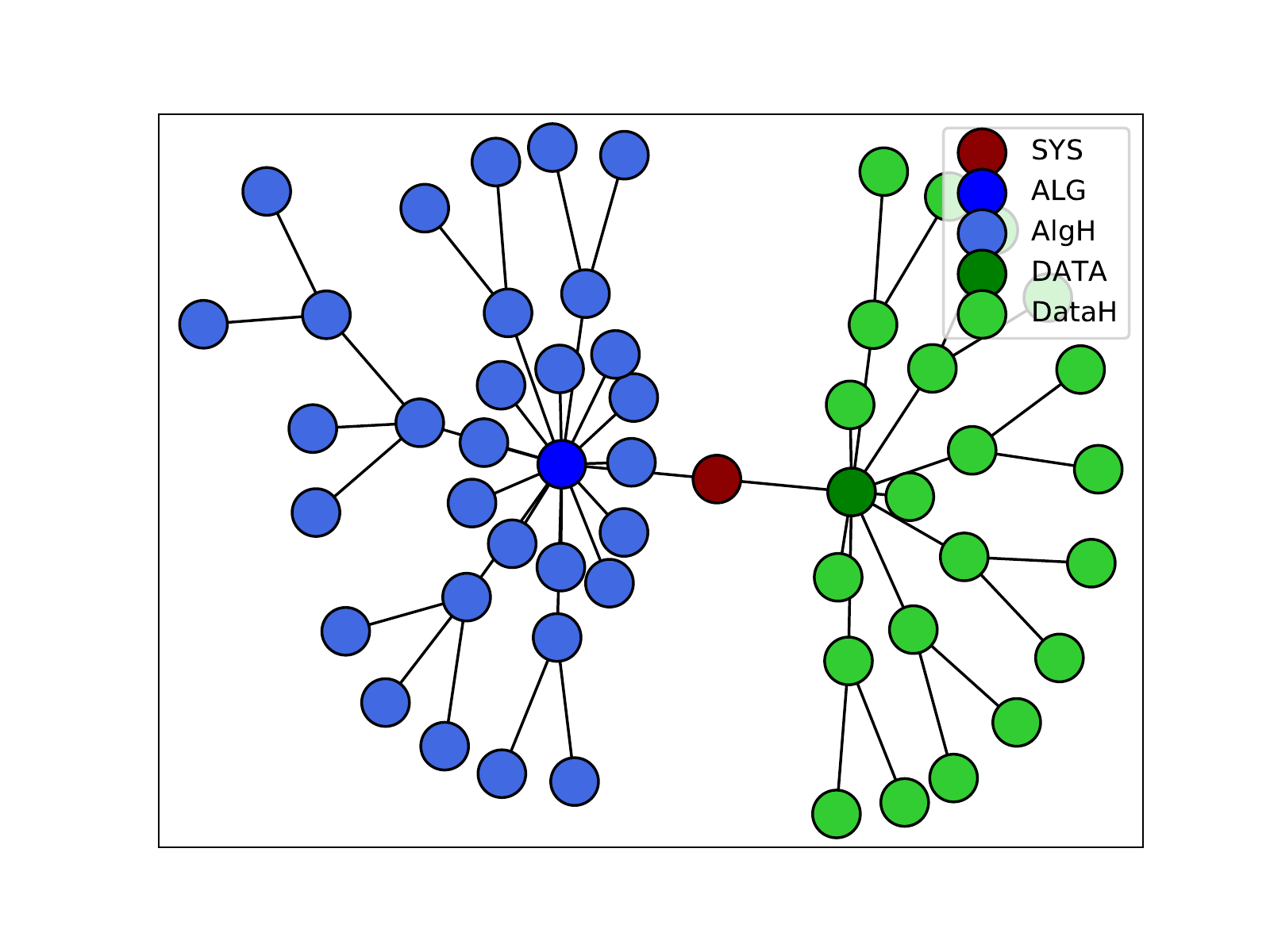}
			\vspace{-.6cm}\caption{~}
			\label{fig:exp-holarchy-clustering-a}
		\end{subfigure}
		\begin{subfigure}[b]{0.49\textwidth}
			\centering
			\includegraphics[page=2,width=.9\textwidth]{figs/exps/Network-After-Clustering.pdf}
			\vspace{-.6cm}\caption{~}
			\label{fig:exp-holarchy-clustering-b}
		\end{subfigure}\\\vspace{-.13cm}
		\begin{subfigure}[b]{0.49\textwidth}
			\centering
			\includegraphics[page=1,width=.9\textwidth]{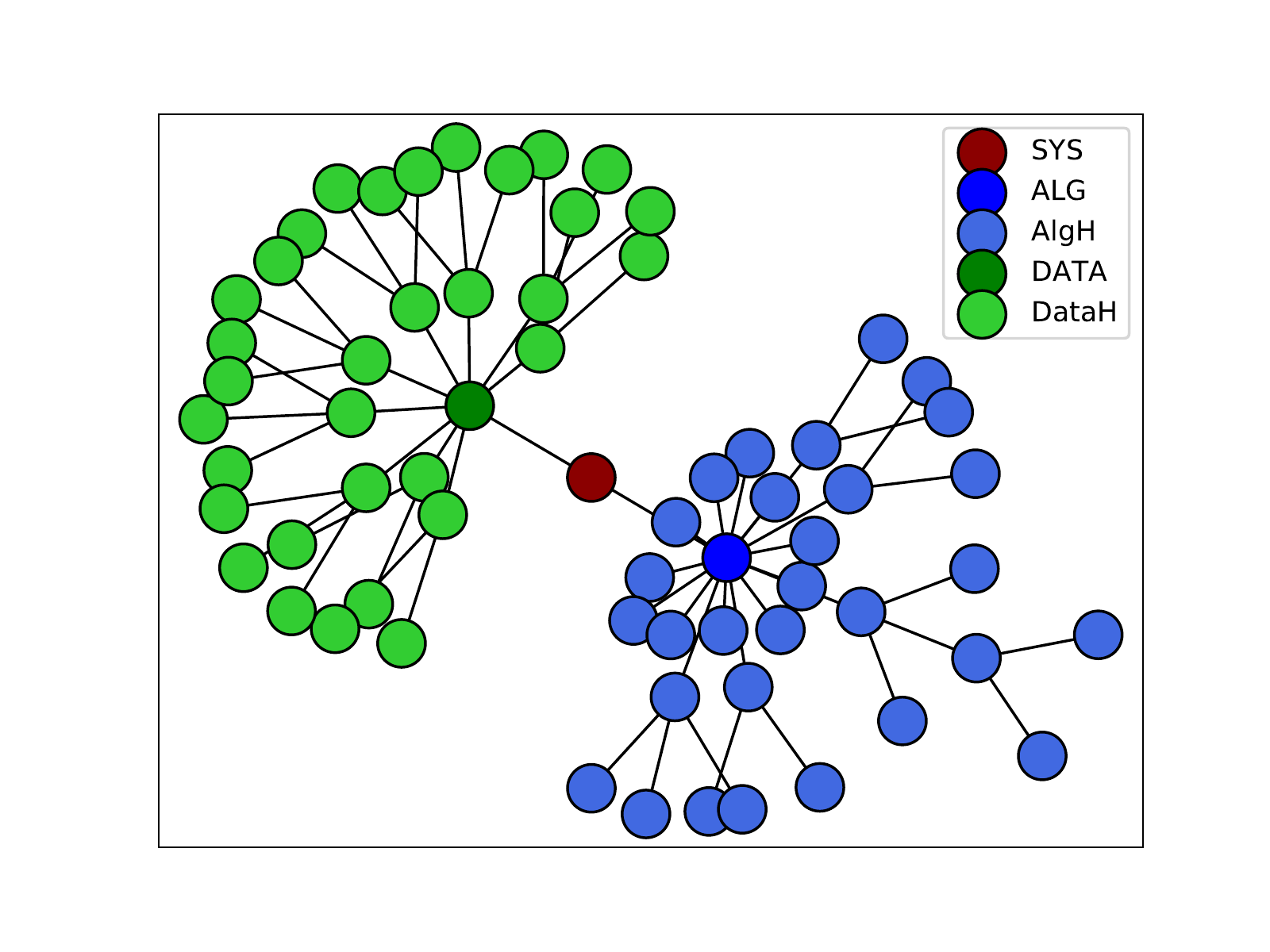}
			\vspace{-.6cm}\caption{~}
			\label{fig:exp-holarchy-data-add-a}
		\end{subfigure}
		\begin{subfigure}[b]{0.49\textwidth}
			\centering
			\includegraphics[page=2,width=.9\textwidth]{figs/exps/Network-After-Data-Add.pdf}
			\vspace{-.6cm}\caption{~}
			\label{fig:exp-holarchy-data-add-b}
		\end{subfigure}
		\caption{The holarchical (left) and the atomic level structures (right) after classification (\ref{fig:exp-holarchy-classification-a},\ref{fig:exp-holarchy-classification-b}), regression (\ref{fig:exp-holarchy-regression-a},\ref{fig:exp-holarchy-regression-b}), clustering (\ref{fig:exp-holarchy-clustering-a},\ref{fig:exp-holarchy-clustering-b}), and adding test datasets (\ref{fig:exp-holarchy-data-add-a},\ref{fig:exp-holarchy-data-add-b}).}
		\label{fig:exp-holarchy-all}
	\end{figure}
	
	The training results for the classification, regression, and clustering are plotted respectively in figures~\ref{fig:exp-train-classification}, \ref{fig:exp-train-regression}, and \ref{fig:exp-train-clustering}. Each plot illustrates the performance of training the algorithms on the available datasets according to tables~\ref{tbl:alg-details} and \ref{tbl:data-details}. In the training queries, we have used \textit{accuracy}, \textit{mean squared error}, and intrinsic \textit{fowlkes-mallows score}\cite{fowlkes1983method} as the measures of performance for classification, regression, and clustering tasks respectively. Each plot in each of the aforementioned figures pertain to the results reported for running multiple algorithms on a specific dataset, and reports both the performance measures(in red) and the amount of time (in blue) the training procedure for each algorithm has taken. Please note that, for the sake brevity, the horizontal axes of the plots are marked after the symbols used in table~\ref{tbl:alg-details}. The reported results provide various analytical and comparative insights about the training procedure. For instance, in figure~\ref{fig:exp-train-classification}, the classification algorithms \emph{A04} and \emph{A07} are among the most efficient and effective algorithms trained on the \emph{iris} dataset. This is because of the achieving the highest accuracy in the lowest amount of time. Likewise, algorithm \emph{A03} is the least efficient and effective algorithm trained on the \emph{artificial moon} dataset due to its relatively high training time and lowest accuracy score. Similar information can also be obtained from studying the reported the results for the other types of conducted machine learning tasks. For example, algorithms \emph{A11} and \emph{A13} are among the best regression algorithms that are included in the platform because of their relatively lower error and training time on all of the available datasets (see figure~\ref{fig:exp-train-regression}). On the other hand, the clustering algorithm \emph{A24} is one of the fast algorithms that yields a relatively higher performance measure on most datasets. Please, note that all the parameters during training, evaluating their results, etc. are chosen based on no specific reason but as an example to demonstrate the capability of the proposed platform.

	\begin{figure}[!htbp]
		\foreach \x in {1,...,6}{
			\begin{subfigure}[b]{0.49\textwidth}
				\centering
				\includegraphics[page=\x,width=.9\textwidth]{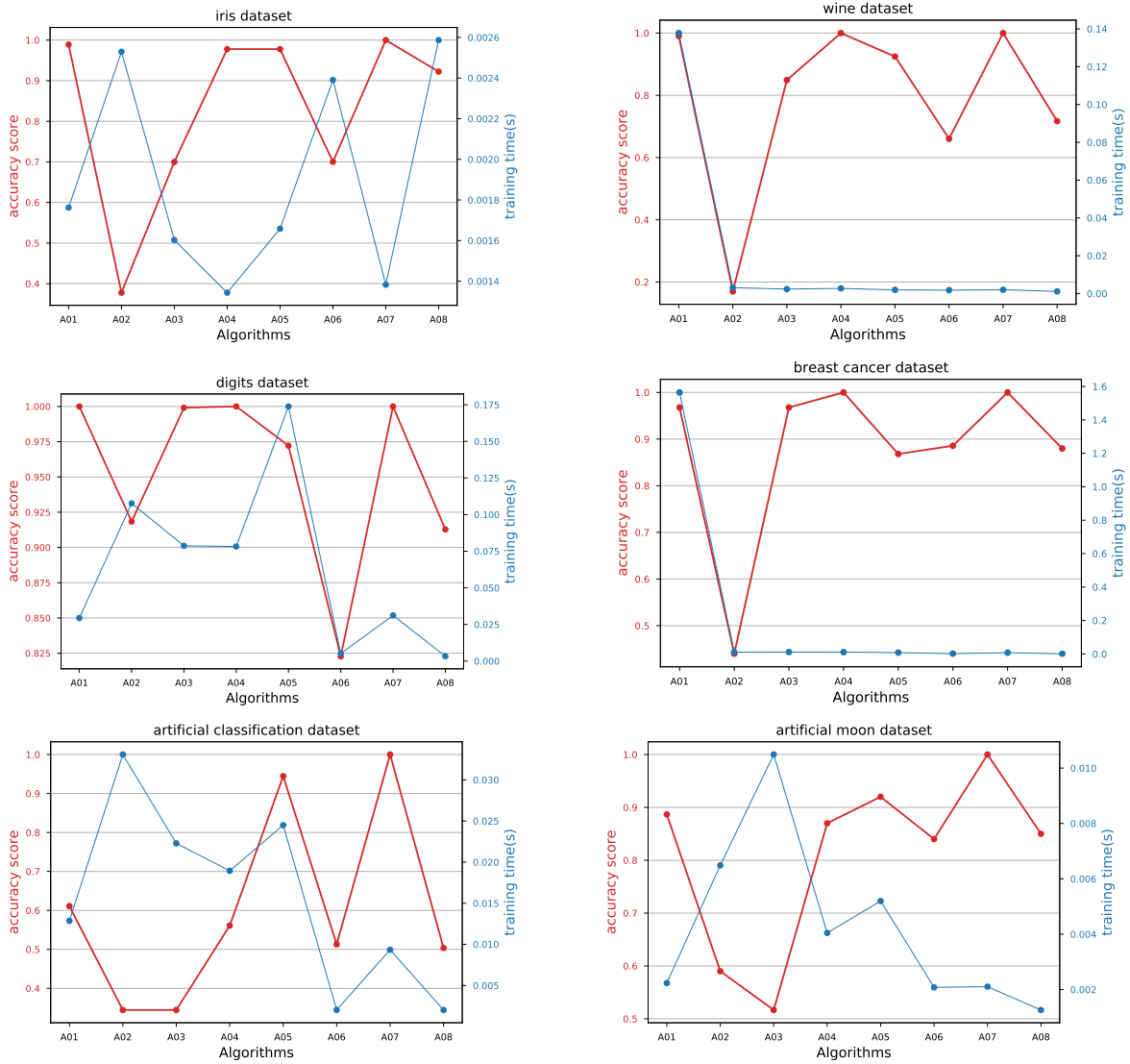}
				% \caption{~}
				\label{fig:exp-train-classification-\x}
			\end{subfigure}    
		}
		\caption{Accuracy and time of training each classification algorithm on a specific dataset.}
		\label{fig:exp-train-classification}
	\end{figure}
	
	\begin{figure}[!htbp]
		\foreach \x in {1,...,3}{
			\begin{subfigure}[b]{0.49\textwidth}
				\centering
				\includegraphics[page=\x,width=.9\textwidth]{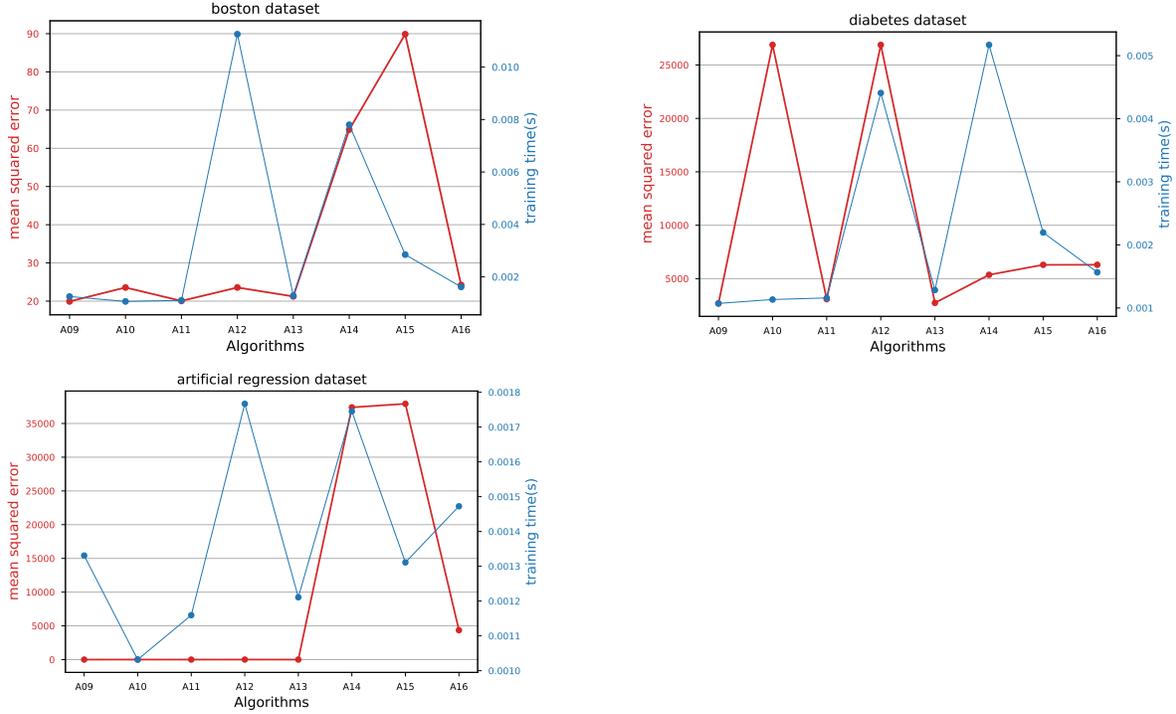}
				% \caption{~}
				\label{fig:exp-train-regression-\x}
			\end{subfigure}    
		}
		\caption{Mean squared error and time of training each regression algorithm on a specific dataset.}
		\label{fig:exp-train-regression}
	\end{figure}
	
	\begin{figure}[!htbp]
		\foreach \x in {1,...,6}{
			\begin{subfigure}[b]{0.49\textwidth}
				\centering
				\includegraphics[page=\x,width=.9\textwidth]{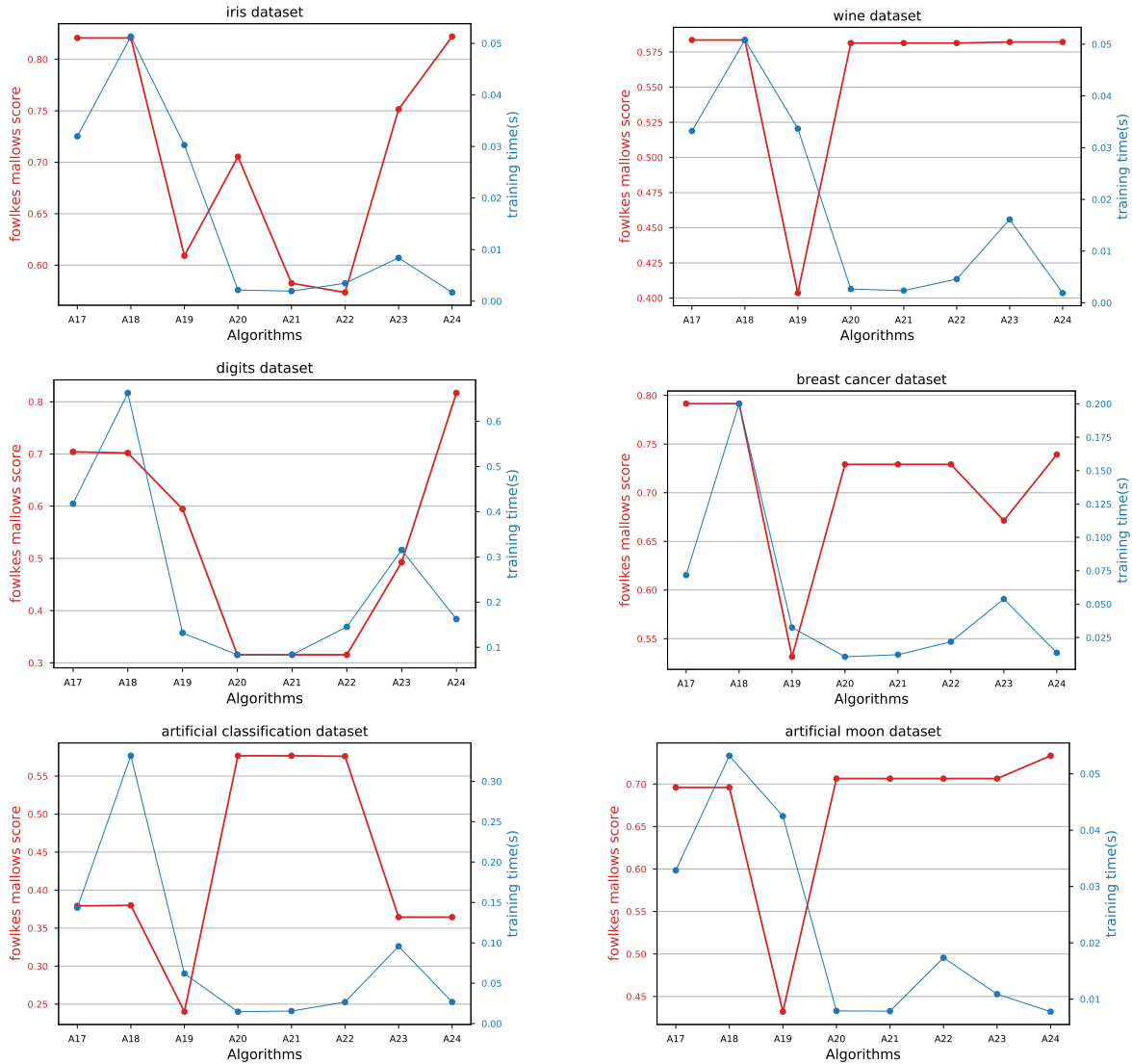}
				% \caption{~}
				\label{fig:exp-train-clustering-\x}
			\end{subfigure}    
		}
		\caption{Fowlkes Mallows score and time of running each clustering algorithm on a specific dataset.}
		\label{fig:exp-train-clustering}
	\end{figure}
	
	\subsection{Testing}
	The algorithms and datasets have been selected is such a way that we could carry out complex testing queries on the platform. This section demonstrates some of the testing capabilities that the proposed platform provides though its flexibility is not limited to only those listed here. The demonstrations are on classification, regression, and clustering tasks and we still use plotting as a compact way of exhibiting the results. It is worth noting that the difference between testing and training a clustering algorithm in this paper is that training allows configuring the algorithm through its parameters, whereas testing merely looks for the algorithm based on the provided criteria and runs it on the specified data.
	
	In the first set of testing queries, the user agent requests the results of testing all the algorithms that have kernel parameter with value 'rbf' on all the inserted test datasets. Please note that based on the assumption that we have made earlier, the algorithms will be tested on the datasets that they have been trained before. Hence, we expect to see the results accordingly. In such a query, we set $\Lambda=\{(*,\{(\text{kernel, rbf})\}\}, \Delta=\{(*,\{(\text{type, test})\})\}$, and $O=$ \{format=plot, measures=[accuracy, mean square error]\}. The results are depicted in figure~\ref{fig:test1}.  As it can be seen, platform has correctly determined the proper algorithms with their corresponding measures for each dataset. For instance, among 8 regression algorithms that are defined in the system, HAMLET has correctly identified \emph{A14} and \emph{A15} as the ones that have ``rbf'' kernel. Similarly, algorithms \emph{A03}, \emph{A04}, and \emph{A05} are the only classification algorithms that match the criteria of the query. Additionally, the platform has successfully tested the algorithms on the datasets that they had been trained on despite the fact that we have not explicitly specified which datasets should be used for the tasks. Such capability of the platform in automatically determining how algorithms, datasets, and measures should be used side by side, together with the analytical and comparative insights that it provides is particularly noteworthy for the cases that users are unaware of the available resources and try to let the system make proper choices based on the most recent resources.
	
	\begin{figure}[!htbp]
		% \centering
		\foreach \x in {1,...,9}{
			\begin{subfigure}[b]{0.325\textwidth}
				\centering
				\includegraphics[page=\x,width=\textwidth]{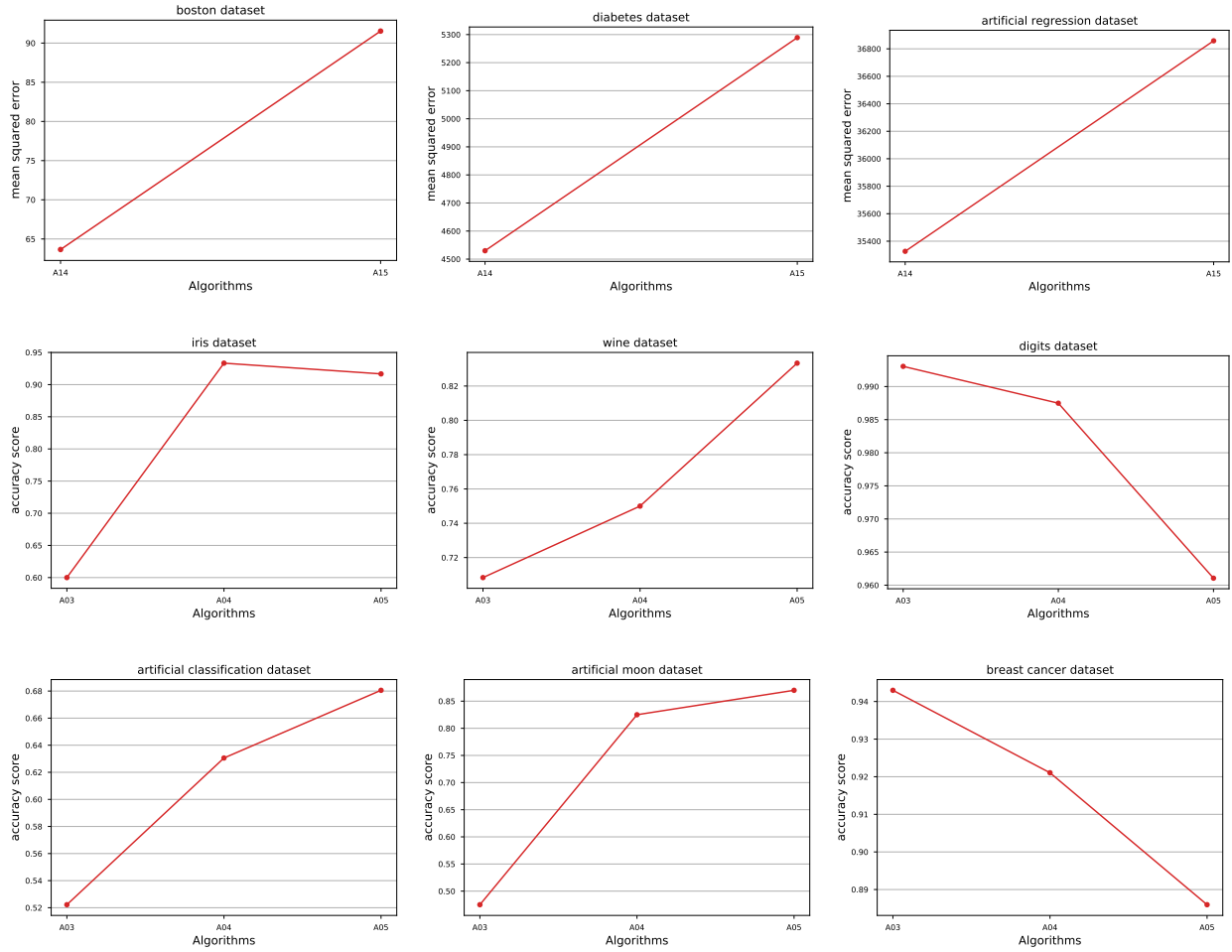}
				% \caption{~}
				\label{fig:test1-\x}
			\end{subfigure}    
		}
		\caption{The results of testing all classification and regression algorithms that have parameter \textit{kernel} equal to \textit{rbf} on all proper datasets, in terms of accuracy and mean squared error.}
		\label{fig:test1}
	\end{figure}
	
	The second experimented query tries to narrow the testing procedure and determine the results of applying all the SVC algorithms on the breast cancer dataset, based on accuracy and area under the ROC curve \cite{fawcett2006introduction} scores. The query configuration for this test is as follows: $\Lambda=\{(\text{SVC},\{*\})\}, \Delta=\{(\text{breast cancer},\{(\text{type, test})\})\}$, and $O=$ \{format=plot, measures=[accuracy, roc-auc]\}; and the generated results are plotted in figure~\ref{fig:test2}. As it can be seen, the subordinate agents of HAMLET have determined algorithms \emph{A01}, \emph{A02}, \emph{A03}, and \emph{A04} as the ones that implement SVC, and reported their test performances using two measures. Please recall that the way the results are reported -- a separate plot for each performance measure in this example -- is decided by the auxiliary visualization agent, and is not enforced by HAMLET.
	
	\begin{figure}[!htbp]
		% \centering
		\foreach \x in {1,2}{
			\begin{subfigure}[b]{0.49\textwidth}
				\centering
				\includegraphics[page=\x,width=\textwidth]{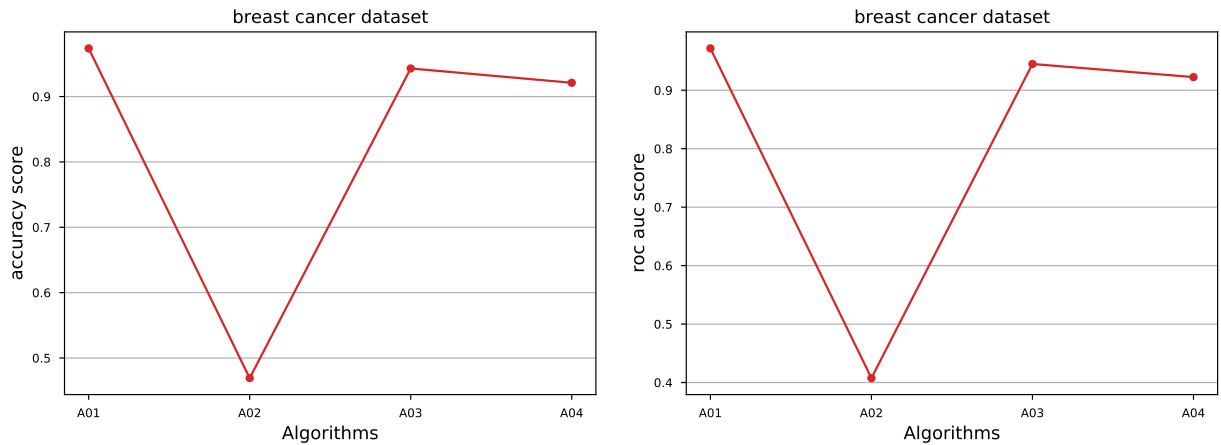}
				% \caption{~}
				\label{fig:test2-\x}
			\end{subfigure}    
		}
		\caption{The results of testing all SVC algorithm with any parameter on the breast cancer dataset, in terms of accuracy and area under the ROC curve score.}
		\label{fig:test2}
	\end{figure}
	
	The third experimental testing query tries to evaluate all the trained algorithms on the artificial moon dataset based on three measures of accuracy, mean squared error, and homogeneity score \cite{rosenberg2007v}. This test is conducted based on the following settings: $\Lambda=\{(*,\{*\})\}, \Delta=\{(\text{moon},\{(\text{type, test})\})\}$, and $O=$ \{format=plot, measures=[accuracy, mean squared error, homogeneity]\}. Please note that we have hypothetically considered that we are not sure whether the moon datasets is of classification or regression type, so we have specified a measure from each machine learning task and let the platform to choose the proper one itself. The obtained results are depicted in figure~\ref{fig:test3}. As it can be seen in the generated plots, only accuracy and homogeneity scores are reported due to the fact that artificial moon is a nominal dataset suitable for classification/clustering tasks. Furthermore, each of the plots in this figure properly reports the testing results for a different set of algorithms using a performance measure that suits their type. Strictly speaking, the plot on the left hand side demonstrates the accuracy of the classification algorithms that HAMLET has examined for this query, and the diagram on the right hand side reports the homogeneity score of the clustering algorithms that match the criteria of the test query.
	
	\begin{figure}[!htbp]
		% \centering
		\foreach \x in {1,2}{
			\begin{subfigure}[b]{0.49\textwidth}
				\centering
				\includegraphics[page=\x,width=\textwidth]{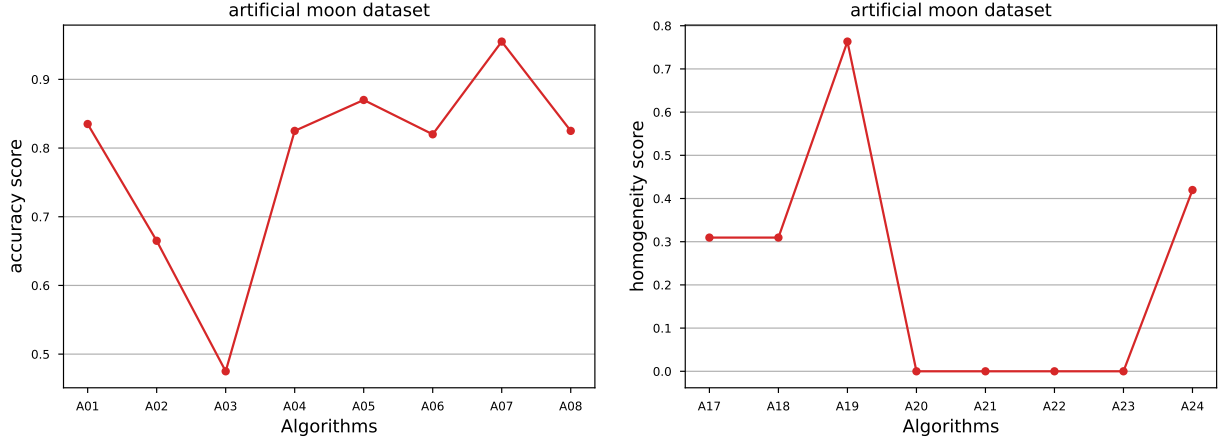}
				% \caption{~}
				\label{fig:test3-\x}
			\end{subfigure}    
		}
		\caption{The results of testing all algorithms with any parameter on the artificial moon dataset, in terms of accuracy, mean square error and/or clustering homogeneity score.}
		\label{fig:test3}
	\end{figure}
	
	Finally, the fourth experiment pertains to testing all of the available algorithms on all of the available test datasets, in terms of the same three measures: accuracy, mean squared error, and homogeneity. This test employs all of the holons in the holarchy to process the query concurrently and based on their capabilities and skills report the results according to the configurations of the query. This query uses the following settings: $\Lambda=\{(*,\{*\})\}, \Delta=\{(*,\{(\text{type, test})\})\}$, and $O=$ \{format=plot, measures=[accuracy, mean squared error, homogeneity]\}; and figure~\ref{fig:test4} illustrates the outcomes. This test shows how HAMLET can be used to comprehensively explore all the available models using a single query. The plots are arranged based on the type of the learning/mining models that have responded to the query, with the regression results in the first row, the classification results in the second and third rows, and finally, the clustering outcomes in the last two rows of the figure. Recalling the number of resources we have inserted and trained, one can easily validate that the results demonstrated in this figure are complete. That is, there is no algorithm/dataset that had been trained before but missed here. This is worth noting that this query deliberately uses a clustering measure, i.e. homogeneity score, from the one we have used before, i.e. fowlked mallows (see figure~\ref{fig:exp-train-clustering}), to demonstrate how the agents representing the models appropriately respond to the query based on their capabilities.
	
	\begin{figure}[!htbp]
		% \centering
		\foreach \x in {1,...,15}{
			\begin{subfigure}[b]{0.325\textwidth}
				\centering
				\includegraphics[page=\x,width=\textwidth]{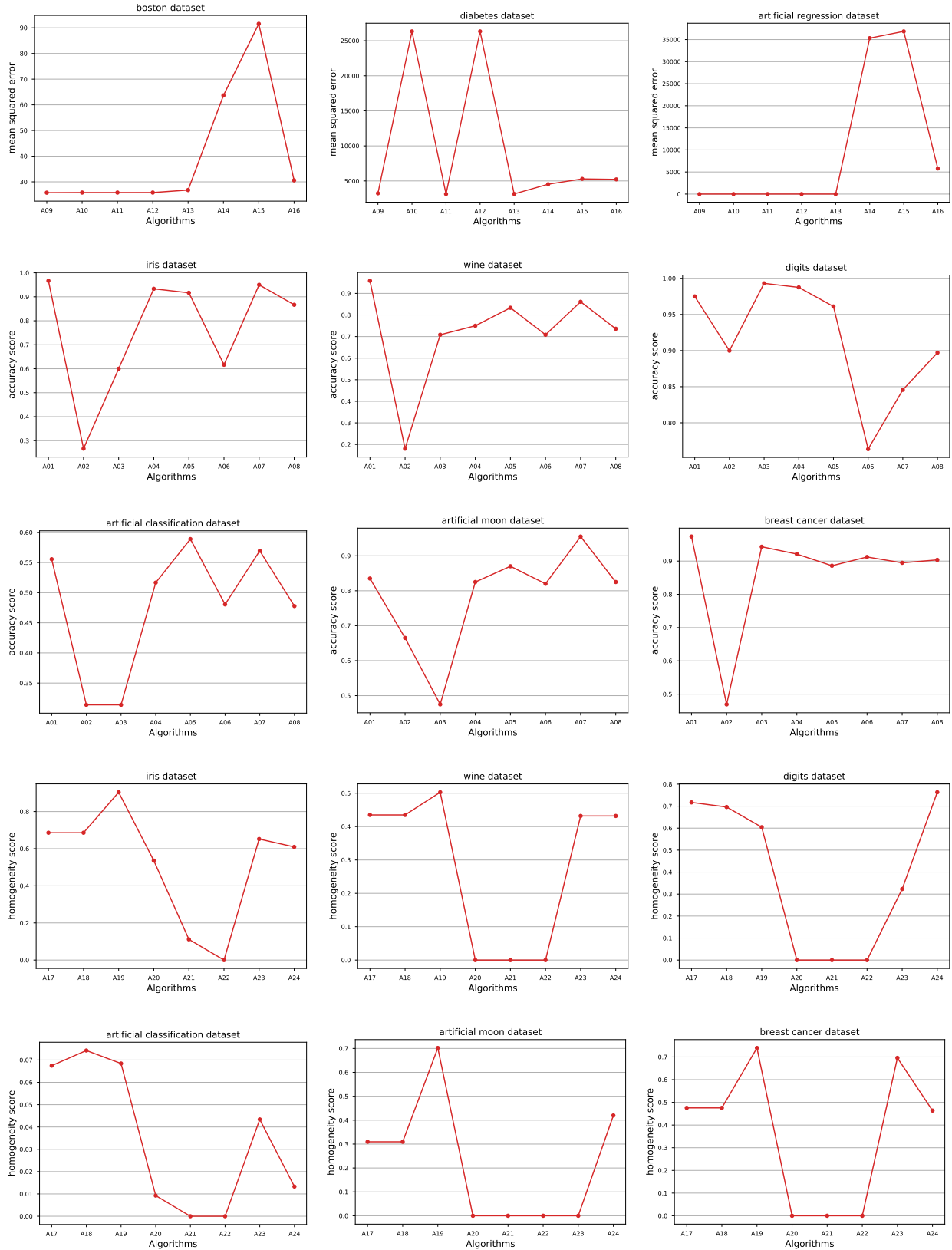}
				% \caption{~}
				\label{fig:test4-\x}
			\end{subfigure}    
		}
		\caption{The results of testing all algorithms with any parameter on all the test dataset, in terms of accuracy, mean square error, and/or clustering homogeneity score.}
		\label{fig:test4}
	\end{figure}
	
	The presented experiments demonstrate the flexibility and capabilities of the proposed platform on performing machine learning tasks. Please note that the platform does not make any optimization on the parameters of the results of the tasks, except the ones that are already employed by each algorithm. Consequently, any poor performance is solely the result of poor underlying utilized algorithms and not because of the platform. Furthermore, as mentioned before, the platform does not enforce any particular reporting procedure. The visualization that we did through the provided queries is merely an example that we chose by adding an appropriate visualization agent to the platform to process the results provided to the system. 
	
	At the end, in order to provide a better understanding about the usefulness of the information that HAMLET can provide for a query, we sent a request to test all the available algorithms on all the available classification datasets and to report accuracy and homogeneity scores as the measures. This example endeavors to exhibit the analytical possibilities that our platform provides based on a selected set of performance measures. Table~\ref{tbl:all-algs-class} summarizes the obtained final results after applying some textual decorations to highlight the important aspects. Please note that, like before, we have used the algorithm and data names from tables~\ref{tbl:alg-details} and \ref{tbl:data-details} for the sake of consistency. The underlined boldface numbers specify the maximum performance that corresponding algorithms have achieved on all datasets. For instance, running the classification algorithm \emph{A01} on all the available datasets, its highest score (\textbf{0.975}) is achieved on the \emph{digit} dataset. The color highlights, on the other hand, accentuate the best algorithms, in terms of the specified performance measures, for conducting a classification or clustering task on a specific dataset. We have used different colors for classification and clustering tasks to help distinguish them from each other. The following are some of the example intuitions that can be made based on the results:
	\begin{itemize}
		\item Algorithm \emph{A01} is the best available classification algorithm for the \emph{breast cancer} dataset.
		\item The majority of the classification algorithms, i.e. A01-A05, have achieved their highest performances on the \emph{digit} dataset.
		\item The \emph{A19} clustering algorithm has performed better than the other algorithms, in terms of the number of the datasets that it resulted the highest homogeneity score.
		\item The \emph{make classification} dataset has been the most challenging dataset for both available classification and clustering algorithms. This because of the relatively low scores reported.
	\end{itemize}
	
	\begin{table}[h]
		\centering
		\caption{The results of testing all possible algorithms on the classification datasets.}
		\label{tbl:all-algs-class}
		\begin{tabular}{rcccccccccccccccc}
			\toprule
			{} & {} & \multicolumn{6}{c}{Dataset} \\\cmidrule(r){3-8}
			{Measure} & {Algorithm}  & {breast cancer} & {digits} & {iris} & {art. class.} & {art. moon} & {wine} \\
			\midrule
			\multirow[c]{8}{*}{\rotatebox[origin=c]{90}{\makecell{classification\\ accuracy}}} & A01 & \backgroundcolorlightgreen 0.974 & \bf \underline {0.975} & \backgroundcolorlightgreen 0.967 & 0.556 & 0.835 & \backgroundcolorlightgreen 0.958\ \\
			& A02 & 0.469 & \bf \underline {0.900} & 0.267 & 0.314 & 0.665 & 0.181 \\
			& A03 & 0.943 & \backgroundcolorlightgreen \bf \underline {0.993} & 0.600 & 0.314 & 0.475 & 0.708 \\
			& A04 & 0.921 & \bf \underline {0.987} & 0.933 & 0.517 & 0.825 & 0.750 \\
			& A05 & 0.886 & \bf \underline {0.961} & 0.917 & \backgroundcolorlightgreen 0.589 & 0.870 & 0.833 \\
			& A06 & \bf \underline {0.912} & 0.764 & 0.617 & 0.481 & 0.820 & 0.708 \\
			& A07 & 0.908 & 0.811 & 0.950 & 0.556 & \backgroundcolorlightgreen \bf \underline {0.955} & 0.931 \\
			& A08 & \bf \underline {0.904} & 0.897 & 0.867 & 0.478 & 0.825 & 0.736 \\\midrule\midrule
			\multirow[c]{8}{*}{\rotatebox[origin=c]{90}{\makecell{clustering\\homogeneity}}} & A17 & 0.476 & \bf \underline {0.695} & 0.686 & 0.066 & 0.310 & 0.435 \\
			& A18 & 0.476 & \bf \underline {0.721} & 0.686 & \backgroundcolorlightblue 0.071 & 0.310 & 0.435 \\
			& A19 & \backgroundcolorlightblue 0.727 & 0.573 & \backgroundcolorlightblue \bf \underline {0.902} & 0.071 & \backgroundcolorlightblue 0.823 & \backgroundcolorlightblue 0.536 \\
			& A20 & 0.000 & 0.000 & \bf \underline {0.536} & 0.009 & 0.000 & 0.000 \\
			& A21 & 0.000 & 0.000 & \bf \underline {0.111} & 0.000 & 0.000 & 0.000 \\
			& A22 & \bf \underline {0.000} & 0.000 & -0.000 & 0.000 & 0.000 & 0.000 \\
			& A23 & \bf \underline {0.696} & 0.323 & 0.652 & 0.043 & 0.000 & 0.432 \\
			& A24 & 0.464 & \backgroundcolorlightblue \bf \underline {0.763} & 0.610 & 0.013 & 0.420 & 0.432 \\
			\bottomrule
		\end{tabular}
	\end{table}

	\clearpage
	\section{Conclusion}
	In this paper, we presented a hierarchical multi-agent platform for the management and execution of data-mining tasks. The proposed solution models a machine learning problem as a hypergraph and employs autonomous agents to cooperatively process and answer training and testing queries based on their innate and learned capabilities. Using an agent-based approach for the problem, on one hand, facilitates the deployment of the system on distributed infrastructures and computer networks, and on the other, provides the researchers with the flexibility and freedom of adding their own machine learning algorithms or datasets with customized behaviors. The platform provides numerous potential benefits for both research and deployment purposes. It can be used by research communities to share, maintain, and have access to the most recent machine learning problems and solutions, such that they are able to analyze new data, compare the performance of new solutions with the state of the art. It can also be utilized in devising new solutions by letting the designers experiment with different versions of their methods distributedly and in parallel in order to understand the behavior of their algorithms under different configurations and select the most appropriate one accordingly.
	
	We have assessed the proposed platform both theoretically and empirically. By means of a set of theorems and lemmas, we proved that the agent-based solution is sound and complete. That is, given a machine learning query, it always returns the correct answer whenever one exists, and warns appropriately otherwise. We have also analyzed its performance in terms of time complexity and space requirements. According to the discussions, our proposed method requires polynomial time and memory to respond to training and testing queries in the worst case. Furthermore, we designed and carried out a set of experiments to show the flexibility and capabilities of the proposed agent-based machine learning solution. We used 24 classification, regression, and clustering algorithms and applied them to 9 real and artificial datasets. The results of both training and testing queries, plotted by the system, demonstrated its correctness together with how a user can perform single and batch queries to extract the existing knowledge.  
	
	This paper proposed the foundations of the suggested agent-based machine learning platform and can be extended in various ways to support new applications and scenarios. At its presented state, tasks such as visualization and preprocessing are managed and conducted by a separate unit and the data agents respectively, whereas they can be performed by separate hierarchies in order to make sophisticated results and data preparation services available to the system. In order to support sophisticated algorithms and analyses, the platform can also be expanded to support machine learning pipe-lined tasks through more horizontal cooperative interactions between the agents at different levels. In the presented version, the structure grows because of the training process and adding new algorithms/datasets. This dynamic growth property can be enhanced even more by letting already-built substructures merge together. Last but not least, the dynamic behavior of the platform can be improved by exclusively handling abnormal events such as changes to the structure because of agent permanent failures. We are currently working on these extensions and suggest them as future work.

% \newpage
\bibliographystyle{unsrt}
\bibliography{main}

\end{document}